\documentclass{article}
\usepackage{latexsym}
\usepackage{amsmath, amsfonts, amsthm, cases}
\usepackage{epsfig}
\usepackage{stmaryrd}
\usepackage{multicol}
\usepackage{pdfsync}
\usepackage{dsfont}
\usepackage{comment}
\usepackage{algorithm, algpseudocode}
\usepackage{caption}
\usepackage{psfrag}
\usepackage{subfig}
\usepackage[dvipsnames]{xcolor}
\usepackage{graphics}
\usepackage{enumitem}
\usepackage{mathtools}
\usepackage{mathrsfs} 
\usepackage{color}
\usepackage{hyperref}
\usepackage{bm}

\DeclarePairedDelimiter\floor{\lfloor}{\rfloor}

\usepackage{tikz-cd}

\hypersetup{
	colorlinks = true, % Colours links instead boxes
	urlcolor = blue, % Colour for external hyperlinks
	linkcolor = red, % Colour of internal links
	citecolor = blue % Colour of citations
}

\newcommand{\la}{\left \langle}
\newcommand{\ra}{\right \rangle}
\newcommand{\prob}{\mathbb{P}}

\newcommand{\A}{\mathcal{A}} % action space

\newcommand{\E}[0]{\mathbb{E}} % expectation

\newcommand{\e}{\mathbb{E}}

\newcommand{\F}{\mathcal{F}} % sigma algebra

\newcommand{\K}{\mathcal{K}}

\newcommand{\M}{\mathcal{M}}
\newcommand{\N}[0]{\mathds{N}} % natural numbers
\newcommand{\p}{\mathbb{P}}
\newcommand{\R}[0]{\mathds{R}} % real numbers

\newcommand{\X}{\mathcal{X}}
\newcommand{\aux}{\mathsf{aux}}
\newcommand{\avg}{\mathsf{avg}}
\newcommand{\TV}{\mathsf{TV}} % state-value functio

\newcommand{\bracket}[1]{\left( #1 \right)}
\newcommand{\sqbracket}[1]{\left[ #1 \right]}
\newcommand{\set}[1]{\left\{ #1 \right \}}
\newcommand{\abs}[1]{\left| #1 \right|}
\newcommand{\norm}[1]{\left\| #1 \right\|}
\newcommand{\pheq}{\phantom{=}}
\newcommand{\de}{\partial}

\newcommand{\bae}{\begin{equation}\begin{aligned}}
\newcommand{\eae}{\end{aligned}\end{equation}}
\newcommand{\beq}{\begin{equation}}
\newcommand{\eeq}{\end{equation}}

\newtheorem{theorem}{Theorem}[section]

\newtheorem{lemma}[theorem]{Lemma}
\newtheorem{proposition}[theorem]{Proposition}
\theoremstyle{definition}
\newtheorem{definition}[theorem]{Definition}

\newtheorem{assumption}[theorem]{Assumption}

\theoremstyle{remark}
\newtheorem{remark}[theorem]{Remark}

\numberwithin{equation}{section}

\newcommand\numberthis{\addtocounter{equation}{1}\tag{\theequation}}

\newcommand{\clip}{\mathrm{clip}}

\begin{document}

\title{Weak Convergence Analysis of Online Neural Actor-Critic Algorithms}

\author{Samuel Lam\footnote{Mathematical Institute, University of Oxford, Oxford, OX2 6GG, UK (samuel.lam@maths.ox.ac.uk).}, \phantom{.} Justin Sirignano\footnote{Mathematical Institute, University of Oxford, Oxford, OX2 6GG, UK (Justin.Sirignano@maths.ox.ac.uk).}, \phantom{.} Ziheng Wang\footnote{Mathematical Institute, University of Oxford, Oxford, OX2 6GG, UK (wangz1@math.ox.ac.uk).}\footnote{Author order is alphabetical.} }

\maketitle

\begin{abstract}
We prove that a single-layer neural network trained with the online actor critic algorithm converges in distribution to a random ordinary differential equation (ODE) as the number of hidden units and the number of training steps $\rightarrow \infty$. In the online actor-critic algorithm, the distribution of the data samples dynamically changes as the model is updated, which is a key challenge for any convergence analysis. We establish the geometric ergodicity of the data samples under a fixed actor policy. Then, using a Poisson equation, we prove that the fluctuations of the model updates around the limit distribution due to the randomly-arriving data samples vanish as the number of parameter updates $\rightarrow \infty$. Using the Poisson equation and weak convergence techniques, we prove that the actor neural network and critic neural network converge to the solutions of a system of nonlinear ODEs with random initial conditions. Analysis of the limit ODE shows that the limit critic network will converge to the true value function, which will provide the actor an asymptotically unbiased estimate of the policy gradient. We then prove that the limit actor network will converge to a stationary point. 
\end{abstract}

\section{Introduction} 
\hspace{1.4em} 
Neural network actor-critic algorithms are one of the most popular methods in deep reinforcement learning. A neural network model is trained to select the policy (the ``actor") while another neural network (the ``critic") is simultaneously trained to learn the value function given the actor's policy. Specifically, the actor selects an action and, given the action, a new state transition occurs according to a Markov stochastic process and a reward (a measurement of the success/failure) is observed. The critic must learn to approximate the value function -- the solution to the Bellman equation -- given the actor's policy. Given the critic's estimate for the value function of the current policy, the actor must be updated to improve the value function (i.e., increase the expected reward). Actor-critic algorithms are well-established methods in reinforcement learning \cite{konda1999actor, konda2002actor}; the key recent advance is using (deep) neural networks as the actor/critic and training their parameters using gradient descent methods \cite{DQNAtari2015, NeuralNaturalAC, mnih2016asynchronous, naturalAC1, naturalAC2}.\\

Analysis of neural network actor-critic algorithms is challenging due to: (1) the non-convexity of the neural networks, (2) the complex feedback loop between the actor and critic (the actor determines the sequence of data samples which are used to train the critic and the critic is used to train the actor), and (3) the simultaneous online updates of both the actor and critic which lead to (3A) the distribution of the data, which depends upon the actor, constantly evolving in time and (3B) the actor being updated with a noisy, biased estimate of the value function. The algorithm's fully online parameter updates and online data generation from the Markov chain -- where data samples are not i.i.d. -- introduces significant complexities which must be addressed in the mathematical analysis. The convergence of a fully online neural network actor-critic algorithm has not been previously studied in the literature; a detailed discussion of previous analysis of actor-critic algorithms
and relevant literature is presented in the next section. \\

We study the behaviour of the online neural network actor-critic algorithm as the number of training steps and the number of hidden units $\to \infty$. First, we prove that the time-rescaled \textit{trajectory} of the actor and critic networks converges \textit{pathwise} weakly to an ODE system with a random initialisation as the number of hidden units $\rightarrow \infty$. Then, we analyze the ODE to prove that, as the (rescaled) training time $\to\infty$, the critic limit converges to the value function and the actor limit converges to a stationary point of the objective function (the expected discounted reward), in the sense that the gradient of the objective function $\to 0$. Convergence rates are also established for both the critic and the actor. \\

The learning rates for the actor/critic parameter updates and exploration schedule for the $\epsilon$-greedy exploration algorithm must be carefully selected to ensure convergence, which can provide useful guidance for the practical implementation of neural network actor-critic algorithms. Our paper also provides theoretical guarantees for a class of online neural network actor-critic algorithms that are widely-used in practice. There is no guarantee that the finite-hidden-unit critic network can accurately approximate the action value function (it may only converge to a local minimizer due to non-convexity), and therefore it is also not guaranteed in practice that the actor parameters update in the direction of steepest descent for the objective function. Establishing convergence in the regime where the number of hidden units $\rightarrow \infty$ provides an important theoretical foundation for these algorithms and, more importantly, suggests that large neural networks (i.e., large numbers of hidden units) should be used in practice to guarantee that the actor-critic algorithm converges during training.

\subsection{Literature Review: Convergence Analysis of Actor-critic Algorithms}

\paragraph{Stochastic Approximations and ODE Methods} Various versions of actor-critic algorithms have been studied under the framework of stochastic approximation algorithms, see \cite{konda1999actorsiam, Borkar2000, konda2002actor, Konda2003, 2019Ramaswamy, Borkar2022, Borkar2024, odeMethod2024} and the associated references for an extensive discussion and literature review. The textbooks \cite{Benveniste1990, Borkar2022} are recommended for a general overview of this method. Our analysis is substantially different from this existing ODE literature such as \cite{Benveniste1990, Borkar1997, Borkar1998async, 2019Ramaswamy, Borkar2000,Borkar2024,Konda2003}. In the stochastic approximation of ODE literature, a common method for analysing the stability and convergence of this class of algorithms would be to show that the algorithm converges to the invariant set(s) of an associated ODE \cite{Benveniste1990, Borkar1997, Borkar1998async, 2019Ramaswamy}. As a result, the algorithm can be studied by characterizing the ODE limit(s) \cite{Borkar2000, DiCastro2010, Borkar2024, odeMethod2024}. In contrast, we prove that the time-rescaled evolution converges \textit{pathwise} to the limit ODE as the number of hidden units of the neural network $\rightarrow \infty$. The stochastic approximation literature for ODEs such as \cite{Borkar2000, Borkar2024} does not consider neural networks. \\

Our approach relates the actor-critic algorithm with an ODE in a different way than in the above literature. Here we establish the weakly, \textit{pathwise} convergence of the time-rescaled trajectory of the actor-critic algorithm with neural network approximations using weak convergence techniques \cite{ethier2009markov} as the number of hidden units and training steps $\rightarrow \infty$. Our paper significantly expands upon \cite{wang2021global}, which only considered tabular actor-critic models and not the neural network case. \textit{Tabular} actor-critic algorithms do not use any function approximations (such as neural networks) for the actor and critic. The function approximation introduces significant additional mathematical challenges in the analysis. In particular, the non-convexity of the neural network and the random limit of the neural network initialization (which requires a new weak convergence approach to prove the convergence as the number of hidden units $\rightarrow \infty$, which was not necessary in \cite{wang2021global}) makes the mathematical setting of this paper much more challenging than \cite{wang2021global}. Also, from an applied perspective, this paper, which considers an online neural network actor-critic algorithm, is directly relevant to recent developments in AI and deep reinforcement learning (as compared to the tabular actor-critic algorithm considered in \cite{wang2021global}, which is not used in practice). We emphasise that weak pathwise convergence of the actor and critic network outputs to a limit ODE as the number of hidden units $\rightarrow \infty$ has not been previously considered in the ODE literature discussed above.

\paragraph{Finite time analysis}
Non-asymptotic convergence rates can also be established for the actor-critic algorithm using finite-time analysis approaches. These results establish a convergence rate to a time when the optimality gap is arbitrarily small. \cite{Liu2019, agarwal2020, wang2019neural} are amongst the earliest works that provide finite-time analyses to reinforcement learning algorithms on a Markov Decision Process (MDP), with \cite{wang2019neural} being the first to consider neural network approximations in the algorithm. Since then, finite time convergence rates were provided for the more complicated actor-critic algorithms on MDPs, with linear approximators for the action value function have been proven in \cite{xu2020nonasymptotic, finite_time_2020, qiu2021, KumarHarshat2023Otsc}. \\

The recent neural tangent kernel (NTK) analysis \cite{oldNTK, oldNTK1, NTK1, NTK2} has enabled subsequent research such as \cite{cai2019temporaldifference, fu2021singletimescaleactorcriticprovablyfinds, cayci2022finitetime} which study actor-critic algorithms with neural network approximations. \cite{cai2019temporaldifference} developed an NTK analysis of temporal difference learning. A related paper which studies actor critic algorithms, although without neural networks, is \cite{hong2023twotimescalecomplexity}. Our convergence analysis is different than the mathematical approaches in \cite{fu2021singletimescaleactorcriticprovablyfinds, hong2023twotimescalecomplexity, cayci2022finitetime}. Our analysis studies the long-term behaviour of the algorithm and provides a bridge between the algorithm and its continuous-time limit, which is not studied in \cite{fu2021singletimescaleactorcriticprovablyfinds, hong2023twotimescalecomplexity, cayci2022finitetime}. There are several other key differences between the algorithms and theoretical results in \cite{fu2021singletimescaleactorcriticprovablyfinds, hong2023twotimescalecomplexity, cayci2022finitetime} as compared to our paper. \cite{hong2023twotimescalecomplexity} does not study neural networks. \cite{cayci2022finitetime} and \cite{fu2021singletimescaleactorcriticprovablyfinds} both study offline algorithms (sometimes referred to as ``batch" actor-critic algorithms) while we study an \emph{online} algorithm, where the latter requires developing a different mathematical approach for the analysis. The mathematical results of \cite{fu2021singletimescaleactorcriticprovablyfinds} are significantly different than the convergence results in our paper. \cite{fu2021singletimescaleactorcriticprovablyfinds} proves convergence of the critic in a time-averaged sense (the time-average of the critic iterations converges) while we directly prove convergence of the critic to the action-value function which satisfies the Bellman equation. Finally, both \cite{fu2021singletimescaleactorcriticprovablyfinds} and \cite{hong2023twotimescalecomplexity} assume that the stationary distribution of the Markov chain for the current policy is known and directly samples from this stationary distribution for the training data samples, which is typically not possible in practice. In practice, the stationary distribution is not known and instead data samples are generated by running the Markov chain while in parallel taking online parameter updates. In our paper, we do not assume the stationary distribution is known and, instead, data samples are taken from the sequence generated by running the Markov chain while simultaneously in parallel updating the critic/actor parameters online. Much of the analysis in our paper is directed at addressing the technical challenges arising from a fully online model with both parameters being updated online and the sequence of data samples from the Markov chain (whose distribution will change as the parameters change) being generated online. \\

Typical convergence results in the literature establish that the \textit{minimum} gradient of the expected discounted reward is close to zero \cite{finite_time_2020, qiu2021, KumarHarshat2023Otsc}  or the \textit{minimum} of reward gap between the optimal policy and the current iterate (performance/regret bounds) \cite{hong2023twotimescalecomplexity, cayci2022finitetime, fu2021singletimescaleactorcriticprovablyfinds}. The limitation of these statements is that they do not inform when \textit{exactly} the actor-critic algorithm obtains the optimal policy (amongst the finite steps being taken). Our results are different in the sense that we show that, for the limit ODE system, the \textit{final} iterate of the limit actor network must approach a stationary point. This means that early stopping is not necessary for the limit actor network to find the optimal policy. \\

Finally, another major distinction between the existing literature and the algorithm studied in this paper is its use of the $\epsilon$-greedy algorithm for exploration. The $\epsilon$-greedy algorithm is a widely-used exploration algorithm in deep reinforcement learning. To our knowledge, it has not been theoretically studied as part of a neural network actor-critic algorithm in the previous literature. In our paper, the actor-critic algorithm essentially has three timescales: (1) critic updates, (2) actor updates, and (3) the exploration rate in the $\epsilon$-greedy algorithm. These three timescales and, in particular the exploration rate, must be selected carefully in order to prove convergence. A class of exploration rates are provided which guarantee convergence of the algorithm. \\

The ODE limit exists in part due to the \textit{online} nature of the actor-critic algorithm considered in this paper. During \textit{one} step of the \textit{online} actor-critic algorithm, the actor and critic are \textit{both} updated after observing \textit{one} state-action pair. This is in contrast to the \textit{offline/batch} version of the actor-critic algorithms previously studied in \cite{xu2020nonasymptotic, qiu2021, cayci2022finitetime}, where a large number (to be sent $\rightarrow \infty$) of critic updates is required before an actor update. An advantage of the online algorithm is that a much larger number of optimization iterations can be completed in the same computational time. For this reason, the online actor-critic algorithm is typically used in practice instead of the batch version of the actor-critic algorithm.

\paragraph{Learning Rates and Exploration} It is typical in the classical ODE approach to assume that the learning rates for both the actor and the critic networks satisfy the Robbin-Munro conditions \cite{RobinMunro1951}, see \cite{Benveniste1990, Borkar1997, konda2002actor, Konda2003, Borkar2022}. Moreover, it is often assumed that the critic learning rates are always greater than the actor learning rates, so that the critic network could serve as an accurate estimate to be used for the actor updates. This would be achieved by, for example, Assumption 5.2.2 of \cite{konda2002actor}. Our paper adopts a different set of learning rate schedules: the critic learning rates are assumed to be constant for all steps while the actor learning rate continues to satisfy the Robbin-Munro conditions. \\

An $\epsilon$-greedy policy is used to obtain samples for the actor and critic updates. This is a simple way to ensure that all state-action pairs have been visited sufficiently frequently to search for an optimal policy \cite{watkins1989qlearning, watkins1992q, sutton2018reinforcement}. The selection of the exploration rate, that is, $\epsilon$, remains an active research area \cite{tokic2011softmax}. As far as we are aware, this paper is the first to study the convergence of neural network actor-critic algorithms with a decaying exploration rate schedule for the $\epsilon$-greedy policy. A key conclusion is that the exploration rate must decay \textit{slower} than the actor learning rate for the critic network to converge to the true action value function, and we have provided a class of exploration rates which guarantee convergence of the algorithm. \\

It is also worth highlighting that \cite{agazzi2020globaloptimalitysoftmaxpolicy} has established the convergence to the stationary point when trying to maximise the discounted reward function by directly using continuous-time gradient descent. \cite{zhang2020meanfieldqlearning} carried out a finite-time analysis of Q-learning with approximations of neural networks with mean-field scaling. Finally, \cite{yamamoto2024} performed a finite-time analysis on a \textit{batch} version of the actor-critic with the mean-field scaling, where the gradient descent is performed with a Langevin dynamics. These articles in the literature study different scalings (while we study an NTK scaling) and different algorithms than considered in this paper. \\

\subsection{Our Mathematical Approach}
Our analysis first proves the convergence of the time-rescaled evolution to a limit ODE. Using the ODE, we show that both the \textit{Bellman error} for the critic model and the norm of the gradient of the objective function with respect to the actor converge to zero as the training time tends to infinity. These results are stated in Section \ref{s:main_result}. \\

Typically, the evolution of the neural network's output during training depends on the following: (1) the normalisation(s) (at initialisation), (2) how the learning rate(s) scales with the width of the neural network, and (3) the time rescaling of the evolution. In this paper (where $N$ is the number of hidden units), the actor and critic networks are normalised with the Xavier's $O(1/\sqrt{N})$ normalisation \cite{glorot2010a}, the parameters are updated with an $O(1/N)$ learning rate, and the training evolution is embedded in continuous-time via a time-rescaling a factor of $1/N$ (see \eqref{eq:rescaled_updates} for a precise definition of this). As a result, the trained network parameters remain in a neighbourhood of their initial conditions as
the number of hidden units $N \rightarrow \infty$. This leads to a linearisation of the training dynamics as $N \rightarrow \infty$ around the initial distribution of parameters \cite{NTK2, sirignano2022scalinglimitneuralnetworks}. We can then approximate the evolution of the networks by summing up the linearised changes of the network output, which gives rise to the neural tangent kernel \cite{NTK1}. Since the NTK is positive definite for a wide class of activation functions \cite{ito1996nonlinearity, sirignano2021asymptotics, Carvalho2024NTKpositive}, the critic network therefore converges to the true action-value function that satisfies the Bellman equation, which is needed for the actor updates (and is not guaranteed due to the non-convexity of the critic loss for finite $N$). In the limit $N \rightarrow \infty$, the neural network actor-critic algorithm satisfies
an ODE system which can be considered a version of a gradient flow for solving the Bellman equation. \\ 

The (1) normalisation of the neural network and (2) how the learning rate(s) scales with $N$ together govern the limit behaviour of the network evolution. The analysis of the actor-critic algorithm with a different normalisation scaling than $N^{-1/2}$ is not considered in this paper but is an interesting topic for future research. Briefly, we refer interested readers to a number of references that study the training of neural networks for regression problems under different scalings. \cite{mei2018, mei2019, sirignano2020meanfieldI, sirignano2019meanfieldII, bach2018, rotskoff2022} study the case of mean-field $O(1/N)$ scaling and where the learning rate is $O(1)$; in this case, the evolution converges to a PDE. \cite{Bortoli2020propagationofChaos} studies the case of a mean-field $O(1/N)$ scaling and a learning rate of $O(N)$, in which the evolution converges to a McKean-Vlasov equation with a Brownian noise. \\

We will first prove the \textit{weak} convergence of the evolution of network outputs as $N\to\infty$. Note that it is not possible to establish a uniform convergence as in \cite{wang2021global} due to the NTK normalisation of the neural network. The proof begins with establishing the relative compactness of the time-scaled evolution of the actor and critic networks using the routines provided in \cite{ethier2009markov}. We shall then show that any weak limit(s) of the time-rescaled evolution must satisfy the limit ODE, which involves analyses of the fluctuation terms. Finally, we show that the limit ODE is well-posed. \\

The most challenging step of establishing the weak convergence is to analyse the fluctuation terms. This is because the data samples are not independent and identically distributed (i.i.d.). In fact, their distributions depend on the actor network. Moreover, the update of the actor parameters depends on the data samples, introducing a complex feedback loop. Fortunately, under mild assumptions on the MDP as discussed \cite{Puterman2014markov, Kallenberg2002, Kallenberg2021}, we are able to establish the geometric ergodicity of the data samples under a fixed $\epsilon$-greedy policy with Softmax by minorisation arguments \cite{Norris_1997, meyn2012markov}. We can then construct a Poisson equation of the Markov process (induced by the fixed $\epsilon$-greedy policy) that connects its stationary distribution with its transition kernel \cite{Konda2003, sirignano2021asymptotics, wang2021global}. Using the Poisson equation, we show that the fluctuations of the model updates around the limit distribution due to the random data samples vanish in $L^1$ as the number of parameter updates $\rightarrow \infty$. As a result, the actor and critic networks converge to the solution of the limit ODE with random initial conditions. It is interesting to note that, in contrast to the limit ODE obtained by training a neural network for the regression task \cite{NTK2, sirignano2022scalinglimitneuralnetworks}, the limit ODE for the actor-critic algorithm is non-linear. The nonlinearity of the limit ODE system for the neural network actor-critic algorithm makes it more challenging to study the convergence as the training time $t \rightarrow \infty$. \\

\paragraph{Convergence Analysis for the Limit ODE as training time $t \rightarrow \infty$} Leveraging the two timescales for the actor and critic ODEs (due to their respective learning rates), we are able to prove that the critic ODE converges to the true value function (the solution of the Bellman equation) for the current actor's policy at time $t$ as the training time $t \rightarrow \infty$. The critic, therefore, provides the actor with an asymptotically unbiased estimate of the policy gradient. Surprisingly, the convergence rate coincides with the exploration rate. \\

Finally, we prove that the limit actor network will converge to a stationary point of the objective function as $t \rightarrow \infty$. Therefore, although in the pre-limit actor-critic algorithm the critic provides a noisy, biased (i.e., there is error) estimate of the value function, we are able to prove that asymptotically the critic will converge sufficiently rapidly such that the actor also converges. Our analysis also provides a necessary condition on how the exploration rate should be chosen based on the actor learning rates to ensure convergence.

\subsection{Organisation of the analysis}
Section \ref{s:ActorCriticAlgorithm} describes the class of neural network actor-critic algorithms that we study. Section \ref{s:main_result} states the main convergence results that are proven. The proof of the convergence of the network evolutions is presented in Section \ref{s:derivation_of_the_limit_ODEs}. Finally, we analyse the limit ODE as the training time $t \rightarrow \infty$ in Section \ref{s:analysis_of_the_limiting_ODE} to establish the convergence of critic network to the true action-value function and the convergence of actor network to a stationary point of the expected discounted future reward.

\section{Actor-Critic Algorithms} \label{s:ActorCriticAlgorithm}
	
\subsection{Markov Decision Processes}    
Reinforcement learning studies games defined by a Markov decision process (MDP).

\begin{definition}[Markov decision process (MDP)] A Markov decision process is the tuple
$\bm{\mathcal{M}} = (\bm{\X}, \bm{\A}, p, \rho_0, r, \gamma)$ of the following:
\begin{itemize}
\item $\bm{\X} \subseteq \R^{d_x}$, the space of all possible states of the MDP (the \textit{state space}); 
\item $\bm{\A} \subseteq \R^{d_a}$, the space of all actions of the MDP (the \textit{action space});
\item $p(x'| x,a)$, the transition kernel that gives the probability of next state being $x'$ if the current state is $x$ and the action $a$ is taken;
\item $\rho_0$, distribution that characterises how the initial state is chosen,
\item $r(x,a)$, reward gained by taking action $a$ at state $x$, and
\item $\gamma \in(0,1)$ being the \textit{discount factor}.
\end{itemize}
Here $\bm{\X} \times \bm{\A} \subset \R^d$, where $d = d_x + d_a$. Any elements $\xi := (x,a) \in \bm{\X} \times \bm{\A}$ are called  \textit{state-action} pairs.
\end{definition}

We make the same assumptions on the MDP as those made in \cite{wang2021global}:

\begin{assumption}[Basic assumptions on the MDP] \label{as:MDP_basic}\phantom{=}
\begin{itemize}
\item The state space $\bm{\X}$ is finite space with size $\#\bm{\X}$,
\item the action space $\bm{\A}$ is finite with size $\#\bm{\A}$, 
\item there exists $C > 0$ such that $|(x,a)| \leq C$ for any $(x,a) \in \bm{\X} \times \bm{\A}$, (here $|\cdot|$ is the Euclidean norm in $\R^d$),
\item the elements of $\bm{\X} \times \bm{\A}$ are pairwise non-proportional, i.e., for any $(x,a)$ and $(x',a')$ there is no constant $c$ such that $(x,a) = c(x',a')$, and
\item The reward function $r$ is bounded in $[-1,1]$.
\end{itemize}
We denote the size of the state-action space $\bm{\X} \times \bm{\A}$ as $M = \#\bm{\X} \times \#\bm{\A}$. Due to the finiteness of $\bm{\X} \times \bm{\A}$, there is always a constant $C$ that satisfies $|(x,a)| \leq C$.
\end{assumption}

\subsection{Policy acting on MDP}
A policy $f$ is any function $f : \bm{\X} \times \bm{\A} \to \R$ such that $f(x,a) \geq 0$ and, for any $x$, $\sum_{a'} f(x,a) = 1$. It represents the probability of taking action $a$ for state $x$ when a player follows this policy. \\

A policy $f$ acts on the MDP $\bm{\M}$ to induce the following time-homogeneous Markov chain on the state-action pair $\xi_k := (x_k, a_k)$:
\begin{equation}
(\bm{\M},f): \begin{cases} 
x_0 &\sim \rho_0 \\
a_{k+1} &\sim f(x_k, \cdot) \\
x_{k+1} &\sim p(\cdot \mid x_k, a_k)
\end{cases}.
\end{equation}
The induced Markov chain $(\bm{\M},f)$ admits a transition kernel
\begin{equation} \label{eq:transition_kernel_induced}
\mathbb{P}_{f}((x,a), \, (x',a')) = f(x',a') p(x'\mid x,a).
\end{equation}
The overall reward for a policy $f$ to act on the MDP $\bm{\M}$ is evaluated by the following state and action-value functions:
\begin{definition}[State and action-value functions]
The state and action-value functions of a policy $f$ acting on MDP $\bm{\M}$ are defined as follow.
\begin{itemize}
\item the \textit{state}-value function $\bar{V}^f: \bm{\mathcal{X}} \rightarrow \mathbb{R}$ is the expected discounted sum of future awards when the MDP is started by a certain state $x$ and the player follows policy $f$ throughout:
\begin{equation}
\bar{V}^{f}(x) = \e \left[\sum_{k=0}^{\infty} \gamma^{k} r(\xi_k) \mid x_{0}=x\right],
\end{equation}
and
\item the \textit{action}-value function $V^{f}: \bm{\mathcal{X}} \times \bm{\mathcal{A}} \rightarrow \mathbb{R}$ is the expected discounted sum of future awards when the MDP is started by a certain state-action pair $(x,a)$ and the player follows policy $f$ throughout:
\begin{equation}
\label{value function}
V^{f}(x,a) = \e \left[\sum_{k=0}^{\infty} \gamma^{k} r(\xi_k) \mid x_{0}=x, a_{0}=a\right].
\end{equation}
\end{itemize}
Both expectations are taken with respect to the induced Markov chain $(\bm{\M},f) := (\xi_k)_{k\geq 0} = (x_k, a_k)_{k\geq 0}$.
\end{definition}

\begin{remark}
These expectations are well-defined as $\gamma \in (0,1)$ as $r(\cdot)$ are bounded. (See, for example, the remark made at the beginning of Section 2 of \cite{wang2021global}).  In fact, we have the following trivial bounds: for all policies $f$ and state-action pairs $(x,a)$,
\begin{equation}
|\bar{V}^f(x)| \leq \frac{1}{1-\gamma}, \quad |V^f(x,a)| \leq \frac{1}{1-\gamma}. \label{eq:trivial_bound}
\end{equation}
\end{remark}

Furthermore, we define the state and state-action visiting measures of a policy $f$:
\begin{definition}[State and state-action visiting measures]
Let $(\bm{\M},f) := (x_k, a_k)_{k\geq 0}$ be the Markov chain induced when the policy $f$ acts on the MDP $\bm{\M}$. Let $\xi = (x,a) \in \bm{\X} \times \bm{\A}$ be a state-action pair of the MDP $\bm{\M}$. Denote
\begin{itemize}
\item $\mathbb{P}(x_k = x)$ be the probability that $x_k = x$ for $(\bm{\M},f)$, and 
\item $\mathbb{P}(x_k = x, a_k = a) := \mathbb{P}(x_k = x) f(x,a)$ be the probability that $x_k = x$ and $a_k = a$ for $(\bm{\M},f)$. 
\end{itemize}
With the above notations, we define the state and state-action visiting measures as $\nu_{\rho_0}^{f}$ and $\sigma_{\rho_0}^{f}$ respectively, such that
\begin{equation}
\label{visiting}
\nu_{\rho_0}^{f}(\{x\}) = \sum_{k=0}^{\infty} \gamma^{k} \prob\left(x_{k}=x\right), \quad \sigma_{\rho_0}^{f}(\{ (x, a) \})= \sum_{k=0}^{\infty} \gamma^{k} \prob\left(x_{k}=x, a_{k}=a\right),
\end{equation}
\end{definition}

\begin{remark} \phantom{=}
\begin{itemize}
\item Both $(1-\gamma) \nu^f_{\rho_0}(\cdot)$ and $(1-\gamma) \sigma^f_{\rho_0}(\cdot)$ are probability measures.
\item Define the auxiliary Markov chain induced when the policy $f$ acts on the MDP $\bm{\M}$ in terms of the state-action pair $\tilde{\xi}_k := (\tilde{x}_k, \tilde{a}_k)$:
\begin{equation}
(\bm{\M},f)_\aux: \begin{cases} 
\tilde{x}_0 &\sim \rho_0 \\
\tilde{a}_{k+1} &\sim f(\tilde{x}_k, \cdot) \\
\tilde{x}_{k+1} &\sim \tilde{p}(\cdot \mid \tilde{x}_k, \tilde{a}_k)
\end{cases}.
\end{equation}
where
\begin{equation}
\tilde{p} \left(\tilde{x}^{\prime} \mid \tilde{x}, \tilde{a} \right)=\gamma p\left(\tilde{x}^{\prime} \mid \tilde{x}, \tilde{a} \right) + (1-\gamma) \rho_0\left(\tilde{x}^{\prime}\right), \quad \forall\left(\tilde{x}, \tilde{a}, \tilde{x}^{\prime}\right) \in \bm{\X} \times \bm{\A} \times \bm{\X} \label{eq:tilde_p}
\end{equation}
sample from the initial distribution $\rho_0$ with probability $1-\gamma$ at each transition to a new state. Then $(1-\gamma) \sigma^f_{\rho_0}$ is the stationary measure of the auxiliary Markov chain $(\bm{\M},f)_\aux$. The above result is deduced on, for example, page 36 of \cite{konda2002actor}.
\end{itemize}
\end{remark}

Note that the auxiliary Markov chain $(\bm{\M},f)_{\aux}$ admits the following transition kernel
\begin{equation} \label{eq:transition_kernel_auxiliary}
\Pi_{f, \rho_0}((x,a), \, (x',a')) = f(x',a') \tilde{p}(x'\mid x,a).
\end{equation}
We write $\Pi_{f, \rho_0}$ as $\Pi_f$ when the context is clear. \\

We shall define the \textit{uniform policy} that assigns each action with equal probability
$$\mathsf{1}: (x,a) \in \bm{\X} \times \bm{\A} \mapsto \frac{1}{\#\bm{\A}}.$$

Moreover, we say that a policy $f$ is \textit{fully supported} if $f(x,a) > 0$ for any state-action pairs. Due to the finiteness of the state and action spaces, it is clear that $f$ is fully supported if and only if there exists $f_{\min} > 0$ such that
$$\forall (x,a), \quad f(x,a) \geq f_{\min} > 0.$$
Note that for $f$ to be a probability measure, we must have $f_{\min} \leq 1/(\#\bm{\A})$. \\

Finally, we say that a policy $f$ is \textit{deterministic} if, for each $x$, there is exactly an $a =: a_x$ such that $f(x,a) > 0$, in which case we must have $f(x,a_x) = 1$. \\

Having the notions of induced and auxiliary Markov chains defined when a policy acts on the MDP, we shall assume the following.

\begin{assumption} \label{as:ergodic_assumption}
We assume there exists $n_0 \in \N$ such that for any $n\geq n_0$, there is constant $C_{\mathsf{1}} \in (0,1/M)$ (with $M=\#\bm{\X} \times \#\bm{\A}$) such that
\begin{equation}
\inf_{x,a,x'} \sum_{\xi_1,...,\xi_{n_0}-1} p(x_1|x,a) ... p(x'|x_{n-1}, a_{n-1}) \geq C_{\mathsf{1}} (\#\bm{\A})^{n_0} > 0, \label{eq:minorisation_for_p}
\end{equation}
where $\xi_\ell = (x_\ell,a_\ell)$.
\end{assumption}

Note that $\tilde{p} \geq \gamma p$ by \eqref{eq:tilde_p}, so \eqref{eq:minorisation_for_p} implies
\begin{equation}
\inf_{\tilde{x},\tilde{a},\tilde{x}'} \sum_{\tilde{\xi}_1,...,\tilde{\xi}_{n_0}-1} \tilde{p}(\tilde{x}_1|\tilde{x},\tilde{a}) ... \tilde{p}(\tilde{x}'|\tilde{x}_{n_0-1}, \tilde{a}_{n_0-1}) \geq C_{\mathsf{1}} (\gamma \#\bm{\A})^{n_0} > 0,
\label{eq:minorisation_for_tilde_p}
\end{equation}
where $\tilde{\xi}_\ell = (\tilde{x}_\ell, \tilde{a}_\ell)$.

\begin{proposition} \label{prop:equiv_communication_1}
Assumption \ref{as:ergodic_assumption} is satisfied if either of the following holds:
\begin{enumerate}
\item The chain induced by the uniform policy $(\bm{\M}, \mathsf{1})$ is irreducible and aperiodic.
\item $\bm{\M}$ is \textit{communicating}, i.e., any state could be reached by one another using a deterministic policy.
\end{enumerate}
\end{proposition}

\begin{proof}
The first part is a direct application of \cite[Lemma 1.8.2]{Norris_1997} on the chain $(\bm{\M}, \mathsf{1})$. For the second part, see Appendix \ref{SS:communicating_MDP}, in which a more formal definition of MDP being communicating is given.
\end{proof}

The notion of communicating MDP is taken from \cite{Puterman2014markov, Kallenberg2002, Kallenberg2021}. \footnote{In \cite{Puterman2014markov}, the author uses \textit{recurrent} in place of \textit{irreducible}. This is because a Markov chain with finite state spaces is (positive) recurrent if and only if it is irreducible.} This is a standard assumption to be made in MDP analyses, see \cite{qiu2021, finite_time_2020}. An alternative assumption adopted in, e.g.  \cite{cayci2022finitetime}, would be to assume that it is able to sample from $(1-\gamma) \nu^f_{\rho_0}$ and therefore initialising the MDP with the $\rho_0 = (1-\gamma) \nu^f_{\rho_0}$. We do not use this assumption as $(1-\gamma) \nu^f_{\rho_0}$ is not available in practice. \\

Due to the finiteness of the state and action spaces, we are able to make the following notes on the generality of the assumption.

\begin{proposition} \label{prop:equiv_communication_2}
\phantom{blah}
\begin{enumerate}
\item Assumption \ref{as:ergodic_assumption} holds if and only if, for any fully supported policy $f$, the induced chain $(\bm{\M}, f)$ is irreducible and aperiodic.
\item If Assumption \ref{as:ergodic_assumption} is satisfied, then $(\bm{\M}, f)_{\aux}$ is also irreducible and aperiodic for any fully supported policies $f$. Moreover, $(\bm{\M}, f)$ and $(\bm{\M}, f)_{\aux}$ admit unique stationary measures $\pi^f$ and $\sigma^f_{\rho_0}$, respectively, and for any state-action pairs $\xi, \xi'$:
\begin{equation*}
\mathbb{P}^n_f(\xi,\xi') \overset{n\to\infty}\to \pi^f(\xi'), \quad \Pi^n_{f,\rho_0}(\xi, \xi') \overset{n\to\infty}\to (1-\gamma) \sigma^f_{\rho_0}(\xi'),
\end{equation*}
where $\mathbb{P}^n_f$ is the composition of $\mathbb{P}_f$ for $n$ times, and likewise for $\Pi^n_{f,\rho_0}$.
\end{enumerate}
\end{proposition}

\begin{proof}
See Appendix \ref{SS:communicating_MDP}.
\end{proof}

Finally, we note that if Assumption \ref{as:ergodic_assumption} is satisfied, then the stationary measure $\pi^f$ is locally Lipschitz (with respect to $f$), and $\sigma^f_{\rho_0}$ is globally Lipschitz. To define the notion of Lipschitzness, we recall the notion of total variation distance between measures: 

\begin{definition}[Total Variation Distance] \label{def:tv_distance}
The TV distance between two probability distributions on $\bm{\X} \times \bm{\A}$ with masses $p_1$ and $p_2$ are defined as
\begin{equation}
d_\TV(p_1, p_2) = \sum_{\xi \in \bm{\X}\times \bm{\A}} |p_1(\xi) - p_2(\xi)|.
\end{equation}
\end{definition}

\begin{proposition}[Lipschitzness of stationary measures]  \label{prop:Lipschitzness_of_stationary_measures} \phantom{blah}
\begin{enumerate}
\item If $f$ is a fully supported policy that satisfies $f(x,a) \geq f_{\min}$ for any state-action pair $(x,a)$ and $g$ is a fully supported policy, then for any $\xi$,
\begin{equation}
|\pi^f(\xi) - \pi^g(\xi)| \leq \frac{1}{MC_{\mathsf{1}} (f_{\min} \#\bm{\A})^{n_0}} \max_{\xi'} d_{\TV}(\mathbb{P}^{n_0}_f(\xi',\cdot), \mathbb{P}^{n_0}_g(\xi',\cdot)),
\end{equation}
where $n_0, C_{\mathsf{1}}$ are defined in Assumption \ref{as:ergodic_assumption}.
\item For any policies $f,g$,
\begin{equation}
|\sigma^f_{\rho_0}(\xi) - \sigma^g_{\rho_0}(\xi)| \leq \frac{1}{(1-\gamma)^2} \max_{\xi'} \sum_{\xi''=(x'',a'')} |f(\xi'') - g(\xi'')| p(x'' \mid x', a').
\end{equation}
\end{enumerate}
\end{proposition}

\begin{proof}
See Appendix \ref{SS:lipschitzness_of_the_stationary_measures}.
\end{proof}

\subsection{Online Neural Network Actor-critic Algorithm} \label{ActorCritic}

The main goal of reinforcement learning is to learn the optimal policy $f^*$ which maximizes the expected discounted sum of the future rewards:
\begin{equation}
\max\limits_{f} J(f),
\end{equation}
where the objective function $J(f)$ is the state-value function, weighted by the initial state-action pair:
\begin{equation}
J(f) = \e \left[\sum_{k=0}^{\infty} \gamma^{k}  r\left(x_{k}, a_{k}\right)\right]=\sum\limits_{x\in \bm{\mathcal{X}}} \rho_0(x) \bar{V}^{f}(x) = \sum\limits_{\xi = (x,a) \in \bm{\mathcal{X}}\times \bm{\mathcal{A}}} \sigma_{\rho_0}^{f}(\xi) r(\xi),
\end{equation}
see also equation (2.3) of \cite{wang2021global}. Policy-based reinforcement learning method optimize the objective function over a class of policies $\left\{f_{\theta} \mid \theta \in \bm{\Theta}\right\}$ based on the policy gradient theorem \cite{sutton2018reinforcement}. In practice, the value functions are unknown and must therefore also be estimated. In this paper, we study the \textit{online} actor-critic algorithms, which simultaneously estimate the action-value function using a \textit{critic} model and the optimal policy using an \textit{actor} model at every steps of the MDP:
\begin{itemize}
\item The \textit{actor model}, acting as the approximation of an optimal policy, is defined as 
\begin{equation}
f_\theta^{N}(\xi) = \textrm{Softmax}(P^N_\theta(\xi)) = \frac{\exp(P_\theta^{N}(x,a))}{\sum\limits_{a'}\exp(P_\theta^{N}(x,a'))}, \quad \xi = (x,a)
\end{equation}
where $P^N_\theta(\xi)$ is the \textit{actor network}:
\begin{equation}
\label{actor NN}
P_\theta^{N}(\xi) = \frac{1}{\sqrt{N}} \sum_{i=1}^N B^{i}
\sigma\left(U^i \cdot \xi \right),
\end{equation}
parametrised by the parameters $\theta = (B^1, \ldots, B^N, U^1, \ldots U^N)$, where $B^i \in \mathbb{R}$ and $U^i \in \mathbb{R}^d$.
\item The \textit{critic model}, acting as the approximation of the unknown action-value function of the optimal policy (approximated by the actor model), is the actual \textit{critic network}
\begin{equation}
\label{critic NN}
Q_\omega^{N}(\xi) = \frac{1}{\sqrt{N}} \sum_{i=1}^N C^{i} \sigma \bracket{W^i \cdot \xi},
\end{equation}
parametrised by the parameters $\omega = (C^1, \ldots, C^N, W^1, \ldots W^N)$, where $C^i \in \mathbb{R}$ and $W^i \in \mathbb{R}^d$.
\end{itemize}

The Softmax function is commonly used to convert vectors in a Euclidean space to probability vectors, in which entries are summed to one. \cite{Goodfellow-et-al-2016, sutton2018reinforcement, policygradient1999, tokic2011softmax}. Notice that $f^N_\theta(\xi) > 0$ for all $\xi$, so by Proposition \ref{prop:equiv_communication_2}, both $(\bm{\M}, f^N_\theta)$ and $(\bm{\M}, f^N_\theta)_{\aux}$ are irreducible and aperiodic, and that they admits unique stationary (and limiting) measures $\pi^{f^N_\theta}$ and $\sigma^{f^N_\theta}_{\rho_0}$ respectively.

\begin{remark}
We emphasise that
\begin{itemize}
\item The outputs of actor and critic networks $P^N_\theta, Q^N_\omega$ could be viewed as either functions on $\bm{\X} \times \bm{\A}$ or as vectors in $\R^M$ indexed by elements in $\bm{\X} \times \bm{\A}$, and
\item $f^N_\theta$ refers to the actor model (i.e., the \textit{probability distribution} output of the actor network), which could be viewed as either a function of $\bm{\X} \times \bm{\A}$ or as a vector in $\R^M$ indexed by elements in $\bm{\X} \times \bm{\A}$.
\end{itemize}
These interpretations are interchangeable.
\end{remark}

\begin{assumption}[Activation function] \label{as:activation_function}
The activation function $\sigma: \mathbb{R} \rightarrow \mathbb{R}$ is assumed to be a 
\begin{itemize}
\item twice continuously differentiable (i.e. in $C^2_b(\R)$) with outputs and derivatives bounded by $1$, and
\item is not a polynomial.
\end{itemize}
\end{assumption}
An example of $\sigma$ would be the standard sigmoid function $\sigma(x) = (1+e^{-x})^{-1}$. \\

We denote the $\theta_k = (B^1_k, \ldots, B^N_k, U^1_k, \ldots U^N_k)$ and $\omega_k = (C^1_k, \ldots, C^N_k, W^1_k, \ldots W^N_k)$ being the parameters of the actor and critic networks after $k$ iterations. We further define $P^N_k := P^N_{\theta_k}$, $f^N_k := \text{Softmax}(P^N_k)$ and $Q^N_k := Q^N_{\omega_k}$. \\

Our Actor-critic algorithm is online, which means that the policies used to sample state-action pairs in MDP will change at each iteration. With reference to \cite{konda1999actor, wang2019neural, xu2020nonasymptotic, wang2021global, wang2022continuous, wang2022forward}, our version of the Actor-Critic algorithm will sample two parallel sequences of state-action pairs:

\begin{itemize}
\item the ``actor" process (used for actor updates):
\begin{equation}
\label{artificial MDP samples}
(\bm{\M}, \mathsf{Ac}): \begin{cases} 
\tilde{x}_0 &\sim \rho_0 \\
\tilde{a}_{k} &\sim g^N_{k}(\tilde{x}_{k}, \cdot) \\
\tilde{x}_{k+1} &\sim \tilde{p}(\cdot \mid \tilde{x}_k, \tilde{a}_k),
\end{cases}
\end{equation}
and
\item the ``critic" process (use for critic updates):
\begin{equation}
\label{origin MDP samples}
(\bm{\M}, \mathsf{Cr}): \begin{cases} 
x_0 &\sim \rho_0 \\
a_k &\sim g^N_k(x_k, \cdot) \\
x_{k+1} &\sim p(\cdot \mid x_k, a_k), \\
\end{cases}
\end{equation}
\end{itemize}
where the \textit{exploration policies} $g^N_k$ is defined as 
\begin{equation}
\label{ActorwithExploration}
g_{k}^N(\xi) = \frac{\eta_k^N}{\#\bm{\A}} + (1 - \eta_k^N) \cdot f_{k}^N(\xi), \quad \xi = (x,a),
\end{equation}
and $(\eta_k^N)_{k\geq 0}$ is a sequence of exploration rates such that $0 < \eta_k^N \leq 1$ and $\eta_k^N \overset{k\to\infty}\to 0$, to be specified in Assumption \ref{as:learning rates}. The concept of exploration is developed to address the \textit{exploration-exploitation trade-off}, making sure that each actions in $\bm{\A}$ are considered with positive probabilities given any states when finding the optimal policy. \cite{ sutton2018reinforcement, watkins1989qlearning} This also provides convenience in our analysis, as by Assumption \ref{as:ergodic_assumption}, the induced Markov chains $(\bm{\M}, g^N_k)$ and $(\bm{\M}, g^N_k)_\aux$ could now be assumed to be ergodic, so that the stationary measures $\pi^{g^N_\theta}$ and $\sigma^{g^N_\theta}_{\rho_0}$ are well defined. This is further discussed in Lemma \ref{lem:geometric} and Appendix \ref{SS:communicating_MDP}. \\

With the above set-up, we state the two main steps of the online actor-critic algorithm taken at each iterations:

\paragraph{Step 1: updating the critic network:} We first update the critic using the Temporal Difference (TD) Learning \cite{watkins1992q}. TD learning could be viewed under the bi-level optimisation framework \cite{hong2023twotimescalecomplexity} as the minimisation of the \textit{critic loss}, known as the \textit{squared Bellman error}, with respect to the critic network parameters $\omega$:
\begin{equation} \label{eq:critic_loss}
L^{\theta}(\omega) := \sum_{\xi} \left[ Y^\theta_\omega(\xi) - Q^N_\omega(\xi) \right]^2 \pi^{f^N_{\theta}}(\xi),
\end{equation}
where the ``target" $Y^{\theta}_\omega$ is defined as 
\begin{equation}
Y^\theta_\omega(\xi) := r(\xi) + \gamma \sum_{x'} \left[\sum_{a'} Q^N_{\omega}(x',a') f^N_{\theta}(x',a')\right] p(x'|\xi)
\end{equation}
and $\pi^{f^N_\theta}$ is the unique stationary distribution of the Markov chain $(\bm{\M}, f^N_\theta)$ as specified in Assumption \ref{as:ergodic_assumption}. In fact, if $\pi^{f^N_\theta}(\xi) > 0$ for all $\xi \in \bm{\X} \times \bm{\A}$ and $L^\theta(\omega)=0$, then $Q^N_{\omega}(\xi)$ satisfies the Bellman equation and hence is a value function of $f^N_{\theta}$. \\

We typically approximate $L^\theta(\omega)$ by sampling from the stationary distribution $\pi^{f^N_\theta}$, as $\pi^{f^N_\theta}(\xi)$ is often inaccessible. For the \textit{online} version of the Actor-Critic algorithm, this could be done as the following: at step $k \geq 1$, we sample from the critic process $(\bm{\mathcal{M}}, \mathsf{Cr})$:
\begin{equation*}
x_k \sim p(\cdot \mid \xi_{k-1}), \quad a_k \sim g^N_k(x_k, \cdot).
\end{equation*}
This is in line with the observation that, for fixed $k$, the chain $(\M, g^N_k)$ admits the stationary distribution $\pi^{g^N_k}(\cdot)$, which is closed to $\pi^{f^N_k}(\cdot)$ when $k$ is large. With this, the loss function at step $k$ then could be approximated as
\begin{equation*}
L^{\theta_k}_{\omega_k} \approx \left[r(\xi_k) + \gamma \sum_{x'', a''} Q^N_{\omega_k}(x'',a'') f^N_{\theta_k}(x'',a'') p(x''| \xi_k) - Q^N_{\omega_k}(\xi_k) \right]^2.
\end{equation*}
We can further approximate the inner summation by taking sample $x_{k+1} \sim p(\cdot \mid x_k, a_k)$, so that 
\begin{align*}
L^{\theta_k}_{\omega_k} &\approx \Bigg[r(\xi_k) + \gamma \sum_{a''} Q^N_{\omega_k}(x_{k+1}, a'') f^N_{\theta_k}(x_{k+1},a'') - Q^N_{\omega_k}(\xi_k) \Bigg]^2 \\
&= \Bigg[r(\xi_k) + \gamma \underbrace{\sum_{a''} Q^N_k(x_{k+1}, a'') f^N_k(x_{k+1},a'')}_{(*)} - Q^N_k(\xi_k) \Bigg]^2.
\end{align*}
The $x_{k+1}$ sampled could be used for the approximation for $L^{\theta_{k+1}}_{\omega_{k+1}}$. Treating the $(*)$ term as constant, we shall update the critic parameters $\omega_k$ using stochastic gradient descent. This yields the following updates for the parameters in the critic parameters for rate $\alpha / N^{\lambda}$, where $\alpha, \lambda > 0$:
\begin{align}
C^{i}_{k+1} &= C^{i}_k + \frac{\alpha}{N^\lambda} \left[\frac{1}{\sqrt{N}} \left(r(\xi_k) + \gamma \sum_{a''} Q^N_k(x_{k+1}, a'') g^N_k(x_{k+1},a'') - Q^N_k(\xi_k) \right) \sigma \left(W^i_k \cdot \xi_k \right) \right], \nonumber \\
W^i_{k+1} &= W^i_k + \frac{\alpha}{N^\lambda} \left[\frac{1}{\sqrt{N}}  \left( r(\xi_k) + \gamma \sum_{a''} Q^N_k(x_{k+1}, a'') g^N_k(x_{k+1},a'') - Q^N_k(\xi_k) \right) C^{i}_k \sigma' \left( W^i_k \cdot \xi_k \right) \xi_k \right]. \label{NACCriticupdates}
\end{align}

\paragraph{Step 2: updating the actor network:} We then use the policy gradient theorem \cite{policygradient1999} to update the actor for a new policy $f^N_{k+1}$. The policy gradient theorem states that if a policy $f_\theta$ is parametrised by $\theta$, then 
\begin{equation}
\nabla_\theta \bar{V}^{f_\theta}(x) = \sum_{x',a} \left[ V^{f_\theta}(x',a) \nabla_\theta f_\theta(x',a) \right] \nu^{f_\theta}_x(x').
\end{equation}
If $f_\theta(x,a) > 0$, then we have
\begin{equation}
\nabla_\theta J(\theta) = \sum_{x,a} \left[ \nabla_\theta (\ln f_\theta(x,a)) V^{f_\theta}(x,a) \right] \sigma^{f_\theta}_{\rho_0}(x,a). \label{eq:policy_gradient}
\end{equation}
Further explanations are provided in Appendix \ref{S:policy_gradient}. \\

Note that the above is an expectation of the quantity $\nabla_\theta (\ln f_\theta(x,a)) V^{f_\theta}(x,a)$ with respect to the visiting measure $\sigma^{f_\theta}_{\rho_0}(\cdot)$. As mentioned earlier, this expectation is often calculated in the literature by sampling directly from the stationary distribution $\sigma^{f^N_\theta}$; however, in practice, $\sigma^{f^N_{\theta}}(\xi)$ is typically unknown and therefore it is not possible to directly sample from the stationary distribution. For the \textit{online} version of the Actor-Critic algorithm, a sequence from a Markov chain is generated which has an appropriate stationary distribution to approximate $\sigma^{f^N_\theta}$ \cite{konda1999actor, wang2019neural, xu2020nonasymptotic}. At step $k\geq 1$, we sample from the actor process $(\bm{\M}, \mathsf{Ac})$:
\begin{equation*}
\tilde{x}_k \sim \tilde{p}(\cdot \,|\, \tilde{\xi}_{k-1}), \quad \tilde{a}_{k} \sim g^N_k(\tilde{x}_{k}, \cdot).
\end{equation*}
This is because, with a fixed $k$, the process $(\bm{\M}, g^N_k)$ admits the stationary distribution $\sigma^{g^N_k}(\cdot)$, which will appropriately approximate $\sigma^{f^N_k}(\cdot)$ when $k$ is large. With this, the policy gradient at step $k$ could be approximated as
\begin{equation*}
\nabla_\theta J(\theta_k) \approx \nabla_\theta \ln f^N_{\theta_k}(\tilde{\xi}_k) V^{f^N_{\theta_k}}(\tilde{\xi}_k), \quad \tilde{\xi}_k = (\tilde{x}_k, \tilde{a}_k). 
\end{equation*}
The Actor-Critic algorithm takes a further step and approximates the action-value function $V^{f^N_{\theta_k}}(\cdot, \cdot)$ with the critic network, as defined in Step 1. In our paper, we shall clip the critic network with the clipping function $\clip(\cdot)$, where
\begin{equation}
\clip(z) = \max\left(\min\left(z, \frac{2}{1-\gamma} \right), 0\right),
\end{equation}
so that the policy gradient estimates now become
\begin{equation*}
\nabla_\theta J(\theta_k) \approx \nabla_\theta \ln f^N_{\theta_k}(\tilde{\xi}_k) \clip(Q^N_{\omega_k}(\tilde{\xi}_k)) = \nabla_\theta \ln f^N_k(\tilde{\xi}_k) \clip(Q^N_k(\tilde{\xi}_k)). 
\end{equation*}
This is to avoid the size of parameter updates $B^i_k$ and $U^i_k$ exploding, and, in turn, ensuring the well-posedness of the limiting ODE \eqref{NN gradient flow} of the parameter updates. \\

To obtain the actual parameter updates, let us compute the partial derivatives of the actor model $f^N_\theta = \mathrm{Softmax}(P^N_\theta)$ with respect to the parameters $\theta$:
\begin{align*}
\frac{d}{dB^i} \ln f^N_\theta(x,a) &= \frac{d}{dB^i} \bracket{P^N_\theta(x,a) - \ln\bracket{\sum_{a'} \exp\bracket{P^N_\theta(x,a')}}} \\
&= \frac{d}{dB^i} \bracket{f^N_\theta(x,a)} - \frac{\sum_{a'} \frac{d}{dB^i} \exp\bracket{P^N_\theta(x,a')}}{\sum_{a''} \exp\bracket{P^N_\theta(x,a'')}} \\
&= \frac{d}{dB^i} \bracket{f^N_\theta(x,a)} - \frac{\sum_{a'} \exp\bracket{P^N_\theta(x,a')} \frac{d}{dB^i} P^N_\theta(x,a')}{\sum_{a''} \exp\bracket{P^N_\theta(x,a'')}} \\
&= \frac{1}{\sqrt{N}} \sigma(U^i \cdot (x,a)) - \sum_{a'} \bracket{\frac{\exp\bracket{P^N_\theta(x,a')}}{\sum_{a''} \exp\bracket{P^N_\theta(x,a'')}} \frac{1}{\sqrt{N}} \sigma(U^i \cdot (x,a')) }\\
&= \frac{1}{\sqrt{N}} \bracket{\sigma(U^i \cdot (x,a)) - \sum_{a'} f^N_\theta(x,a') \sigma(U^i \cdot (x,a'))}, \numberthis
\end{align*}
and, similarly,
\begin{align*}
\nabla_{U^i}(\ln f^N_\theta(x,a)) = \frac{1}{\sqrt{N}} \bracket{B^i \sigma'(U^i \cdot (x,a)) (x,a) - \sum_{a'} f^N_\theta(x,a') B^i \sigma'(U^i \cdot (x,a')) (x,a')}.
\end{align*}
With the above calculations, we end up with following updates of the parameters for the online Actor-Critic algorithm for learning rate $\zeta^N_k / N^{\lambda}$:
\begin{align*}
B^i_{k+1} &= B^i_k + \frac{\zeta^N_k}{N^{\lambda}} \left[\frac{1}{\sqrt{N}} \clip(Q^N_k(\tilde{\xi}_k)) \bracket{\sigma(U^i_k \cdot (\tilde{\xi}_k)) - \sum_{a''} f^N_{k}(\tilde{x}_k, a'') \sigma({U^i_k} \cdot (\tilde{x}_k, a''))} \right], \\
U^i_{k+1} &= U^i_k + \frac{\zeta^N_k}{N^{\lambda}} \left[\frac{1}{\sqrt{N}} \clip(Q^N_k(\tilde{\xi}_k)) {B^i_k} \bracket{\sigma'(U^i_k \cdot (\tilde{\xi}_k)) (\tilde{\xi}_k) - \sum_{a''} f^N_{k}(\tilde{x}_k,a'') \sigma'(U^i_k \cdot (\tilde{x}_k,\tilde{a}_k)) (\tilde{x}_k,a'')} \right] \numberthis \label{NACActorupdates}
\end{align*}

The online Actor-Critic algorithm is summarised in Algorithm \ref{alg:onlineNAC}. \\

\begin{algorithm}[ht]
\caption{Online Actor-Critic Algorithm with Neural Network Approximation}\label{alg:onlineNAC}
\begin{algorithmic}[1]
\Procedure{onlineNAC}{$\bm{\M}, N, T, \nu_0, \mu_0$} \Comment{MDP, network size, running time, initial distributions of critic and actor parameters.}
    \State initialise neural network parameters: $\forall i, (C^i_0, W^i_0) \overset{\text{iid}}\sim \nu_0$ and $(B^i_0, W^i_0) \overset{\text{iid}}\sim \mu_0$.
    \State set $k = 0$
    \State initialise states/actions $x_0, \tilde{x}_0 \sim \rho_0$,
    \While{$k \leq NT$}
    \State simulate $a_k \sim g^N_k(x_k, \cdot)$ and $\tilde{a}_k \sim g^N_k(\tilde{x}_k, \cdot)$
    \State simulate $x_{k+1} \sim p(\cdot \,|\, \xi_k)$ and $\tilde{x}_{k+1} \sim \tilde{p}(\cdot \,|\, \tilde{\xi}_k)$
    \ForAll{$i \in \{1,2,...,N\}$}
        \State obtain $(C^i_{k+1}, W^i_{k+1})$ from \eqref{NACCriticupdates} using $x_k$, $a_k$, $x_{k+1}$ and $(C^i_k, W^i_k)$ \Comment{Temporal difference}
        \State obtain $(B^i_{k+1}, U^i_{k+1})$ from \eqref{NACActorupdates} using $\tilde{x}_k, \tilde{a}_k$ and $(B^i_k, U^i_k)$ \Comment{Policy gradient}
    \EndFor
    \EndWhile
\EndProcedure
\end{algorithmic}
\end{algorithm}

The main contribution of this paper is to show that,  under the scaling $\lambda = 1$, the evolution of the ``actor" and the ``critic" network according to this online Actor-Critic algorithm converges weakly to the solution of a limiting ODE over the time interval $[0,T]$ when embedded in the space of c\`adl\`ag processes. We then study the convergence of the online Actor-Critic algorithm via this limiting ODE.

\begin{remark}
In practical implementation, both the ``actor" and ``critic" networks should contain bias parameters, and should be written in the form
\begin{equation}
\frac{1}{\sqrt{N}} \sum_{i=1}^N C^i \sigma(\mathrm{weight}^i \cdot (x,a) + \mathrm{bias}^i),
\end{equation}
where $\mathrm{bias}^i \in \R$. The bias parameter could be incorporated into the weight vectors by introducing an additional column of $1$ in the state vector $x$, so that the networks could be expressed as 
\begin{equation}
\frac{1}{\sqrt{N}} \sum_{i=1}^N C^i \sigma(\widetilde{\mathrm{weight}}^i \cdot (x',a)), \quad x' = (x,1).
\end{equation}
\end{remark}

\paragraph{Choice of Learning and Exploration Rates}
As we use  the critic network $Q^N_k$ to estimate  $\nabla_\theta J(f_{\theta_k})$, it is reasonable to assume that the critic network updates more frequently than the actor network. We therefore assume a constant learning rate $\alpha/N^{1+\lambda}$ for the critic network and a decaying learning rate $\zeta^N_k/N$ for the actor network, where $\zeta^N_k$ decays as $k\to \infty$ under the Robin-Munro conditions \cite{RobinMunro1951}. This is slightly different from \cite{konda2002actor, Konda2003}, where the learning rates for both the actor and critic networks satisfies the Robin-Munro conditions. \\

Exploration rates $\eta^N_k$ also have to be chosen carefully, to ensure that a sufficient amount of exploration has been carried out around the state action space to find the optimal policy. In fact, we shall see in Theorem \ref{thm:global_critic_convergence} that $\eta^N_k$ also governs the convergence of the critic network to a valid value function. \\ 

The choice of learning and exploration rates is summarised as follows.

\begin{assumption} \label{as:learning rates}
\label{learning}
We assume the actor learning rate is 
\begin{align}
\label{eq:actor_learning_rate}
\zeta_k^N &= \frac{1}{(1+\frac{k}{N})^\beta}, \quad 0 < \beta \leq 1.
\end{align}
Furthermore, we assume the exploration rate to be \emph{either}
\begin{align}
\label{eq:learning rates}
\eta_k^N =  \frac{1}{1+\log^2(1+\frac{k}{N})}, \quad \textit{or} \quad \eta_k^N  = \frac{1}{(1+\frac{k}{N})^\varepsilon}, \quad \text{for } \varepsilon > 0. 
\end{align}
The hyper-parameters $\beta$ and $\varepsilon$ are to be specified.
\end{assumption}

We note that for fixed $t > 0$, we have
$$\zeta^N_{\floor{Nt}} \rightarrow \zeta_t = \frac{1}{(1+t)^\beta},$$
and
$$\quad \eta^N_{\floor{Nt}} \rightarrow \eta_t = \frac{1}{1+\log^2(t+1)}, \quad \textit{or} \quad \frac{1}{(1+t)^\varepsilon}.$$

\begin{remark}[Batch/offline version of Actor-Critic Algorithm]
It is also common to study versions of the actor-critic algorithm, where the critic network performs a large number of updates to fully minimise the critic loss function in \eqref{eq:critic_loss}, before taking an update of the actor network. Examples include \cite{qiu2021, xu2020nonasymptotic, fu2021singletimescaleactorcriticprovablyfinds, KumarHarshat2023Otsc, cayci2022finitetime}. However, these versions of algorithm are either no longer online in nature, or would require significantly more samples to achieve convergence in practice. 
\end{remark}

\section{Main Result}
\label{s:main_result}
Our results are proven under some assumptions for the neural networks, MDP and learning rates. 
\begin{assumption}
\label{as:NN_condition}
For the ``actor" network in \eqref{actor NN} and ``critic" network in \eqref{critic NN}, we assume:
\begin{itemize}
\item the randomly initialized parameters $\left(C_{0}^{i}, W^i_0, B^i_0, U^i_0 \right)$ are independent and identically distributed (i.i.d.) mean-zero random variables for all $i$ with distribution $\nu_{0}(dc,dw) \otimes \mu_0(db,dw)$, where $\otimes$ refers to the product of measures. Furthermore, we assume that both $\nu_0$ and $\mu_0$ are absolutely continuous with respect to the Lebesgue measure,
\item the random variables $\max(|C_0^i|, |B_0^i|) \leq 1$ for each $i$, and $\max(\E\|W^i_0\|, \E\|U^i_0\|) \leq 1$, and
\item the distribution of $W^i_0$ and $U^i_0$ both have \textit{full} support on $\R^M$.
\end{itemize}
We assume further that $\nu_0 = \mu_0$ for simplicity, although this addition assumption could be easily removed.
\end{assumption} 

We prove that the outputs of the actor and critic models converge to the solution of a nonlinear ODE system as the numbers of hidden units of the neural networks $N \to \infty$. We define the empirical measure
\begin{equation}
\mu_{k}^{N}=\frac{1}{N} \sum_{i=1}^{N} \delta_{B_{k}^i, U_k^i}, \quad  \nu_{k}^{N}=\frac{1}{N} \sum_{i=1}^{N} \delta_{C_{k}^i, W_k^i}.
\end{equation}
In addition, we define the scaled evolutions for any $\xi = (x,a) \in \bm{\X} \times \bm{\A}$
\begin{align*}
P^N_t(\xi) &= P_{\floor{Nt}}^{N}(\xi), \quad f_{t}^{N}(\xi) = f_{\floor{Nt}}^{N}(\xi), \quad g_{t}^{N}(\xi) = g_{\floor{Nt}}^{N}(\xi), \\
Q^N_t(\xi) &= Q_{\floor{Nt}}^{N}(\xi), \quad \mu_{t}^{N} = \mu_{\floor{Nt}}^{N}, \quad  \nu_{t}^{N} = \nu_{\floor{Nt}}^{N}. \numberthis \label{eq:rescaled_updates}
\end{align*}
Using Assumptions \ref{as:activation_function} and \ref{as:NN_condition}, we know that $\mu_{0}^{N}, \nu_{0}^{N} \stackrel{d}{\rightarrow} \nu_0$ and $P_0^N, Q_{0}^{N} \stackrel{d}{\rightarrow} \mathcal{G}, \mathcal{H}$ as $N \rightarrow \infty$, where $\mathcal{G}, \mathcal{H}$ are mean-zero Gaussian random variable by the law of large numbers and central limit theorem for i.i.d. random variables, respectively. \\

Define the state space for the rescaled process $\left(\mu_{t}^{N}, \nu_t^N, P_{t}^{N}, Q_t^N \right)$ of the parameter updates as specified in \eqref{eq:rescaled_updates}:
\begin{equation}
    E = \mathcal{M}(\R^{1+d}) \times \mathcal{M}(\R^{1+d}) \times \R^M \times \R^M, \quad d = d_x + d_a, \quad M = |\mathcal{X} \times \mathcal{A}|,
\end{equation}
where $\mathcal{M}(\R^{1+d})$ is the set of all probability measures on $\R^{1+d}$. Define further the space
\begin{equation}
    D_E([0, T]) = \{\text{c\`adl\`ag paths } f: [0,T] \to E\}.
\end{equation}
The convergence of the time-rescaled process \eqref{eq:rescaled_updates} as $N\to\infty$ shall be studied in the space $D_E([0,T])$.

\begin{remark} \label{rmk:choice_of_norm}
We note that the choice of norm/distance to study the pre-limit processes $(P^N_t, Q^N_t)$ in Theorem \ref{limit odes} does not matter as $(P^N_t, Q^N_t) \in \R^{2M}$ is finite-dimensional. The choice of norm for Theorem \ref{thm:global_critic_convergence} does not matter for the same reason. We will use $\|\cdot\|_\infty$ as the supremum norm as defined in \eqref{critic convergence}
\begin{equation*}
\|P - \tilde{P}\|_\infty = \max_{\xi \in \bm{\X}\times \bm{\A}} |P(\xi) - \tilde{P}(\xi)|
\end{equation*}
and the usual Euclidean norm
\begin{equation*}
\|P - \tilde{P}\| = \bracket{\sum_{\xi \in \bm{\X}\times \bm{\A}} |P(\xi) - \tilde{P}(\xi)|}^{1/2}
\end{equation*}
Note that the Softmax function is Lipschitz in the following sense: there exist constants $C, C' >0$ such that for $P, \tilde{P} \in \R^M$, 
\begin{equation}
    d_\TV(\mathrm{Softmax}(P), \mathrm{Softmax}(\tilde{P})) \leq C' \|P - \tilde{P}\|_\infty.
\end{equation}
\end{remark}

For notational convenience, we make use of the notion of the canonical pairing between a measurable function $f$ and a measure $\nu$,
\begin{equation*}
    \la f, \nu \ra = \int f \, d\nu,
\end{equation*}
to define the kernel matrix $A(\nu)$ that will appears in our limit ODE in Theorem \ref{limit odes}. The kernel matrix $A(\nu)$ shall depend on the measure $\nu$, and to be indexed by the state-action pairs $\xi := (x,a)$ and $\xi' := (x',a')$ as the following:
\begin{align*} \label{eq:kernel_matrix}
[A(\nu)]_{\xi, \xi^{\prime}} &= \left\langle\sigma\left(w \cdot \xi^{\prime}\right) \sigma(w \cdot \xi) + c^{2} \sigma^{\prime}\left(w \cdot \xi^{\prime}\right) \sigma^{\prime}(w \cdot \xi) (\xi \cdot \xi^{\prime}), \nu \right\rangle \\
&= \int \sigma\left(w \cdot \xi^{\prime}\right) \sigma(w \cdot \xi) + c^{2} \sigma^{\prime}\left(w \cdot \xi^{\prime}\right) \sigma^{\prime}(w \cdot \xi) (\xi \cdot \xi^{\prime}) \, \nu(dc,dw) \numberthis
\end{align*}

The convergence of the online actor-critic algorithm is characterised by the following theorem:
\begin{theorem}[Scaled Limit Evolution of the Critic and Actor Networks]
\label{limit odes}
Assume assumptions \ref{as:MDP_basic}, \ref{as:ergodic_assumption}, \ref{as:learning rates}, \ref{as:activation_function}, \ref{as:NN_condition}, and let the scaling of the parameter updates be $\lambda = 1$. Then the process $\left(\mu_{t}^{N}, \nu_t^N, P_{t}^{N}, Q_t^N \right)$ converges weakly in the space $D_E([0, T])$ as $N \rightarrow \infty$ to the process $\left(\mu_{t}, \nu_t, P_{t}, Q_t\right)$, so that for any $t \in[0, T]$, any $(x, a) \in \mathcal{X} \times \mathcal{A}$, and for every $\varphi, \bar{\varphi} \in C_{b}^{2}\left(\mathbb{R}^{1+d}\right)$, the limit process $\left(\mu_{t}, \nu_t, P_{t}, Q_t\right)$ satisfies the random ODE:
\begin{align*}
\frac{dQ_t}{dt}(\xi) &= \alpha \sum_{\xi' = (x',a')} A_{\xi,\xi'} \left(r(\xi') + \gamma \sum_{z, a''} Q_t(z, a'') g_t(z,a'') p(z| \xi') - Q_t(\xi') \right) \pi^{g_t}(\xi'),\\
\frac{dP_t}{dt}(\xi) &= \sum_{\xi' = (x', a')} \zeta_t \clip(Q_t(\xi')) \left[A_{\xi,\xi'} - \sum_{a''} f_t(x',a'')  A_{\xi,x',a''} \right] \sigma_{\rho_0}^{g_t}(\xi'), \\
P_0(\xi) &= \mathcal{G}(\xi), \quad Q_0(\xi) = \mathcal{H}(\xi) \\
\left\langle \bar{\varphi}, \mu_{t}\right\rangle &=\left\langle \bar{\varphi}, \nu_{0}\right\rangle, \quad \left\langle \varphi, \nu_{t}\right\rangle = \left\langle \varphi, \nu_{0}\right\rangle, \numberthis \label{NN gradient flow}
\end{align*}
where 
\begin{itemize}
\item $A$ is the kernel matrix $A(\nu_0)$, as defined in \eqref{eq:kernel_matrix},
\item $\mathcal{G}, \mathcal{H}$ are the weak limits of $P^N_0$ and $Q^N_0$, which are mean-zero Gaussian random variables,
\item the policies $f_t$ and $g_t$ are defined as
\begin{equation*}
f_t(\xi) = \mathrm{Softmax}(P_t(\xi)), \quad g_t(\xi) = \frac{\eta_t}{\#\bm{\A}} + (1-\eta_t) f_t(\xi),
\end{equation*}
and 
\item the limiting measure of the parameters $\mu_t, \nu_t$ does not evolve, i.e., both empirical measures of parameters $\mu^N_t, \nu^N_t$ converge weakly to their initial distribution $\nu_0$.
\end{itemize}
\end{theorem}

Let us compare this limiting ODE with the prequel of this paper \cite{wang2021global}, which studies the \textit{tabular} version of the online actor-critic algorithm (i.e., no function approximations have been applied to the actor and critic). If we define $P^N_k$ and $Q^N_k$ as the outputs of the actor and critic networks respectively, and $P^N_t, Q^N_t$ are the scaled outputs as defined in \eqref{eq:rescaled_updates}; then as $N\to\infty$, $(P^N, Q^N)$ converges to $(P,Q)$ that evolves with the following ODE. \footnote{This is a slightly modified version of the original ODE as provided in \cite{wang2019neural} due to the application of the clipping function on $Q_t$ and the fact that $g_t$ is used in $(\bm{\M}, \mathsf{Ac})$.}
\begin{align*}
\frac{dQ_t}{dt}(\xi) &= \alpha \left(r(\xi) + \gamma \sum_{z, a''} Q_t(z, a'') g_t(z,a'') p(z| \xi) - Q_t(\xi) \right) \pi^{g_t}(\xi),\\
\frac{dP_t}{dt}(\xi) &= \zeta_t  \left[\clip(Q_t(\xi)) - \sum_{a'} \clip(Q_t(x,a')) f_t(x,a') \right] \sigma_{\rho_0}^{g_t}(\xi'). \numberthis \label{eq:limit_ODE_without_function_approx}
\end{align*}
The convergence as establishes in \cite{wang2019neural} \textit{uniformly} holds in the finite interval $t\in [0,T]$ as $N\to\infty$. This is not possible in our case due to the NTK normalisation of the neural networks, so instead of having $(P^N_0, Q^N_0) \to 0$, we have $P^N_0$ converges weakly to a Gaussian process indexed by $\xi$. This also means that $(P^N_t, Q^N_t)_{t\in [0,T]}$ cannot converge uniformly over the finite interval $t \in [0,T]$. The paper utilises the techniques provided in \cite{sirignano2021asymptotics} to establish the weak and pathwise convergence of $(P^N_t, Q^N_t)_{t\in [0,T]}$. As the technique does not provide a quantitative estimate of the difference between the initialisation $(P^N_0, Q^N_0)_{t\in [0,T]}$ and the limiting Gaussian processes $(\mathcal{G}, \mathcal{H})$, so it would not be possible to provide quantitative estimates for the convergence of $(P^N_t, Q^N_t)_{t\in [0,T]}$. We shall, therefore, not include the study of the behaviour of actor-critic algorithm in the case when $N$ is large but not infinite in this work. \\

The main difference between this ODE and our new ODE \eqref{NN gradient flow} is the emergence of the kernel matrix $A$, known as the neural tangent kernel (NTK). It appears due to the choice of using neural networks for the actor and critic with an NTK normalisation, so that the outputs evolve along their \textit{linearisation} around the initial distribution of the neural network parameters before training begins \cite{NTK1, NTK2}. The interactions of the network outputs $P_t(\xi)$ as a result of the NTK is the main source of the technical challenges when proving if the limit actor-critic algorithm converges; the necessary analysis and calculations significantly change from the tabular  actor critic algorithm \cite{wang2021global}. The dynamics of the limit ODE depend upon a matrix multiplication with the kernel matrix $A$, leading to the evolution of the critic and actor at state-action pair $\xi$ depending upon the critic values at all other points $\xi'$. \\

We note that the matrix $A := A(\nu_0)$ is positive definite. This is recently established in \cite{Carvalho2024NTKpositive}. \footnote{Note that \cite{Carvalho2024NTKpositive} studies neural networks when $(C^i_0, W^i_0)$ is initialised by a joint Gaussian distribution. However, the actual proof only requires that the distribution of $W^i_0$ be fully supported in $\R^M$.}
\begin{lemma}
Under assumptions \ref{as:MDP_basic} and \ref{as:activation_function}, the matrix $A$ is positive definite.
\end{lemma}

We prove the convergence of the ``critic" network to a true action-value function and the convergence of the ``actor" network to a stationary point of the objective function.

\begin{theorem}[Critic Convergence] \label{thm:global_critic_convergence}
Assume assumptions \ref{as:MDP_basic}, \ref{as:ergodic_assumption}, \ref{as:activation_function}, and  \ref{as:NN_condition} hold. If both the actor network $P_t$ and critic network $Q_t$ evolved according to the limit ODE \eqref{NN gradient flow}, then the critic network converges globally to the value function of the policy $f_t = \mathrm{Softmax}(P_t)$ as $t \to \infty$:
\begin{itemize}
\item For the case when $\eta_t = (1+(\log(1+t))^2)^{-1}$,
\begin{equation}
\label{critic convergence}
\|Q_t - V^{f_t} \|^2 = \sum_\xi |Q_t(\xi) - V^{f_t}(\xi)|^2 = O\left(\frac{(1+(\log(1+t))^2)^{n_0-2}}{1+t}\right),
\end{equation}
\item For the case when $\eta_t = (1+t)^{-\varepsilon}$, if $0 < n_0\varepsilon < \beta_a \leq 1$, then
\begin{equation}
\label{critic convergence 2}
\|Q_t - V^{f_t} \|^2 = \sum_\xi |Q_t(\xi) - V^{f_t}(\xi)|^2 = O\left(\frac{1}{(1+t)^{2((\beta_a-n_0\varepsilon) \wedge \varepsilon)}}\right),
\end{equation}
\end{itemize}
where $n_0$ is defined in assumption \ref{as:ergodic_assumption}.
\end{theorem}

\begin{theorem} \label{thm:global_actor_convergence_rate}
Under the same settings as Theorem \ref{thm:global_critic_convergence}:
\begin{itemize}
\item if $\beta=1$ and $\eta_t = (1+(\log(1+t))^2)^{-1}$, then there exists a constant $C > 0$ such that
\begin{align}
\inf_{t\in[0,T]} \|\nabla_P J(f_t)\|^2 \leq \frac{\int_0^T\zeta_t \|\nabla_P J(f_t)\|^2  \, dt}{\int_0^T \zeta_t \, dt} 
 &\leq \frac{C}{\log(1+T)}; \label{eq:inf_norm_rate} 
\end{align}
\item if $\beta \in (\frac{n_0+1}{n_0+3}, 1)$ and $\eta_t = (1+t)^{-\varepsilon}$, where $\beta + 2((\beta-n_0\varepsilon) \wedge \varepsilon) > 1$, then there exists a constant $C > 0$ such that
\begin{equation}
\inf_{t\in[0,T]} \|\nabla_P J(f_t)\|^2 \leq \frac{\int_0^T\zeta_t \|\nabla_P J(f_t)\|^2  \, dt}{\int_0^T \zeta_t \, dt} \leq \frac{C}{(1+T)^{1-\beta}-1}. \label{eq:inf_norm_rate_2} 
\end{equation}
\end{itemize}
\end{theorem}

\begin{theorem}[Actor Convergence]
\label{thm:global_actor_convergence}
Under the same settings as Theorem \ref{thm:global_critic_convergence}, the actor network converges to a stationary point,
\begin{equation}
\label{actor convergence}
\nabla_{P} J(f_t) \overset{t\to\infty}\to 0,
\end{equation}
under either of the following conditions:
\begin{itemize}
\item $\beta=1$ and $\eta_t = (1+(\log(1+t))^2)^{-1}$,
\item $\beta \in (\frac{n_0+1}{n_0+2}, 1)$ and $\eta_t = (1+t)^{-\varepsilon}$, where $\beta + ((\beta-n_0\varepsilon) \wedge \varepsilon) > 1$.
\end{itemize}
\end{theorem}

Under certain assumptions, we can relate the convergence of the gradient with the regret bound $J(f^*) - J(f_t)$, where $f^*$ is an optimal policy. We shall follow the approach of \cite{wang2019neural} and make a similar assumption as in their paper regarding the absolute continuity of the optimal policy $f^*$ with respect to $f_t = \mathrm{Softmax}(P_t)$.

\begin{assumption}[Absolute continuity of optimal policy, adapted directly from  {\cite[Lemma C.3]{wang2019neural}}] \label{as:absolute_continuity_of_optimal_policy}
Let 
\begin{equation}
u_t(x,a) = \frac{\sigma^{f^*}_{\rho_0}(x,a)}{\sigma^{f_t}_{\rho_0}(x,a)} - \frac{\nu^{f^*}_{\rho_0}(x)}{\nu^{f_t}_{\rho_0}(x)}.
\end{equation}
We assume that there exist a $\tilde{T} > 0$ and $\bar{R}_0 > 0$ such that, for all $\xi$ and $t > \tilde{T}$,
\begin{equation}
\|u_t\| \leq \bar{R}_0.
\end{equation}
\end{assumption}

\begin{lemma} \label{lem:regret_and_gradient}
Under the same setting as Theorem \ref{thm:global_critic_convergence}, if assumption \ref{as:absolute_continuity_of_optimal_policy} holds, then for any $t \geq \tilde{T}$,
\begin{equation}
J(f^*) - J(f_t) \leq \bar{R}_0 \|\nabla_P J(f_t)\|.
\end{equation}
\end{lemma}

\begin{proof}
By the Performance Difference Lemma, see e.g. \cite[Lemma 5.1]{wang2019neural}, we have
\begin{align*}
J(f^*) - J(f_t) &= \sum_{x,a} \nu^{f^*}_{\rho_0}(x) A^{f_t}(x,a) (f^*(x,a)-f_t(x,a)) \\
&= \sum_{x,a} \sigma^{f_t}_{\rho_0}(x,a) A^{f_t}(x,a) \frac{\nu^{f^*}_{\rho_0}(x)}{\sigma^{f_t}_{\rho_0}(x,a)} (f^*(x,a)-f_t(x,a)) \\
&= \sum_{x,a} \sigma^{f_t}_{\rho_0}(x,a) A^{f_t}(x,a) \left[\frac{\sigma^{f^*}_{\rho_0}(x,a)}{\sigma^{f_t}_{\rho_0}(x,a)} - \frac{\nu^{f^*}_{\rho_0}(x)}{\nu^{f_t}_{\rho_0}(x)}\right] \\
&= \sum_{x,a} \sigma^{f_t}_{\rho_0}(x,a) A^{f_t}(x,a) u_t(x,a) \\
&= \nabla_P J(f_t)^\top u_t \\
&\leq \bar{R}_0 \|\nabla_P J(f_t) \|.
\end{align*}
\end{proof}

Combining Theorems \ref{thm:global_actor_convergence_rate} and \ref{thm:global_actor_convergence} with Lemma \ref{lem:regret_and_gradient}, we have the following main results:

\begin{theorem} \label{thm:global_optimality}
Under the same setting as Theorem \ref{thm:global_critic_convergence} and assumption \ref{as:absolute_continuity_of_optimal_policy} holds,
\begin{itemize}
\item if $\beta=1$ and $\eta_t = (1+(\log(1+t))^2)^{-1}$, then $J(f^*) - J(f_t) \overset{t\to\infty}\to 0$, and there exists $C>0$ such that
\begin{equation}
\inf_{t \in [0,T]} [J(f^*) - J(f_t)] \leq \frac{C}{\log(1+T)};
\end{equation}
\item if $\beta \in (\frac{n_0+1}{n_0+2}, 1)$ and $\eta_t = (1+t)^{-\varepsilon}$, where $\beta + ((\beta-n_0\varepsilon) \wedge \varepsilon) > 1$, then $J(f^*) - J(f_t) \overset{t\to\infty}\to 0$, and there exists $C>0$ such that
\begin{equation}
\inf_{t \in [0,T]} [J(f^*) - J(f_t)] \leq \frac{C}{(1+T)^{1-\beta} - 1}.
\end{equation}
\end{itemize}
\end{theorem}

Theorem \ref{thm:global_optimality} asserts that the output limit of the actor policy converges to the optimal policy, as long as $\sigma^{f_t}_{\rho_0}(\xi)$ is uniformly bounded away from zero for any $\xi$ supporting $\sigma^{f^*}_{\rho_0}(\cdot)$. This is assumption is directly adapted from \cite{wang2019neural}. Although the assumption cannot necessarily be verified in practice, it provides interesting insights into the optimality of neural actor-critic algorithms, and it furthermore provides a clear relationship between the mismatch of the visiting measures and the size of regret $J(f^*) - J(f_t)$. \\

We make a several remarks regarding the results above:
\begin{remark}
\begin{itemize}
\item[] 
\item The term
$$\frac{\int_0^T\zeta_t \|\nabla_P J(f_t)\|^2  \, dt}{\int_0^T \zeta_t \, dt},$$
as first introduced in \cite{xu2020nonasymptotic}, could be interpreted as an expectation $\mathbb{E}_{\tilde{t}}[\|\nabla_P J(f_{\tilde{t}})\|^2]$, where $\tilde{t}$ is a random variable on $[0,T]$, independent of the initial Gaussian processes $\mathcal{G}$ and $\mathcal{H}$, such that its distribution admits the density $\tilde{t} \sim \rho(t) := \zeta_t \left(\int_0^T \zeta_s \, ds \right)^{-1}$ over $s \in [0,T]$.
\item We have provided two convergence rates in Theorem \ref{thm:global_actor_convergence}, depending on whether $n_0$ (from assumption \ref{as:ergodic_assumption}) is known. The first convergence rate is slower-than-polynomial (where the exploration does not depend on $n_0$) while the second convergence rate is a faster polynomial convergence rate (where the exploration rate depends on $n_0$).
\item The first convergence rate \eqref{eq:inf_norm_rate}, given under the actor learning rate $\zeta_t = (1+t)^{-1}$ and the exploration rate $\eta_t = (1+(\log(1+t))^2)^{-1}$, is independent of $n_0$ as defined in Assumption \ref{as:ergodic_assumption}. This gives a ``universal" statement that convergence of the algorithm is guaranteed regardless of $n_0$, i.e., how quickly the Markov decision process mixes. The drawback is that the actor convergence rate is of order $(\log(1+t))^{-1}$, which is much slower than the standard convergence rates for non-convex optimisation in the literature.
\item The second convergence rate \eqref{eq:inf_norm_rate_2}, given under the alternative learning rate $\zeta_t = (1+t)^{-\beta}$ and the exploration rate $\eta_t = (1+t)^{-\varepsilon}$, is a much faster polynomial convergence rate. However, the analysis relies on $n_0$ being known for choosing an appropriate $\beta$ and $\varepsilon$, which might not be available in practice. Readers should take note of the following two limiting cases:
\begin{enumerate}
\item if $n_0 = 1$, then we can choose $\beta$ arbitrarily close to $\frac{\beta_0+1}{\beta_0+3} = \frac{1}{2}$, so that the decay rate of $\min_{t \in [0,T]} \|\nabla_P J(f_t)\|^2$ gets arbitrarily close to the rate $O((1+T)^{-1/2})$. This is similar to the results in \cite{wang2019neural}, where this faster convergence rate is achieved due to the work assuming that it is possible to directly sample from the stationary measures).
\item as $n_0 \nearrow +\infty$ increases, we must have $\beta \nearrow 1$ and $\varepsilon \searrow 0$. We can therefore view the convergence rate in \eqref{eq:inf_norm_rate} as an extension of the convergence rate in \eqref{eq:inf_norm_rate_2} for (arbitrarily) large $n_0$.
\end{enumerate}

Although the exploration rate depends on $n_0$, \eqref{eq:inf_norm_rate_2} remains useful in practice, as it reveals that a power-law convergence rate for $\min_{t \in [0,T]} \|\nabla_P J(f_t)\|^2$ is achievable. In particular, optimal hyper-parameters (exploration rate $\varepsilon$ and learning rate $\beta$) could be obtained by hyperparameter evaluation/experimentation with a series of experiments with different pairs of $(\beta,\varepsilon)$. Such hyperparameter experiments are standard when building deep learning/AI models, and the convergence result indicates that a power-law convergence rate could be achieved in practice via such hyperparameter exploration. Furthermore, if some knowledge of the MDP is available, then the theory provides a simple mathematical formula determining the necessary exploration rate and the final convergence rates. 

\item Ultimately, the choice of exploration rate restricts the final actor convergence rate as compared to the standard non-convex optimisation literature. As we shall reveal in the proof of Theorem \ref{thm:global_critic_convergence}, the exploration is in place ($a_k \sim g^N_k(x_k, \cdot)$ instead of $a_k \sim f^N_k(x_k, \cdot)$) to ensure that all state-action pairs are visited sufficiently frequently, so that the stationary measure $\pi^{g_t}(\cdot)$ is minorised by $C_{\mathsf{1}}\eta_t^{n_0} > 0$, and that the critic network converges. The amount of exploration decays to zero as the training time $\rightarrow \infty$, at a rate asymptotically slower than the learning rate for the optimisation steps. If the exploration rate decays too quickly, we would essentially be taking optimisation steps ``before the Markov chain reached its stationary distribution", and thus using an incorrect gradient even for large $t$, and then the critic will not converge. This is the reason why we have a slower-than-polynomial convergence rate $O(\log(1+t))^{-1}$ when $\eta_t = (1+(\log(1+t))^2)^{-1}$, and a power-rule convergence rate $O((1+t)^{1-\beta})$ when $\eta_t = (1+t)^{-\varepsilon}$.
\item Finally, we highlight that the convergence result $\nabla_P J(f_t) \overset{t\to\infty}\to 0$ (proven in Theorem \ref{thm:global_actor_convergence}) is stronger than the convergence $\min_{t \in [0,T]} \|\nabla_P J(f_t) \|^2 \overset{T\to\infty}\to 0$ (proven in Theorem \ref{thm:global_actor_convergence_rate}), where the latter result is often typical in non-convex optimisation (e.g., \cite{wang2019neural, xu2020nonasymptotic}). Theorem \ref{thm:global_actor_convergence} states that the actor-critic algorithm converges even \emph{without} early stopping. 
\end{itemize}
\end{remark}

\section{Derivation of the limit ODEs}
\label{s:derivation_of_the_limit_ODEs}

\hspace{1.4em} Building upon \cite{sirignano2021asymptotics}, we prove convergence to the limit ODEs via the following steps:
\begin{enumerate}
\item We first derive pre-limit process for the outputs of the actor and critic networks, and a-priori controls on the size of increments of parameters. The pre-limit process will contain stochastic remainder terms from the non-i.i.d. data samples.
\item We prove the relative compactness of the pre-limit process, which requires proof of the compact containment and regularity of the sample paths.
\item We then use the Poisson equation to derive the limit ODEs by showing the fluctuation terms in the limit ODEs disappears as $N\to+\infty$.
\item As a sanity check, we prove the the existence and uniqueness of the limits ODEs.
\item Finally, we combine the above preparations to conclude the convergence in Theorem \ref{limit odes}.
\end{enumerate}

\subsection{Evolution of the Prelimit Processes}

\hspace{1.4em} Before diving into the technical details of proof, we first provide some intuitions on the limit ODEs of the neural actor-critic algorithm (algorithm \ref{alg:onlineNAC}). We clarify the following notations. 
\begin{definition} \label{def:stochastic_boundedness}
For a random variable $Z_N$,
\begin{itemize}
\item $Z_N = O_p(\beta_N)$ if $Z_N/\beta_N$ is \textit{stochastically} bounded, i.e. for any $\epsilon > 0$, there exists $M<\infty$ and some $N_0 < \infty$ such that 
$$
\p\left(\left|\frac{Z_N}{\beta_N}\right| > M\right) < \epsilon, \quad \forall N > N_0.
$$
\item The notation $Z_N = O(\beta_N)$ means there exists a constant $C<\infty$ independent of $N$ such that 
$$
|Z_N| \le C |\beta_N|, \quad \forall N.
$$
\end{itemize}
\end{definition}
In the following parts of the paper, constants $C, C_T$ denote generic constants and we will sometimes use $\xi, \xi_k, \xi'_k, \tilde \xi_k$ to denote the state-action pair $(x, a), (x_k, a_k), (x'_k, a'_k), (\tilde x_k, \tilde a_k)$ respectively. With the choice of scaling $\lambda = 1$, one could derive the following (pre-limit) evolution of the outputs of the actor and critic networks.

\begin{align}
Q^N_{k+1}(\xi) &= Q^N_k(\xi)+\frac{\alpha}{N} \left[ r(\xi_k) + \gamma \sum_{a''} Q^N_k(x_{k+1},a'') g^N_k(x_{k+1},a'') - Q^N_k(\xi_k) \right] [A(\nu^N_k)]_{\xi,\xi'} + \mathrm{error\;term}. \label{eq:evolution_of_q_1} \\
P^N_{k+1}(\xi) &= P^N_k(\xi) + \frac{\zeta^N_k}{N} \clip(Q^N_k(\tilde{\xi}_k)) \sqbracket{ [A(\mu^N_k)]_{\xi, \tilde{\xi}_k} - \sum_{a''} f^N_k(\tilde{x}_k, a'') [A(\mu^N_k)]_{\xi,(\tilde{x}_k,a'')}} + \mathrm{error\;term}. \label{eq:evolution_of_p_with_tensor}
\end{align}

The size of the error terms are formally bounded by the following proposition:

\begin{proposition}[Pre-limit evolution of the actor and critic networks] \label{prop:pre-limit evolution}
For $k \leq NT$, the evolution of the critic network yields,
\begin{align*}
\E\sqbracket{\max_\xi \abs{Q^N_{k+1}(\xi) - Q^N_k(\xi) - \frac{\alpha}{N} \bracket{ r(\xi_k) + \gamma \sum_{a''} Q^N_k(x_{k+1},a'') g^N_k(x_{k+1},a'') - Q^N_k(\xi_k)} [A(\nu^N_k)]}_{\xi,\xi_k}} \leq \frac{C_T}{N^{5/2}},
\end{align*}
while the evolution of the actor network yields
\begin{align*}
\max_\xi \abs{P^N_{k+1}(\xi) - P^N_k(\xi) - \frac{\zeta^N_k}{N} \clip(Q^N_k(\tilde{\xi}_k)) \bracket{ [A(\mu^N_k)]_{\xi, \tilde{\xi}_k} - \sum_{a''} f^N_k(\tilde{x}_k, a'') [A(\mu^N_k)]_{\xi,(\tilde{x}_k,a'')}}} \leq \frac{C_T}{N^{5/2}}.
\end{align*}
\end{proposition}

\begin{proof}
See Appendix \ref{S:pre_limit_evolution}.
\end{proof}

\subsection{Relative Compactness} \label{relative compactness section}

Recall $E = \M(\R^{1+d}) \times \M(\R^{1+d}) \times \R^M \times \R^M$, where $\M(\R^{1+d})$ space of all probability measures on $\R^{1+d}$. We shall equip $E$ with the metric:
$$d((\mu,\nu,P,Q), (\tilde{\mu}, \tilde{\nu}, \tilde{P}, \tilde{Q})) = d_{\mathrm{LP}}(\nu, \tilde{\nu}) + d_{\mathrm{LP}}(\mu, \tilde{\mu}) + \|P-\tilde{P}\|_\infty + \|Q - \tilde{Q}\|_\infty,$$
where $d_{\mathrm{LP}}(\cdot, \cdot)$ is the usual L\'evy–Prokhorov metric used for metricising weak convergence \cite{billingsley2013convergence}. Let further $D_E([0,T])$ be the space of all c\'adl\'ag paths from $[0,T]$ to $E$, equipped with the Skorohod metric. Then, the family of processes $X^N_t := \left(\mu^N_t, \nu_t^N, P_t^N, Q_t^N \right)$ could be seen as elements of $D_E([0,T])$, where
\begin{align*}
P^N_t(\xi) &= P_{\lfloor Nt \rfloor}^{N}(\xi), \quad f_{t}^{N}(\xi) = f_{\lfloor Nt \rfloor}^{N}(\xi), \quad g_{t}^{N}(\xi) = g_{\lfloor Nt \rfloor}^{N}(\xi), \\
Q^N_t(\xi) &= Q_{\lfloor Nt \rfloor}^{N}(\xi), \quad \mu^N_t = \mu_{\lfloor Nt \rfloor}^{N}, \quad \nu^N_t = \nu_{\lfloor Nt \rfloor}^{N}.
\end{align*}

In this Section, we shall outline the proof for the following result:

\begin{proposition} \label{prop:relative_compactness}
The distribution of the processes $(X^N_t)$ is relative compact in $D_E([0,T])$.
\end{proposition}

We shall follow the steps in \cite{sirignano2021asymptotics}. By Theorem 8.6 and Remark 8.7 of \cite{ethier2009markov}, it would be sufficient if we establish that the process $X^N$ is \textit{ exactly contained} and has regular sample paths. These results are summarised by the following. Firstly, for the compact containment of the process - 

\begin{lemma}[Compact Containment] \label{lem:compact_containment}
For any $\eta > 0$, there is a compact subset $\K$ of $E$ such that
\begin{equation*}
\sup_{N \in \N, 0 \leq t \leq T} \p\sqbracket{\bracket{\mu^N_t, \nu_t^N, P_t^N, Q_t^N} \notin \K} < \eta.
\end{equation*}
\end{lemma}

\begin{proof}
See Appendix \ref{SS:proof_compact_containment}.
\end{proof}

For the regularities of the sample paths, we denote (as in \cite{sirignano2021asymptotics}) $\F^N_t$ as the $\sigma$-algebra generated by $\set{(C_0^i, W^i_0, B^i_0, U^i_0)}_{i=1}^N$ and $\set{\left(\xi_j, \tilde \xi_j\right)}_{j=0}^{\floor{Nt}-1}$. Then

\begin{lemma} \label{lem:regularity_of_nu}
Let $f \in C^2_b(\R^{1+d})$ and $q = \min(|z_1 -z_2|, 1)$. For any $\delta \in (0,1)$, there is a constant $C_T < \infty$ such that for $u \in [0,\delta]$, $t \in [0,T]$,
\begin{align}
\E\sqbracket{q\bracket{\la f, \nu^N_{t+u}\ra, \la f, \nu^N_t\ra} \,|\, \F^N_t} &\leq C_T\delta + \frac{C_T}{N^{3/2}}, \\
\E\sqbracket{q\bracket{\la f, \mu^N_{t+u}\ra, \la f, \mu^N_t\ra} \,|\, \F^N_t} &\leq C_T\delta + \frac{C_T}{N^{3/2}}.
\end{align}
\end{lemma}

\begin{proof}
See Appendix \ref{SS:proof_of_path_regularities}.
\end{proof}

\begin{lemma} \label{lem:regularity_of_Q}
We have 
\begin{equation}
    \sup_{k\leq NT} \max\bracket{\E\sqbracket{\sup_\xi |Q^N_{k+1}(\xi) - Q^N_k(\xi)|}, \sup_\xi |P^N_{k+1}(\xi) - P^N_k(\xi)|} \leq \frac{C_T}{N}.
\end{equation}
We abuse notation and now define $q(z_1, z_2) = \norm{z_1 - z_2}_\infty \wedge 1$ for $z_1, z_2 \in \R^M$. With a more delicate analysis, we could show that for any $\delta \in (0,1)$, there is a $C_T<\infty$ such that for $0 \leq u \leq \delta < 1$, $t \in [0,T]$,
\begin{align*}
\E\bracket{q\bracket{Q^N_{t+u}, Q^N_t} \,|\, \F^N_t} \leq C_T \delta + \frac{C_T}{N}, \\ 
\E\bracket{q\bracket{P^N_{t+u}, P^N_t} \,|\, \F^N_t} \leq C_T \delta + \frac{C_T}{N}.
\end{align*}
\end{lemma}

\begin{proof}
See Appendix \ref{SS:evolution_of_empirical_measure} for the derivation of the evolution of the empirical measures, and Appendix \ref{SS:proof_of_path_regularities} for the final steps of the proof.
\end{proof}

\begin{proof}[Proof of Proposition \ref{prop:relative_compactness}]
By Theorem 8.6 and Remark 8.7 of \cite{ethier2009markov},
\begin{itemize}
    \item combining Lemma \ref{lem:compact_containment} and Lemma \ref{lem:regularity_of_nu} yields the weak convergence of $(\mu^N, \nu^N)$ in $D_{\M(\R^{1+d}) \times \M(\R^{1+d})}([0,T])$,
    \item combining Lemma \ref{lem:compact_containment} and Lemma \ref{lem:regularity_of_nu} yields the weak convergence of $(P^N, Q^N)$ in $D_{\R^M \times \R^M}([0,T])$.
\end{itemize}
Therefore, the convergence of the joint process $X^N$ follows.
\end{proof}

\subsection{Identification of the Limit}

\hspace{1.4em} With the relative compactness result in Section \ref{relative compactness section}, we can conclude that $(\mu^N, \nu^N, P^N, Q^N)$ contains a subsequence that converges weakly. To prove the convergence in Theorem \ref{limit odes}, we need to identify the limit points, that is, to derive the ODEs dynamic for the limit point(s).

Our first step is to further rewrite \eqref{eq:evolution_of_q_1} and \eqref{eq:evolution_of_p_with_tensor} in terms of fluctuation terms. Define
\begin{align*}
M^{1,N}_t(\xi) &= -\frac{1}{N}\sum_{k=0}^{\floor{Nt}-1} Q^N_k(\xi_k) [A(\nu^N_k)]_{\xi, \xi_k} + \frac{1}{N}\sum_{k=0}^{\floor{Nt}-1}
\sum_{\xi'} Q^N_k(\xi') [A(\nu^N_k)]_{\xi, \xi'} \pi^{g_k^N}(\xi'),  \\
M^{2,N}_t(\xi) &= \frac{1}{N}\sum_{k=0}^{\floor{Nt}-1} r(\xi_k) [A(\nu^N_k)]_{\xi, \xi_k} - \frac{1}{N}\sum_{k=0}^{\floor{Nt}-1}
\sum_{\xi'} r(\xi') [A(\nu^N_k)]_{\xi, \xi'} \pi^{g_k^N}(\xi'), \\
M^{3,N}_t(\xi) &= \frac{1}{N}\sum_{k=0}^{\floor{Nt}-1} \gamma \bracket{\sum_{a''} Q^N_k(x_{k+1},a'') g^N_k(x_{k+1},a'')} [A(\nu^N_k)]_{\xi, \xi_k} \\
&\phantom{=}- \frac{1}{N}\sum_{k=0}^{\floor{Nt}-1} \sum_{\xi'} \sum_{z, a''} \gamma Q^N_k(z, a'') g^N_k(z, a'') [A(\nu^N_k)]_{\xi, \xi'} \pi^{g_k^N}(\xi') p(z| \xi'), \numberthis \label{eq:MNt}
\end{align*}
then 
\begin{align*}
Q_t^N(\xi) 
&= Q_0^N(\xi) + \sum_{k=0}^{\floor{Nt}-1} [Q^N_{k+1}(\xi) - Q^N_k(\xi)] \\
&= Q_0^N(\xi) + \frac{\alpha}{N} \sum_{k=0}^{\floor{Nt}-1} \left[ r(\xi_k) + \gamma \sum_{a''} Q^N_k(x_{k+1},a'') g^N_k(x_{k+1},a'') - Q^N_k(\xi_k) \right] [A(\nu^N_k)]_{\xi,\xi_k} + O_p(N^{-3/2}) \\
&= Q_0^N(\xi) + \frac{\alpha}{N} \sum_{k=0}^{\floor{Nt}-1} \sum_{\xi'} [A(\nu^N_k)]_{\xi,\xi'} \pi^{g^N_k}(\xi') \bigg( -Q^N_k(\xi') + r(\xi') + \gamma \sum_{z,a''} Q^N_k(z,a'') g^N_k(z,a'') p(z|\xi') \bigg) \\
&\phantom{=}+ \alpha \left( M^{1,N}_t(\xi) + M^{2,N}_t(\xi) + M^{3,N}_t(\xi) \right) + O_p(N^{-3/2}) \\
&= Q_0^N(\xi) + \frac{\alpha}{N} \sum_{k=0}^{\floor{Nt}-1} \int_{k/N}^{k+1/N} \sum_{\xi'} [A(\nu^N_{\floor{Ns}})]_{\xi,\xi'} \, \pi^{g^N_{\floor{Ns}}}(\xi') \bigg( r(\xi') + \gamma \sum_{z,a''} Q^N_{\floor{Ns}}(z,a'') g^N_{\floor{Ns}}(z,a'') p(z|\xi') \\
&\phantom{=}- Q^N_{\floor{Ns}}(\xi') \bigg) \, ds + \alpha \left(M^{1,N}_t(\xi) + M^{2,N}_t(\xi) + M^{3,N}_t(\xi) \right) + O_p(N^{-3/2}) \\
&= Q_0^N(\xi) + \alpha \int_0^t \sum_{\xi'} [A(\nu^N_s)]_{\xi, \xi'} \pi^{g_s^N}(\xi') \left[r(\xi') + \gamma \sum_{z, a''} Q^N_s(z, a'') g^N_s(z, a'') p(z|\xi') - Q^N_s(\xi') \right] \, ds\\
&\phantom{=}+ \alpha \left( M^{1,N}_t(\xi) + M^{2,N}_t(\xi) + M^{3,N}_t(\xi) \right) + O_p(N^{-3/2}). \numberthis \label{eq:Q_pre_limit}
\end{align*}
Similarly, define the fluctuation terms
\begin{align*}
M^{N}_t(\xi) &= \frac{1}{N}\sum_{k=0}^{\floor{Nt}-1} \zeta^N_k \clip(Q^N_k(\tilde{\xi}_k)) \left[[A(\mu^N_k)]_{\xi,\tilde{\xi}_k} - \sum_{a''} f^N_k(\tilde{x}_k,a'') [A(\mu^N_k)]_{\xi,(\tilde{x}_k,a'')} \right] \\
&- \frac{1}{N}\sum_{k=0}^{\floor*{Nt}-1} \zeta^N_k \sum_{\xi'} \clip(Q^N_k(\xi')) 
\left[[A(\mu^N_k)]_{\xi,\xi'} - \sum_{a''} f^N_k(x', a'') [A(\mu^N_k)]_{\xi,(x', a'')} \right] \sigma_{\rho_0}^{g^N_k}(\xi'), \numberthis \label{concentration2}
\end{align*}
where $\sigma_{\rho_0}^{g_k^N}(\xi')$ is the visiting measure of Markov chain as defined in \eqref{artificial MDP samples}. Then: 
\begin{align*}
P^N_t(\xi) &= P^N_0(\xi) + \sum_{k=0}^{\floor{Nt}-1} (P^N_{k+1}(\xi) - P^N_k(\xi)) \\
&= P^N_0(\xi) + \sum_{k=0}^{\floor{Nt}-1} \frac{\zeta^N_k}{N} \clip(Q^N_k(\tilde{\xi}_k)) \sqbracket{[A(\mu^N_k)]_{\xi, \tilde{\xi}_k} - \sum_{a''} f^N_k(\tilde{x}_k, a'') [A(\mu^N_k)]_{\xi,(\tilde{x}_k,a'')}} + O(N^{-3/2}) \\
&= P^N_0(\xi) + \sum_{k=0}^{\floor{Nt}-1} \zeta^N_k \sum_{\xi'} \clip(Q^N_k(\xi')) 
\left[[A(\mu^N_k)]_{\xi,\xi'} - \sum_{a''} f^N_k(x', a'') [A(\mu^N_k)]_{\xi,(x', a'')} \right] \sigma_{\rho_0}^{g_k^N}(\xi') \\
&\phantom{=}+ \alpha M^N_t(x, a) + O(N^{-3/2}) \\
&= P^N_0(\xi) + \sum_{k=0}^{\floor{Nt}-1} \int_{k/N}^{(k+1)/N} \zeta^N_{\floor{Ns}} \sum_{\xi'} \clip(Q^N_{\floor{Ns}}(\xi')) 
\Bigg[[A(\mu^N_{\floor{Ns}})]_{\xi,\xi'} \\
&\phantom{=}- \sum_{a''} f^N_{\floor{Ns}}(x', a'') [A(\mu^N_{\floor{Ns}})]_{\xi,(x', a'')} \Bigg] \sigma_{\rho_0}^{g^N_{\floor{Ns}}}(\xi') 
+ \alpha M^N_t(\xi) + O(N^{-3/2}) \\
&= P^N_0(\xi) + \int_0^t \zeta^N_{\floor{Ns}} \sum_{\xi'} \sigma_{\rho_0}^{g^N_s}(\xi') \, \clip(Q^N_s(\xi')) \Bigg[[A(\mu^N_s)]_{\xi, \xi'} - \sum_{a''} f^N_s(x',a'') [A(\mu^N_s)]_{\xi, (x', a'')} \Bigg] ds \\
&\phantom{=}+ M^N_t(\xi) + O(N^{-3/2}). \numberthis 
\end{align*}

We are now getting closer to the desired ODE form. We shall now prove that the fluctuation terms $M^N_t, M^{1,N}_t, M^{2,N}_t, M^{3,N}_t \to 0$ in $L^1$ as $N\to\infty$, uniformly over $t \in [0,T]$. Then, the desired convergence follows from the uniqueness of the limit ODEs. \\

We first define some notations:

\begin{itemize}
\item For any $k \ge 0$ and state-action pairs $\xi = (x,a)$ and $\tilde{\xi} = (\tilde{x},\tilde{a})$, define the transition probabilities
\begin{align}
\mathbb{P}^N_{k}(\xi,\xi') &:= \mathbb{P}_{g^N_k}(\xi,\xi') =  p(x'|x, a) g^N_k(x', a'), \nonumber \\
\Pi^N_{k}(\tilde{\xi}, \tilde{\xi}') &= \Pi_{g^N_k, \rho_0}(\tilde{\xi}, \tilde{\xi}') = \tilde{p}(\tilde{x}'| \tilde{x}, \tilde{a}) g^N_k(\tilde{x}', \tilde{a}'), \label{2dim transition}
\end{align}
where $\mathbb{P}_f$ and $\Pi_{f,\rho_0}$ are defined in \eqref{eq:transition_kernel_induced} and \eqref{eq:transition_kernel_auxiliary} respectively. The transition probabilities shall induce operators acting on any (Borel) function 
\begin{align}
\mathbb{P}_k^N h(\xi) &:= \sum_{\xi' \in \bm{\mathcal{X} \times \mathcal{A}}} h(\xi') \mathbb{P}_k^N(\xi, \xi') \nonumber \\ 
\Pi_k^N h(\tilde{\xi}) &:= \sum_{\xi' \in \bm{\mathcal{X} \times \mathcal{A}}} h(\tilde{\xi}') \Pi_k^N(\tilde{\xi}, \tilde{\xi}'), \label{transition integral}
\end{align}
We highlight the superscript $N$ in the transition probabilities $\prob^N_{k}, \Pi^N_k$ comes from the pre-limit neural network $P_k^N$. 
\item Let $\pi^{g^N_k}$ and $\sigma_{\rho_0}^{g^N_k}$ be the stationary distributions of the induced Markov chain $(\bm{\M}, g^N_k)$ and the auxiliary Markov chain $(\bm{\M}, g^N_k)_{\aux}$, respectively, whose existence and uniqueness are given by the assumption \ref{as:ergodic_assumption}. Sometimes we suppress the dependence of $\rho_0$ in $\sigma_{\rho_0}^{g^N_k}$ when there is no confusion.
\item Let $\mathscr{F}_{n}$ be the $\sigma$-field of events generated by the Markov chains $  \xi_1, \cdots, \xi_n, \tilde{\xi}_1, \ldots, \tilde{\xi}_n$ in \eqref{origin MDP samples} and \eqref{artificial MDP samples}.
\end{itemize}

\subsubsection{Poisson Equations}
\label{SS:Poisson_equation}
\hspace{1.4em} Now we rigorously derive the limit ODEs by using a Poisson equation \cite{pardoux2001poisson, wang2021global, wang2022continuous, wang2022forward}, which can be comprehended as the limit of the Kolmogorov forward equation (Fokker-Planck equation \cite{liu2021investigating, liu2022rigorous, pavliotis2016stochastic}) for stochastic process, to bound the fluctuations terms around the trajectory of the limit ODE. Such analysis is needed as the fluctuation terms evolve as the actor and critic networks evolve, which further depend on the non-i.i.d data samples from the Markov chains \eqref{artificial MDP samples} and \eqref{origin MDP samples}. We first prove
\begin{equation}
\lim_{N\to \infty}\e \sup_{t\in [0,T]}\left| M_t^N(x,a) \right| = 0, \quad \forall (x,a) \in \bm{\mathcal{X}} \times \bm{\mathcal{A}}.
\end{equation}
Using a similar method, we can also prove the convergence of $M_t^1, M_t^2,$ and $M_t^3$. 

It is well known that an irreducible and aperiodic finite-state Markov chain has a geometric convergence rate to its stationary distribution \cite{meyn2012markov}. In particular, if we define the $n$-step transition probabilities:
\begin{equation*}
\prob^{N,n}_k := \prob^n_{g^N_k}, \quad \Pi^{N,n}_k := \Pi^n_{g^N_k},
\end{equation*}
where $\prob_f$ and $\Pi_f$ are defined in \eqref{eq:transition_kernel_induced} and \eqref{eq:transition_kernel_auxiliary} respectively, then

\begin{lemma} \label{lem:geometric}
The chains $(\bm{\M}, g^N_k)$ and $(\bm{\M}, g^N_k)_{\aux}$ are uniformly geometric ergodic: for any $\xi \in \bm{\mathcal{X} \times \mathcal{A}}$:
\begin{align}
\sup_{k \le NT} \| \prob^{N,n}_k(\xi, \cdot) - \pi^{g^N_k}(\cdot) \|_{\TV} &\leq  (1-MC_{\mathsf{1}} \eta_T^{n_0})^{\lfloor n/n_0 \rfloor}, \label{eq:geometric_MDP} \\
\sup_{k \le NT} \| \Pi^{N, n}_k(\tilde{\xi}, \cdot) - \sigma^{g^N_k}(\cdot) \| &\leq  (1-MC_{\mathsf{1}} (\gamma \eta_T)^{n_0})^{\lfloor n/n_0 \rfloor}, \label{eq:geometric_aux}
\end{align}
where $C_{\mathsf{1}} > 0$ is a constant, and $M = \#\bm{\X} \times \#\bm{\A}$. 
\end{lemma}

\begin{proof}
See Appendix \ref{SS:communicating_MDP}.
\end{proof}

In order to prove the stochastic fluctuation term vanishes as $N \rightarrow \infty$, we solve several Poisson equations corresponding to the induced and auxiliary Markov chains. Let us begin by formulating the Poisson equation for the auxiliary Markov chain $(\bm{\M}, g^N_k)_{\aux}$, which relates the transition kernel $\Pi^N_k$ with its unique stationary distribution $\sigma^{g^N_k}_{\rho_0}$:

\begin{definition}
Let $N \in \mathbb{N}$, $T> 0$ and $k \le NT$. The Poisson equation corresponding to the chain $(\M, g^N_k)_{\aux}$ seeks a function $\nu_k(\tilde{\xi}; \cdot) : \bm{\mathcal{X}} \times \bm{\mathcal{A}} \to \R$ for each state-action pair $\tilde{\xi} = (\tilde{x},\tilde{a})$, such that 
\begin{equation}
\label{poisson}
\nu_{k}(\tilde{\xi}; \tilde{\xi}') - \Pi^N_{k} \nu_{k}(\tilde{\xi},\tilde{\xi}') =  \mathsf{1}_{\{ \tilde{\xi}' = \tilde{\xi}\}} - \sigma^{g^N_k}_{\rho_0}(\tilde{\xi}), \quad \forall \tilde{\xi}' \in \bm{\mathcal{X}\times \mathcal{A}}.
\end{equation}
Here we abuse the notation and use $\Pi^N_{k}$ to denote the transition operator that sends $\nu_k$ to
\begin{equation*}
\Pi^{N,n}_{k} \nu_{k}(\tilde{\xi};\tilde{\xi}') = \sum_{y} \nu_{k}(\tilde{\xi}; y)\Pi^N_{k}(\tilde{\xi}', y).
\end{equation*}
\end{definition}

\begin{remark}
Note that the Poisson equation corresponds to the Markov chain $(\bm{\M}, g^N_k)_{\aux}$ instead of the ``actor" process $(\bm{\M}, \mathrm{Ac})$.
\end{remark}

\begin{lemma}
\label{poisson equation}
The Poisson equation \eqref{poisson} admits a uniformly bounded solution \footnote{Solution to the Poisson equation \eqref{poisson} is not unique.} For the purposes of our subsequent analysis, it is only necessary to find a uniformly bounded solution $\nu_\theta$ which satisfies \eqref{nu}.
\begin{equation}
\label{nu}
\nu_{k}(\tilde{\xi};\tilde{\xi}') := \sum_{n\ge 0} \left[\Pi^{N, n}_{k} (\tilde{\xi}', \tilde{\xi}) - \sigma^{g^N_k}_{\rho_0}(\tilde{\xi}) \right],
\end{equation}
and there exists a constant $C_T$ (which only depends on $T$) such that
\begin{equation}
\label{eq:uniform}
\sup_{k \le NT} \left|\nu_{k}(\tilde{\xi}, \tilde{\xi}')\right| \leq C_T, \quad \forall \tilde{\xi}, \tilde{\xi}' \in \bm{\mathcal{X}\times \mathcal{A}}.	
\end{equation}
\end{lemma}

\begin{proof}
By Lemma \ref{lem:geometric}, we have
\begin{equation}
\label{uniform ergodicity}
\left|\Pi^{N, n}_{k}(\tilde{\xi}', \tilde{\xi}) - \sigma^{g^N_k}_{\rho_0}(\tilde{\xi}) \right| \leq (1-MC_{\mathsf{1}} (\gamma \eta_T)^{n_0})^{\lfloor n/n_0 \rfloor},
\end{equation}
which can be used to show the convergence of the series in \eqref{nu}. Consequently, $\nu_k$ is well-defined. The uniform bound \eqref{eq:uniform} follows from 
\begin{equation*}
\left|\nu_{k}(\tilde{\xi};\tilde{\xi}')\right| \le \sum_{n\ge 0} \left| \Pi^{N,n}_{k} (\tilde{\xi}', \tilde{\xi}) - \sigma^{g^N_k}_{\rho_0}(\tilde{\xi}) \right| \leq \sum_{n \ge 0}(1-MC_{\mathsf{1}} (\gamma \eta_T)^{n_0})^{\lfloor n/n_0 \rfloor} \le C_T.
\end{equation*}
Finally, we can verify that $\nu_k$ is a solution to the Poisson equation by observing that
\begin{align*}
\Pi^{N}_{k} \nu_{k}(\tilde{\xi}; \tilde{\xi}') &= \sum_{y} \nu_{k}(\tilde{\xi}; y)\Pi^N_{k}(\tilde{\xi}', y)\\ 
&= \sum_y \left( \sum_{n\ge 0} \left[\Pi^{N, n}_k (y, \tilde{\xi}) - \sigma^{g^N_k}_{\rho_0}(\tilde{\xi}) \right] \right) \Pi^N_{k}(\tilde{\xi}', y) \\ 
&\overset{(a)}{=} \sum_{n\ge 0} \left(\sum_y \left[ \Pi^{N, n}_{k}(y, \tilde{\xi}) - \sigma^{g^N_k}_{\rho_0}(\tilde{\xi}) \right] \Pi^N_{k}(\tilde{\xi}', y) \right) \\ 
&= \sum_{n\ge 0} \left[\Pi^{N,n+1}_{k} (\tilde{\xi}', \xi) - \sigma^{g^N_k}_{\rho_0}(\tilde{\xi}) \right] \\ 
&= \nu_{k}(\tilde{\xi}; \tilde{\xi}') - (\mathsf{1}_{\{\tilde{\xi}' = \tilde{\xi}\}} - \sigma^{g^N_k}_{\rho_0}(\tilde{\xi})),
\end{align*}
where the step $(a)$ uses \eqref{uniform ergodicity} and the Dominated Convergence Theorem.
\end{proof}

\begin{remark}
Noting that the bound in \eqref{eq:uniform} is uniform for all states $\tilde{\xi}$ and thus for notation simplification, we will use notation $\nu_k(\cdot)$ below as the family of solutions $\nu_k(\tilde{\xi};\cdot)$.
\end{remark}

Using the Poisson equation \eqref{poisson equation}, we can prove that the fluctuations of the data samples around a dynamic visiting measure $\sigma^{g^N_k}$ decay when the iteration steps becomes large.
\begin{lemma}
\label{fluctuation}
For any fixed state action pair $\tilde{\xi} = (\tilde{x},\tilde{a})$ and $T > 0$,    
\begin{equation}
\label{eq:online_convergence}
\lim_{N \to \infty} \e\left| \frac{1}{N}\sum_{ k=0}^{ \lfloor NT \rfloor - 1 } \left[ \mathsf{1}_{\{\tilde{\xi}_{k} = \tilde{\xi} \}} - \sigma^{g^N_k}(\tilde{\xi})\right] \right|^2 = 0.
\end{equation}
\end{lemma}

\begin{proof}
We define the error $\epsilon_k$ to be 
\begin{align}
\epsilon_{k} &:= \mathsf{1}_{\{ \tilde{\xi}_{k+1} = \tilde{\xi}\}} - \sigma^{g^N_k}(\tilde{\xi}) \nonumber \\
&= \nu_{k}\left(\tilde{\xi}_{k+1}\right)-\Pi^N_{k} \nu_{k}\left(\tilde{\xi}_{k+1}\right) \nonumber \\
&= \left[ \nu_{k}(\tilde{\xi}_{k+1})-\Pi^N_{k} \nu_{k}\left(\tilde{\xi}_{k}\right) \right] + \left[ \Pi^N_{k} \nu_{k}\left(\tilde{\xi}_{k}\right)-\Pi^N_{k} \nu_{k}\left(\tilde{\xi}_{k+1}\right) \right],
\end{align}
where we have used the definition of the Poisson equation \eqref{poisson}. Define also
\begin{equation}
\psi^N_k(y) := \Pi^N_k \nu_k(y),
\end{equation}
then
\begin{align}
\sum_{k=0}^{\lfloor NT \rfloor -1} \epsilon_{k}
&= \sum_{k=0}^{\lfloor NT \rfloor-1} \left[ \nu_{k}(\tilde{\xi}_{k+1})-\Pi^N_{k} \nu_{k}\left(\tilde{\xi}_{k}\right) \right] + \sum_{k=0}^{\lfloor NT \rfloor-1} \left[ \psi_{k}\left(\tilde{\xi}_{k}\right)-\psi_{k}\left(\tilde{\xi}_{k+1}\right) \right] \nonumber \\
&= \sum_{k=0}^{\lfloor NT \rfloor-1} \left[\nu_{k}(\tilde{\xi}_{k+1})-\Pi^N_{k} \nu_{k}\left(\tilde{\xi}_{k}\right)\right] + \sum_{k=1}^{\lfloor NT \rfloor-1} \left[\psi_{k}\left(\tilde{\xi}_{k}\right)-\psi_{k-1}\left(\tilde{\xi}_{k}\right)\right] \nonumber \\
&\phantom{=}+ \psi_{0}\left(\tilde{\xi}_{0}\right) - \psi_{\lfloor NT \rfloor-1}\left(\tilde{\xi}_{\lfloor NT \rfloor}\right).
\end{align}
Define the error term 
\begin{equation}
\label{eq:decompose}
\sum_{k=0}^{\lfloor NT \rfloor-1} \epsilon_{k} = \sum_{k=0}^{\lfloor NT \rfloor-1} \epsilon_{k}^{(1)} + \sum_{k=1}^{\lfloor NT \rfloor-1} \epsilon_{k}^{(2)} + \rho_{\lfloor NT \rfloor ; 0},
\end{equation}
where
\begin{align*}
\epsilon_{k}^{(1)} &= \left[ \nu_{k}\left(\tilde{\xi}_{k+1}\right)-\Pi^N_{k} \nu_{k}\left(\tilde{\xi}_{k}\right) \right], \nonumber \\
\epsilon_{k}^{(2)} &= \left[\psi_{k}\left(\tilde{\xi}_{k}\right)-\psi_{k-1}\left(\tilde{\xi}_{k}\right)\right], \\
\rho_{\lfloor NT \rfloor ; 0} &= \psi_{\theta_{0}}\left(\tilde{\xi}_{0}\right) -  \psi_{\theta_{\lfloor NT \rfloor-1}}\left(\tilde{\xi}_{\lfloor NT \rfloor}\right).
\end{align*}
To prove the convergence \eqref{eq:online_convergence}, it suffices to appropriately bound the fluctuation term $ \left| \sum\limits_{k=0}^{\lfloor NT \rfloor-1} \epsilon_k \right|$. Actually, the first term can be bound due to the martingale property while the second term can be bounded using the uniform geometric ergodicity and Lipschitz continuity. The third and fourth terms are uniformly bounded by \eqref{eq:uniform}. 

For the first term in \eqref{eq:decompose}, note that
\beq
\e\left\{ \nu_{k}\left(\tilde{\xi}_{k+1}\right) \mid \mathscr{F}_{k}\right\} = \Pi^N_{k} \nu_{k}\left(\tilde{\xi}_{k}\right),
\eeq
and thus 
$$
\left\{Z_n = \sum_{k=0}^{n-1} \epsilon_k^{(1)}, \ \mathscr{F}_n \right\}_{n\ge 0} 
$$
is a martingale. Since the conditional expectation is a contraction in $L^{2}$, we have 
\begin{equation}
\e \left| \Pi^N_{k} \nu_{k}\left(\tilde{\xi}_{k}\right)\right|^{2} \leq \e\left| \nu_{k}\left(\tilde{\xi}_{k+1}\right) \right|^{2}.
\end{equation}
Then,
\begin{align}
\e\left| \frac{1}{N} \sum_{k=0}^{\lfloor NT \rfloor-1} \epsilon_k^{(1)} \right|^2 
&=\frac{1}{N^2} \sum_{k=0}^{\lfloor NT \rfloor-1} \e\left| \Pi^N_{k} \nu_{k}\left(\tilde{\xi}_{k}\right)-\nu_{k}\left(\tilde{\xi}_{k+1}\right)\right|^{2} \nonumber \\
&\leq \frac{4}{N^2} \sum_{k=0}^{\lfloor NT \rfloor-1}\e\left|\nu_{k}\left(\tilde{\xi}_{k+1}\right)\right|^{2} \nonumber \\
&\overset{(a)}{\le} \frac{4C_T}{N},
\end{align}
where the step (a) is by the uniform boundedness \eqref{eq:uniform}. Thus, for any $T >0$,
\begin{equation}
\label{eq:error1}
\lim_{N \to \infty} \e \left| \frac{1}{N} \sum_{k=0}^{\lfloor NT \rfloor-1} \epsilon_k^{(1)} \right|^2 = 0.
\end{equation}
For the second term of \eqref{eq:decompose}, by the uniform geometric ergodicity \eqref{eq:geometric_aux}, for any fixed $\gamma_0>0$ we can choose $N_0$ large enough such that 
\begin{equation}
\sup_{k \le NT} \left( \sum_{n = \lfloor N_0 T \rfloor}^{\infty} \left| \Pi^{N, n}_{k}(y, \xi) - \sigma^{g^N_k}(\xi) \right| \right)^2 < \gamma_0, \quad \forall y \in \bm{\mathcal{X}} \times \bm{\mathcal{A}}.
\end{equation}
Then
\begin{align*}
&\phantom{=} \left| \frac{1}{N} \sum_{k=1}^{\lfloor NT \rfloor-1} \epsilon_{k}^{(2)} \right|^2 \\
&= \left| \frac{1}{N} \sum_{k=1}^{\lfloor NT \rfloor-1}\left[\psi_{k}\left(\tilde{\xi}_{k}\right)-\psi_{k-1}\left(\tilde{\xi}_{k}\right)\right] \right|^2 \\
&\leq C\left| \frac{1}{N} \sum_{k=1}^{\lfloor NT \rfloor-1} \left[\sum_{n=1}^{ \lfloor N_0 T \rfloor-1} \left[\Pi_{k}^{N, n}\left(\tilde{\xi}_k, \tilde{\xi} \right)- \sigma^{g^N_k}(\tilde{\xi}) \right]-\sum_{n=1}^{\lfloor N_0 T \rfloor-1} \left[\Pi_{k-1}^{N, n}\left(\tilde{\xi}_k, \tilde{\xi}\right) - \sigma^{g^N_{k-1}}(\tilde{\xi})\right]\right]  \right|^2 + C_T \gamma_{0} \\
&\leq C\left| \frac{1}{N} \sum_{k=1}^{\lfloor NT \rfloor-1} \sum_{n=1}^{\lfloor N_0 T \rfloor -1} \left[\Pi_{k}^{N, n} \left(\tilde{\xi}_k, \tilde{\xi} \right)- \Pi_{k-1}^{N, n}\left(\tilde{\xi}_k, \tilde{\xi}\right) \right] \right|^2 + C \frac{\lfloor N_0T \rfloor^{2}}{N^{2}}\left| \sum_{k=1}^{\lfloor NT \rfloor-1} \left[\sigma^{g^N_k}(\tilde{\xi})-\sigma^{g^N_{k-1}}(\tilde{\xi})\right] \right|^2 + C_T \gamma_{0}\\
&:= I^N_1 + I^N_2 + C_T \gamma_0. \numberthis \label{lipschitz bound}
\end{align*}
To bound $I^N_1$, we note that for any finite $n$, $\Pi^n_f$ is Lipschitz with respect to $f$. Using the Lipschitz continuity of the Softmax function, we have 
\begin{align*}
I_1^N 
&\leq \frac{C \lfloor N_0 T \rfloor}{N} \sum_{k=1}^{\lfloor NT \rfloor -1}  \sum_{n=1}^{\lfloor N_0 T \rfloor - 1} |\Pi^{N,n}_k(\tilde{\xi}_k, \tilde{\xi}) - \Pi^{N,n}_{k-1}(\tilde{\xi}_k, \tilde{\xi})|^2 \\
&\leq \frac{C \lfloor N_0 T \rfloor^2}{N} \sum_{k=1}^{\lfloor NT \rfloor-1}  \left[ \left|\eta^N_{k}-\eta^N_{k-1}\right|^2 + \left\|P^N_{k}-P^N_{k-1}\right\|^2 \right]
\end{align*}
By Lemma \ref{lem:regularity_of_Q}, we have
\begin{equation*}
\|P^N_k - P^N_{k-1}\|^2 \leq M^{1/2} \max_{\tilde{\xi}} |P^N_k(\xi) - P^N_{k-1}(\xi)|^2 \leq \frac{C_T}{N^2}
\end{equation*}
Moreover, $\eta_t$ is Lipschitz in $t$:
\begin{itemize}
\item if $\eta_t = (1+(\log(1+t))^2)^{-1}$, then
\begin{align*}
\left|\frac{d\eta_t}{dt} \right| = \left|-\frac{2 \log (1+t)}{(1+t) (1+(\log(1+t))^2)^2}\right| \leq 2,
\end{align*}
\item if $\eta_t = (1+t)^{-\varepsilon}$, then
\begin{align*}
\left|\frac{d\eta_t}{dt} \right| = \left|-\frac{\varepsilon}{(1+t)^{\varepsilon+1}}\right| \leq \varepsilon.
\end{align*}
\end{itemize}
In both cases, there exists a constant $C > 0$ such that
\begin{align*}
|\eta^N_{k-1} - \eta^N_k|
&= |\eta_{\frac{k-1}{N}} - \eta_{\frac{k}N}| \leq \frac{C}{N}.
\end{align*}
Therefore, 
\begin{align}
I^N_1 \leq \frac{C \lfloor N_0 T \rfloor^2}{N} \sum_{k=1}^{\lfloor NT \rfloor-1}  \left[ \left|\eta^N_{k}-\eta^N_{k-1}\right|^2 + \left\|P^N_{k}-P^N_{k-1}\right\|^2 \right] \leq \frac{C_T}{N^2}.
\end{align}
To bound $I^N_2$, we note that $\sigma^f$ is Lipschitz with respect to $f$ by Proposition \ref{prop:Lipschitzness_of_stationary_measures}, so
\begin{align}
I_2^N 
&\leq C\frac{T\lfloor N_0 T \rfloor^2}{N} \sum_{k=1}^{\lfloor NT \rfloor-1} \left[\sigma^{g^N_k}(\tilde{\xi})-\sigma^{g^N_{k-1}}(\tilde{\xi})\right]^2 \nonumber \\
&\leq C\frac{T\lfloor N_0 T \rfloor^2}{N} \sum_{k=1}^{\lfloor NT \rfloor-1} \left[ \left|\eta^N_{k}-\eta^N_{k-1}\right|^2 + \left\|P^N_{k}-P^N_{k-1}\right\|^2 \right] \leq \frac{C_T}{N^2}.
\end{align}
Thus, when $N$ is large enough,
\begin{equation*}
\left|\frac{1}{N} \sum_{k=1}^{\lfloor NT \rfloor-1} \epsilon_{k}^{(2)} \right|^2 \leq C_T \gamma_{0}
\end{equation*}
Since $\gamma_{0}$ is arbitrary,
\begin{equation}
\label{eq:error2}
\lim _{N \rightarrow \infty} \mathbf{E}\left|\frac{1}{N} \sum_{k=1}^{\lfloor NT \rfloor-1} \epsilon_{k}^{(2)}\right|^2=0
\end{equation}
Obviously, for the last term of \eqref{eq:decompose} by the boundedness in \eqref{eq:uniform} we have
$$
\lim_{N\to \infty} \frac{1}{N} \rho_{\lfloor NT \rfloor ; 0} = 0,
$$
which together with \eqref{eq:error1} and \eqref{eq:error2} derive the convergence of $ \frac{1}{N} \sum\limits_{k=0}^{\lfloor NT \rfloor-1} \epsilon_k$ and therefore proving \eqref{eq:online_convergence}.
\end{proof}

Now we can show the convergence of the stochastic fluctuation terms from the actor update. 
\begin{lemma}
\label{limit lemma}
For any $\xi =(x,a)$ and the stochastic error $M_t^N$ defined in \eqref{concentration2}, we have  
\beq
\label{actor fluctuation disappear}
\lim _{N \rightarrow \infty} \sup _{t \in(0, T]} \mathbb{E}\left|M_{t}^{N}(\xi)\right|=0.
\eeq
\end{lemma}

\begin{proof} The proof of \eqref{actor fluctuation disappear} consists of two parts. We first set up a bound for the difference of the actor's update. Actually, define 
\begin{equation}
\bar{H}^N_{\xi, \xi', k} := \zeta^N_k \clip(Q^N_k(\xi')) \left[[A(\mu^N_k)]_{\xi,\xi'} - \sum_{a''} f^N_k(x',a'') [A(\mu^N_k)]_{\xi,(x',a'')}\right]
\end{equation}
and we aim to prove 
\begin{equation}
\label{actor diff}
\mathbb{E}|\bar{H}^N_{\xi, \xi', k+1} - \bar{H}^N_{\xi, \xi', k}| \leq \frac{C_T}{N}
\end{equation}
uniformly in $\xi, \xi'$. Then we can use Lemma \ref{fluctuation} to prove that as the training step becomes large, the fluctuations of the data samples around the stationary distribution will disappear, which concludes the result. \\

(\romannumeral1) To bound the difference \eqref{actor diff}, note that 
\begin{align*}
&\phantom{=} \abs{\bar{H}^N_{\xi, \xi', k+1} - \bar{H}^N_{\xi, \xi', k}} \\
&= |\zeta^N_{k+1} - \zeta^N_k| \abs{\clip(Q^N_{k+1}(\xi')) \left[[A(\mu^N_{k+1})]_{\xi,\xi'} - \sum_{a''} f^N_{k+1}(x',a'') [A(\mu^N_{k+1})]_{\xi,(x',a'')} \right]} \\
&+ \zeta^N_k \abs{\clip(Q^N_{k+1}(\xi')) - \clip(Q^N_k(\xi'))} \abs{ [A(\mu^N_{k+1})]_{\xi,\xi'} - \sum_{a''} f^N_{k+1}(x',a'') [A(\mu^N_{k+1})]_{\xi,(x',a'')}} \\
&+ \zeta^N_k \abs{\clip(Q^N_k(\xi'))} \Bigg|\left[[A(\mu^N_{k+1})]_{\xi,\xi'} - \sum_{a''} f^N_{k+1}(x',a'') [A(\mu^N_{k+1})]_{\xi,(x',a'')} \right] \\
&\phantom{=}- \left[[A(\mu^N_k)]_{\xi,\xi'} - \sum_{a''} f^N_k(x',a'') [A(\mu^N_k)]_{\xi,(x',a'')} \right] \Bigg| \\
&:= I_3^N + I_4^N + I_5^N.
\end{align*}
As the derivative of $\zeta_t$ satisfies
$$\left|\frac{d\zeta_t}{dt} \right| = \left|\frac{-\beta}{(1+t)^{\beta+1}} \right| \leq \beta,$$
so the term $I^N_3$ could be bounded as follows:
\begin{equation}
\label{actor diff first term}
I_3^N \le C_T |\zeta^N_{k+1} - \zeta^N_k| = C_T |\zeta_{\frac{k+1}{N}} - \zeta_{\frac{k}{N}}| \leq \frac{\beta C_T}{N}.
\end{equation}
For $I^N_4$, recall that the function $\clip(.)$ is 1-Lipschitz, i.e., $|\clip(x) - \clip(y)| \leq |x-y|$), so
\begin{align*}
I_4^N &\le \frac{|Q^N_{k+1}(\xi') - Q^N_k(\xi')|}{N} |[A(\mu^N_{k+1})]_{\xi,\xi',k+1} - \sum_{a''} f^N_k(x',a'') [A(\mu^N_{k+1})]_{\xi,(x',a'')}| \\
&\leq C_T |Q^N_{k+1}(\xi') - Q^N_k(\xi')|, \\
&\overset{\eqref{eq:4.41}}\leq \frac{C_T}{N} \sqrt{1+(1+\gamma^2) \max_{\xi \in \bm{\X} \times \bm{\A}} |Q^N_k(\xi)|^2}.
\end{align*}
Therefore, by Lemma \ref{lem:maximal_Q_bounded}, we have
\begin{equation}
\E|I^N_4| \leq \frac{C_T}{N}. \label{actor diff second term}
\end{equation}
For $I^N_5$, recall from Lemma \ref{lem:diff_of_B_bar_B} that for any $k \leq NT$,
\begin{equation*}
\sup_{\xi,\xi' \in \X \times \A} \abs{[A(\mu^N_{k+1})]^N_{\xi, \xi'} - [A(\mu^N_k)]_{\xi, \xi'}} \leq \frac{C_T}{N}.
\end{equation*}
Hence, 
\begin{align*}
I_5^N &\le C\abs{\left[[A(\mu^N_{k+1})]_{\xi,\xi'} - \sum_{a''} f^N_{k+1}(x',a'') [A(\mu^N_{k+1})]_{\xi,(x',a'')} \right] - \left[[A(\mu^N_k)]_{\xi,\xi'} - \sum_{a''} f^N_k(x',a'') [A(\mu^N_k)]_{\xi,(x',a'')} \right]} \\
&\leq C\Bigg[ \abs{[A(\mu^N_{k+1})]_{\xi,\xi'} - [A(\mu^N_k)]_{\xi,\xi'}} + \sum_{a''} \abs{f^N_{k+1}(x',a'') - f^N_{k}(x',a'')} \cdot \abs{[A(\mu^N_{k+1})]_{\xi,(x',a'')}} \\
&+ \sum_{a''} f^N_k(x',a'') \abs{[A(\mu^N_{k+1})]_{\xi,(x',a'')} - [A(\mu^N_k)]_{\xi,(x',a'')}} \Bigg] \\
&\leq C\bracket{1+\sum_{a''} f^N_k(x',a'')} \sup_{\xi' \in \X\times\A} \abs{[A(\mu^N_{k+1})]_{\xi,\xi'} - [A(\mu^N_k)]_{\xi,\xi'}} + C \norm{P_{k+1}^N - P_k^N} \leq \frac{C_T}{N}. \numberthis \label{actor diff third term}
\end{align*}
Combining \eqref{actor diff first term}, \eqref{actor diff second term} and \eqref{actor diff third term} yields \eqref{actor diff}. \\

(\romannumeral2) Now we can prove the convergence \eqref{actor fluctuation disappear}. We let $K := K(N) \in \mathbb{N}$, such that $1 \ll K(N) \ll N$. This means as $N\to\infty$, we have $K(N) \to +\infty$ and $K(N)/N \to 0$. Further define $\Delta=t/K$. Then
\begin{align*}
M_{t}^{N}(\xi) &= \frac{1}{N} \sum_{k=0}^{\floor{Nt}-1} \left(\bar{H}^N_{\xi,\tilde{\xi}_k,k} - \sum_{\xi'\in\bm{\X}\times\bm{\A}} \bar{H}^N_{\xi,\xi',k} \sigma^{g^N_k}(\xi') \right) \\
&= \frac{1}{N} \sum_{j=0}^{K-1} \sum_{k=j\floor{N\Delta}}^{(j+1)\floor{N\Delta}-1} \left(\bar{H}^N_{\xi,\tilde{\xi}_k,k} - \sum_{\xi'\in\bm{\X}\times\bm{\A}} \bar{H}^N_{\xi,\xi',k} \sigma^{g^N_k}(\xi') \right) + r^N_t(\xi),
\end{align*}
where
\begin{align*}
r^N_t(\xi) = \frac{1}{N} \sum_{k=K\floor{N\Delta}}^{\min((K+1)\floor{N\Delta}-1, \floor{Nt}-1)} \bracket{\bar{H}^N_{\xi,\tilde{\xi}_k,k} - \sum_{\xi'\in\bm{\X}\times\bm{\A}} \bar{H}^N_{\xi,\xi',k} \sigma^{g^N_k}(\xi')}.
\end{align*}
The terms $\bar{H}^N_{\xi,\xi',k}$ are bounded by some constant $C_T > 0$ as the kernel entries $|[A(\mu^N_k)]_{\xi,\xi'}|$ are bounded. As a result, the summands are also bounded, so
\begin{equation}
|r^N_t(\xi)| \leq \frac{\floor{N\Delta}}{N} C_T \leq \frac{TC_T}{K}.
\end{equation}
Next, we have
Then,
\begin{align*}
M^N_t(\xi) - r^N_t(\xi) &= \frac{1}{N} \sum_{j=0}^{K-1} \sum_{k=j\floor{N\Delta}}^{(j+1)\floor{N\Delta}-1} \bigg[ \bracket{\bar{H}^N_{\xi,\tilde{\xi}_k,k} - \bar{H}^N_{\xi,\tilde{\xi}_k,j\floor{N\Delta}}} \\
&\phantom{=}+ \bracket{\bar{H}^N_{\xi,\tilde{\xi}_k,j\floor{N\Delta}} - \sum_{\xi'} \bar{H}^N_{\xi,\xi',j\floor{N\Delta}} \sigma ^{g^N_k}(\xi')} + \sum_{\xi'} \bracket{\bar{H}^N_{\xi,\xi',j\floor{N\Delta}} - \bar{H}^N_{\xi,\xi',k}} \sigma^{g^N_k}(\xi')\bigg] \\
&= J^N_{1,t}(\xi) + J^N_{2,t}(\xi) + J^N_{3,t}(\xi), 
\numberthis
\end{align*}
where
\begin{align*}
J^N_{1,t}(\xi) &= \frac{1}{N} \sum_{j=0}^{K-1} \sum_{k=j\floor{N\Delta}}^{(j+1)\floor{N\Delta}-1} \bracket{\bar{H}^N_{\xi,\tilde{\xi}_k,k} - \bar{H}^N_{\xi,\tilde{\xi}_k,j\floor{N\Delta}}} \\
J^N_{2,t}(\xi) &= \frac{1}{N} \sum_{j=0}^{K-1} \sum_{k=j\floor{N\Delta}}^{(j+1)\floor{N\Delta}-1} \bracket{\bar{H}^N_{\xi,\tilde{\xi}_k,j\floor{N\Delta}} - \sum_{\xi'} \bar{H}^N_{\xi,\xi',j\floor{N\Delta}} \sigma ^{g^N_k}(\xi')} \\
J^N_{3,t}(\xi) &= \frac{1}{N} \sum_{j=0}^{K-1} \sum_{k=j\floor{N\Delta}}^{(j+1)\floor{N\Delta}-1} \sum_{\xi'} \left[\left(\bar{H}^N_{\xi,\xi',j\floor{N\Delta}} - \bar{H}^N_{\xi,\xi',k} \right) \sigma^{g^N_k}(\xi') \right].
\end{align*}
From \eqref{actor diff}, we have for any $j \in 0,1, \cdots, K-1$ and any $k\in [j\lfloor N\Delta \rfloor, (j+1) \lfloor N\Delta \rfloor -1]$, 
\begin{align*}
\E \left| \bar{H}^N_{\xi, \xi', k} - \bar{H}^N_{\xi, \xi', j\floor{N\Delta}} \right| 
\leq \sum_{\ell=j \floor{N\Delta}}^{k-1} \E \left| \bar{H}^N_{\xi, \xi', \ell+1} - \bar{H}^N_{\xi, \xi', \ell} \right| 
\leq (k - j\lfloor N\Delta\rfloor) \frac{C_T}{N}.
\end{align*}
Moreover,
\begin{align*}
\E \left| \sum_{\xi'} \left[\left(\bar{H}^N_{\xi, \xi', k} - \bar{H}^N_{\xi, \xi', j\floor{N\Delta}} \right) \sigma^{g^N_k}(\xi') \right] \right|
\leq \sum_{\ell=j \floor{N\Delta}}^{k-1} \sum_{\xi'} \E \left| \bar{H}^N_{\xi, \xi', \ell+1} - \bar{H}^N_{\xi, \xi', \ell} \right|
\leq (k - j\lfloor N\Delta\rfloor) \frac{MC_T}{N}.
\end{align*}
Therefore,
\begin{align*}
\label{actor diff estimation}
\max(\E|J^N_{1,t}(\xi)|, \, \E|J^N_{3,t}(\xi)|)
&\leq \frac{1}{N} \sum_{j=0}^{K-1} \sum_{k=j \lfloor N\Delta \rfloor}^{(j+1)\lfloor N\Delta \rfloor-1} C_T \frac{k-j\lfloor N\Delta\rfloor}{N} \\
&= \frac{1}{N} \sum_{j=0}^{K-1} \sum_{k=0}^{\lfloor N\Delta \rfloor-1} \frac{C_T k}{N}\\
&\leq \frac{C_T}{N} \sum_{j=0}^{K-1} \frac{\floor{N\Delta}^2}{N} \\
&= \frac{KC_T \floor{N\Delta}^2}{N^2} \leq KC_T \Delta^2 = C_T K\bracket{\frac{t}{K}}^2 \leq \frac{C_T}{K}. \numberthis
\end{align*}
To control $J^N_{2,t}(\xi)$, we note that
\begin{equation}
\bar{H}^N_{\xi, \tilde \xi_k, j\floor{N\Delta}} - \sum_{\xi' \in \bm{\mathcal{X}} \times \bm{\mathcal{A}}} \bar{H}^N_{\xi, \xi', j\floor{N\Delta}} \sigma^{g^N_{k}}(\xi') = \sum_{\xi'} \bar{H}^N_{\xi, \xi', j\floor{N\Delta}} \left[ \mathsf{1}_{\{ \tilde \xi_k = \xi'\}} - \sigma^{g^N_{k}}(\xi')\right],
\end{equation}
so one could control $J^N_{2,t}(\xi)$ by the uniform boundedness of $\bar{H}^N_{\xi, \xi', j\floor{\Delta N}}$ and Lemma \ref{fluctuation}. Indeed, 
\begin{align*}
\left| J^N_{2,t}(\xi) \right| &= \left| \frac{1}{N} \sum_{j=0}^{K-1} \sum_{k=j\lfloor N\Delta \rfloor}^{(j+1)\lfloor N\Delta \rfloor-1} \sum_{\xi'} \bar{H}^N_{\xi, \xi', j\floor{N\Delta}}
\left[ \mathsf{1}_{\{ \xi_k = \xi'\}} - \sigma^{g^N_k}(\xi')\right] \right| \\
&= \left|\frac{1}{N} \sum_{j=0}^{K-1} \sum_{\xi'} \bar{H}^N_{\xi, \xi', j\floor{N\Delta}} \sum_{k=j\lfloor N\Delta \rfloor}^{(j+1)\lfloor N\Delta \rfloor-1} \left[ \mathsf{1}_{\{ \xi_k = \xi'\}} - \sigma^{g^N_k}(\xi')\right] \right|\\
&\le C_T \sum_{\xi'}\left| \frac{1}{N} \sum_{j=0}^{K-1} \sum_{k=j\lfloor N\Delta \rfloor}^{(j+1)\lfloor N\Delta \rfloor-1}  \left[ \mathsf{1}_{\{ \xi_k = \xi'\}} - \sigma^{g^N_k}(\xi')\right] \right|, \\
&= C_T \sum_{\xi'}\left| \frac{1}{N} \sum_{k=0}^{K\lfloor N\Delta \rfloor-1}  \left[ \mathsf{1}_{\{ \xi_k = \xi'\}} - \sigma^{g^N_k}(\xi')\right] \right|, \numberthis
\end{align*}
which together with Lemma \ref{fluctuation} derive
\begin{equation*}
\lim_{N \to \infty} \e\left| J^N_{2,t}(\xi) \right|^2 = 0.
\end{equation*}
Collecting our results, we have shown that
\begin{equation}
\sup _{t \in(0, T]} \e\left|M_{t}^{N}(\xi)\right| \leq \frac{C_T}{K(N)} \overset{N\to\infty}\to  0
\end{equation}
by the assumption that $1 \ll K(N)$.
\end{proof}

Following the same method, we can finish proving the convergence of the stochastic fluctuation terms from the critic dynamic. \begin{lemma}
\label{concentration lemma}
For any $\xi =(x,a)$ and the stochastic error $M_t^{i, N}, i=1,2,3$ defined in \eqref{eq:MNt}, we have  
\beq
\label{critic fluctuation disappear}
\lim _{N \rightarrow \infty} \sup _{t \in(0, T]} \mathbb{E}\left|M_{t}^{i, N}(\xi)\right|=0, \quad i=1,2,3.
\eeq
\end{lemma}

\begin{proof}
As in the proof for the decay of $M_t^N$, we use two steps to prove the result. 
\begin{itemize}
\item [(\romannumeral1)] Prove that the fluctuations of the data samples around a dynamic stationary distribution $\pi^{g_k}$ decay when the number of iteration steps becomes large. Actually, with exactly the same approach as in Lemma \ref{fluctuation}, we can prove that for any finite $T > 0$ and fixed state action pair $\xi = (x,a)$,
\begin{equation}
\label{online convergence 2}
\lim_{N \to 0} \e\left| \frac{1}{N}\sum_{ k=0 }^{ \lfloor NT \rfloor - 1 } \left[ \mathsf{1}_{\{ \xi_{k} = \xi\}} - \pi^{g_k}(\xi)\right] \right|^2 = 0.
\end{equation}
\item [(\romannumeral2)] Use the same method as in Lemma \ref{limit lemma} to prove the stochastic fluctuation terms vanish as $N \rightarrow \infty$. 
\end{itemize}

We shall only study $M_t^{3,N}$, as the same procedure could be applied to $M_t^{1,N}$ and $M_t^{2,N}$. Recalling the notation in \eqref{transition integral}, we have 
\begin{align*}
M^{3,N}_t(\xi) &= \frac{1}{N}\sum_{k=0}^{\floor*{Nt}-1} \gamma \left[ \sum_{a''} Q^N_k(x_{k+1}, a'') g_k^N(x_{k+1}, a'') - \prob_k^N Q_k^N(\xi_k) \right] [A(\nu^N_k)]^N_{\xi, \xi_k}  \\ 
&+ \frac{1}{N}\sum_{k=0}^{\floor*{Nt}-1} \gamma \left[ \prob_k^N Q_k^N(\xi_k) [A(\nu^N_k)]^N_{\xi, \xi_k} - \sum_{\xi'} \prob_k^N Q_k^N(\xi') [A(\nu^N_k)]_{\xi, \xi'} \pi^{g_k^N}(\xi') \right]\\ 
&:= I^{1,N}_t(\xi) + I^{2,N}_t(\xi). \numberthis
\end{align*} 
For $I^{1,N}_t(\xi)$, define 
\begin{equation}
\epsilon_k := \left[ \sum_{a''} Q^N_k(x_{k+1}, a'') g^N_k(x_{k+1}, a'') - \prob_k^N Q_k^N(\xi_k) \right] [A(\nu^N_k)]_{\xi, \xi_k}.
\end{equation}
Noting that 
\begin{equation}
\e\left[\sum_{a''} Q^N_k \left(x_{k+1}, a'' \right) g_k^N(x_{k+1}, a'') \mid \mathscr{F}_{k}\right] = \prob^N_{k} Q^N_{k}\left(\xi_{k}\right),
\end{equation}
we know  
$$
\left\{Y_n = \sum_{k=0}^{n-1} \epsilon_k, \ \mathscr{F}_n \right\}_{n\ge 0} 
$$
is a martingale. Since the conditional expectation is a contraction in $L^{2}$, we have 
\begin{equation}
\e \left| \prob^N_{k} Q^N_{k}\left(\xi_{k}\right) \right|^{2} \leq \e\left| \sum_{a''} Q^N_k \left(x_{k+1}, a'' \right) g_k^N(x_{k+1}, a'') \right|^{2}.
\end{equation}
Then,
\begin{align*}
\e \left| \frac{1}{N} \sum_{k=0}^{\lfloor NT \rfloor-1} \epsilon_k \right|^2 &=\frac{1}{N^2} \sum_{k=0}^{\lfloor NT \rfloor-1} \e\left| \prob^N_{k} Q^N_{k}\left(\xi_{k}\right) - \sum_{a''} Q^N_k \left(x_{k+1}, a'' \right) g_k^N(x_{k+1}, a'') \right|^{2}\\
& \leq \frac{4}{N^2} \sum_{k=0}^{\lfloor NT \rfloor-1}\e\left| \sum_{a''} Q^N_k \left(x_{k+1}, a'' \right) g_k^N(x_{k+1}, a'') \right|^{2} \\
& \overset{(a)}{\le} \frac{4}{N^2} \sum_{k=0}^{\lfloor NT \rfloor-1}\e \sum_{a''} \left| Q^N_k \left(x_{k+1}, a'' \right) \right|^{2}  \\
& \overset{(b)}{\le} \frac{C_T}{N},
\end{align*}
where the step $(a)$ is by Cauthy-Schwartz's inequality and step $(b)$ by \eqref{pre kernel bound} and Lemma \ref{lem:maximal_Q_bounded}. Thus, for any $T >0$,
\begin{equation}
\lim_{N \to \infty} \e \left| I_t^{1, N} \right| = \lim_{N \to \infty} \gamma \e \left| \frac{1}{N} \sum_{k=0}^{\lfloor NT \rfloor-1} \epsilon_k \right| = 0.
\end{equation}

For $I^{2, N}_t$, as in the proof of Lemma \ref{limit lemma}, we define 
\begin{equation}
H^N_{\xi, \xi', k} := \prob_k^N Q_k^N(\xi') [A(\nu^N_k)]_{\xi, \xi'} = \sum_{z, a''} Q_k^N(z, a'') g^N_k(z, a'') p(z| \xi') [A(\nu^N_k)]_{\xi, \xi'}.
\end{equation}
First by Lemma \ref{lem:diff_of_B_bar_B} and \ref{lem:maximal_Q_bounded}, we have the bound 
\begin{equation}
\label{critic l2 bound}
\sup _{0 \leq k \leq\lfloor T N\rfloor} \sup _{\xi' \in \mathcal{X} \times \mathcal{A}}  \e \left| H^N_{\xi, \xi', k} \right|^2 \le C_T. 
\end{equation}
Then by Lemma \ref{lem:diff_of_B_bar_B} and 
\ref{lem:regularity_of_Q}, we have 
\begin{align*}
\e \left| H^N_{\xi, \xi', k+1} - H^N_{\xi, \xi', k} \right|^2 &\le \sum_{z, a''} \mathbb{E} \left| Q_{k+1}^N(z, a'') g^N_{k+1}(z, a'') [A(\nu^N_{k+1})]_{\xi, \xi'} - Q_k^N(z, a'') g^N_k(z, a'') [A(\nu^N_k)]_{\xi, \xi'} \right|^2 \\ 
&\le 3 \sum_{z, a''} \left| \left(Q^N_{k+1}(z, a'') - Q^N_k(z, a'') \right) g^N_{k+1}(z, a'') [A(\nu^N_{k+1})]_{\xi, \xi'} \right|^2 \\ 
&+ 3 \sum_{z, a''} \left| Q_{k}^N(z, a'') [A(\nu^N_{k+1})]_{\xi, \xi'} \left(g^N_{k+1}(z, a'') - g^N_{k}(z, a'') \right) \right|^2 \\
&+ 3 \sum_{z, a''} \left| Q_{k}^N(z, a'') g^N_{k}(z, a'') \left([A(\nu^N_{k+1})]_{\xi, \xi'} - [A(\nu^N_k)]_{\xi, \xi'} \right) \right|^2 \\
&\le \frac{C_T}{N^2}, \numberthis
\end{align*}
which derives 
\begin{equation}
\label{critic difference l1 bound}
\sup _{0 \leq k \leq\lfloor T N\rfloor-1} \sup _{\xi' \in \mathcal{X} \times \mathcal{A}} \e \left| H^N_{\xi, \xi', k+1} - H^N_{\xi, \xi', k} \right| \le \left( \sup _{0 \leq k \leq\lfloor T N\rfloor-1} \sup _{\xi' \in \mathcal{X} \times \mathcal{A}} \e \left| H^N_{\xi, \xi', k+1} - H^N_{\xi, \xi', k} \right|^2 \right)^{\frac12} \le  \frac{C_T}{N}.
\end{equation}

Then following the step 
(\romannumeral2) in the proof of Lemma \ref{limit lemma}, now we can prove the convergence $I_t^{2, N}$. Then, following step 
(\romannumeral2) in the proof of Lemma \ref{limit lemma}, now we can prove the convergence $I_t^{2,N}(\xi)$. We let $K := K(N) \in \N$ such that $1 \ll K \ll N$ and define $\Delta = t/K$. Then, we can decompose $I_t^{2,N}(\xi)$ into the following terms:
\begin{equation}
I^{2,N}_t(\xi) = J^N_{1,t}(\xi) + J^N_{2,t}(\xi) + J^N_{3,t}(\xi) + r^N_t(\xi),
\end{equation}
where
\begin{align*}
J^N_{1,t}(\xi) &= \frac{1}{N} \sum_{j=0}^{K-1} \sum_{k=j\floor{N\Delta}}^{(j+1)\floor{N\Delta}-1} \bracket{H^N_{\xi,\tilde{\xi}_k,k} - H^N_{\xi,\tilde{\xi}_k,j\floor{N\Delta}}} \\
J^N_{2,t}(\xi) &= \frac{1}{N} \sum_{j=0}^{K-1} \sum_{k=j\floor{N\Delta}}^{(j+1)\floor{N\Delta}-1} \bracket{H^N_{\xi,\tilde{\xi}_k,j\floor{N\Delta}} - \sum_{\xi'} H^N_{\xi,\xi',j\floor{N\Delta}} \pi^{g^N_k}(\xi')} \\
J^N_{3,t}(\xi) &= \frac{1}{N} \sum_{j=0}^{K-1} \sum_{k=j\floor{N\Delta}}^{(j+1)\floor{N\Delta}-1} \sum_{\xi'} \bracket{H^N_{\xi,\xi',j\floor{N\Delta}} - H^N_{\xi,\xi',k}} \pi^{g^N_k}(\xi') \\
r^N_t(\xi) &= \frac{1}{N} \sum_{k=K\floor{N\Delta}}^{\min((K+1)\floor{N\Delta}-1, \floor{Nt}-1)} \bracket{H^N_{\xi,\tilde{\xi}_k,k} - \sum_{\xi'} H^N_{\xi,\xi',k} \pi^{g^N_k}(\xi')}.
\end{align*}
Again, we have
\begin{align*}
|r^N_t(\xi)|^2 &\leq \frac{\floor{N\Delta}}{N^2} \sum_{k=K\floor{N\Delta}}^{\min((K+1)\floor{N\Delta}-1, \floor{Nt}-1)} \bracket{H^N_{\xi,\tilde{\xi}_k,k} - \sum_{\xi'} H^N_{\xi,\xi',k} \pi^{g^N_k}(\xi')}^2 \\
&\leq \frac{2\Delta}{N} \sum_{k=K\floor{N\Delta}}^{\min((K+1)\floor{N\Delta}-1, \floor{Nt}-1)} \sqbracket{\bracket{H^N_{\xi,\tilde{\xi}_k,k}}^2 + \sum_{\xi'} \bracket{H^N_{\xi,\xi',k} \pi^{g^N_k}(\xi')}^2} \\
&\leq \frac{2\Delta}{N} \sum_{k=K\floor{N\Delta}}^{\min((K+1)\floor{N\Delta}-1, \floor{Nt}-1)} \sqbracket{\bracket{H^N_{\xi,\tilde{\xi}_k,k}}^2 + \sum_{\xi'} \bracket{H^N_{\xi,\xi',k}}^2 \pi^{g^N_k}(\xi')},
\end{align*}
so by \eqref{critic l2 bound},
\begin{equation}
\E|r^N_t(\xi)|^2 \leq \frac{C_T \Delta \floor{N\Delta}}{N} \leq C_T \Delta^2 \leq \frac{C_T}{K^2}.
\end{equation}
Moreover,
\begin{align*}
\mathbb{E}[J^N_{1,t}(\xi)]^2 
&\leq \frac{K \floor{N\Delta}}{N^2} \sum_{j=0}^{K-1} \sum_{k=j\floor{N\Delta}}^{(j+1)\floor{N\Delta}-1} \mathbb{E}\sqbracket{H^N_{\xi,\tilde{\xi}_k,k} - H^N_{\xi,\tilde{\xi}_k,j\floor{N\Delta}}}^2 \\
&\leq \frac{T}{N} \sum_{j=0}^{K-1} \sum_{k=j\floor{N\Delta}}^{(j+1)\floor{N\Delta}-1} \bracket{\frac{C_T(k - j\floor{N\Delta})}{N}}^2 \\
&\leq \frac{T}{N} \sum_{j=0}^{K-1} \sum_{k=0}^{\floor{N\Delta}-1} \bracket{\frac{kC_T}{N}}^2 \\
&\leq \frac{TC^2_T}{3N^3} \sum_{j=0}^{K-1} \floor{N\Delta}^3 \leq KC_T \Delta^3 \leq \frac{C_T}{K^2}. \numberthis
\end{align*}
We can similarly control $J^N_{3,t}(\xi)$ as follows:
\begin{align*}
\mathbb{E}[J^N_{3,t}(\xi)]^2 &\leq \frac{K\floor{N\Delta}}{N}\sum_{j=0}^{K-1} \sum_{k=j\floor{N\Delta}}^{(j+1)\floor{N\Delta}-1} \mathbb{E}\sqbracket{\sum_{\xi'} \bracket{H^N_{\xi,\xi',j\floor{N\Delta}} - H^N_{\xi,\xi',k}} \pi^{g^N_k}(\xi')}^2 \\
&\leq \frac{K\floor{N\Delta}}{N}\sum_{j=0}^{K-1} \sum_{k=j\floor{N\Delta}}^{(j+1)\floor{N\Delta}-1} \mathbb{E}\sqbracket{\sum_{\xi'}  \bracket{H^N_{\xi,\xi',j\floor{N\Delta}} - H^N_{\xi,\xi',k}}^2 \pi^{g^N_k}(\xi')} \\
&\leq \frac{T}{N} \sum_{j=0}^{K-1} \sum_{k=j\floor{N\Delta}}^{(j+1)\floor{N\Delta}-1} \bracket{\frac{C_T(k - j\floor{N\Delta})}{N}}^2 \leq \frac{C_T}{K^2}. \numberthis
\end{align*}
Finally, note that
\begin{equation}
H^N_{\xi, \tilde \xi_k, j\floor{\Delta N}} - \sum_{\xi' \in \mathcal{X} \times \mathcal{A}} H^N_{\xi, \xi', j\floor{\Delta N}} \pi^{g^N_{k}}(\xi') = \sum_{\xi'} H^N_{\xi, \xi', j\floor{\Delta N}} \left[ \mathsf{1}_{\{ \xi_k = \xi'\}} - \pi^{g^N_{k}}(\xi')\right].
\end{equation}
Thus,
\begin{align*}
\label{l2 bound}
\e \left| J^N_{2,t}(\xi) \right| &= \frac{1}{N} \e \left| \sum_{j=0}^{K-1} \sum_{k=j\lfloor N\Delta \rfloor}^{(j+1)\lfloor N\Delta \rfloor-1} \sum_{\xi'} H^N_{\xi, \xi', j\floor{N\Delta}}
\left[ \mathsf{1}_{\{ \xi_k = \xi'\}} - \pi^{g^N_k}(\xi')\right] \right| \\
&\leq \frac{1}{N} \e \left| \sum_{j=0}^{K-1} \bracket{\max_{\xi'} H^N_{\xi, \xi', j\floor{N\Delta}}} \sum_{k=j\lfloor N\Delta \rfloor}^{(j+1)\lfloor N\Delta\rfloor-1} \sum_{\xi'} \left[ \mathsf{1}_{\{ \xi_k = \xi'\}} - \pi^{g^N_k}(\xi')\right] \right|\\
&\overset{\text{(CS)}}\leq \frac{1}{N} \e \sqbracket{\bracket{\sum_{j=0}^{K-1} \bracket{\max_{\xi'} H^N_{\xi, \xi', j\floor{N\Delta}}}^2}^{1/2} \bracket{\sum_{j=0}^{K-1} \bracket{\sum_{k=j\lfloor N\Delta \rfloor}^{(j+1)\lfloor N\Delta \rfloor-1} \sum_{\xi'} \left[ \mathsf{1}_{\{ \xi_k = \xi'\}} - \pi^{g^N_k}(\xi')\right]}^2}^{1/2}} \\
&\overset{\text{(CS)}}\leq \frac{1}{N}  \sqbracket{\e\bracket{\sum_{j=0}^{K-1} \bracket{\max_{\xi'} H^N_{\xi, \xi', j\floor{N\Delta}}}^2} \e\bracket{\sum_{j=0}^{K-1} \bracket{\sum_{k=j\lfloor N\Delta \rfloor}^{(j+1)\lfloor N\Delta \rfloor-1} \sum_{\xi'} \left[ \mathsf{1}_{\{ \xi_k = \xi'\}} - \pi^{g^N_k}(\xi')\right]}^2}}^{1/2} \\
&\overset{\eqref{critic l2 bound}}\leq \frac{KC_T}{N}  \sqbracket{\mathbb{E}\sqbracket{\frac{1}{K}\sum_{j=0}^{K-1} \bracket{\sum_{k=j\lfloor N\Delta\rfloor}^{(j+1)\lfloor N\Delta \rfloor-1} \sum_{\xi'} \left[ \mathsf{1}_{\{ \xi_k = \xi'\}} - \pi^{g^N_k}(\xi')\right]}^2}}^{1/2} \\
&= \frac{KC_T \floor{N\Delta}}{N}  \sqbracket{\mathbb{E}\sqbracket{\frac{1}{K}\sum_{j=0}^{K-1} \bracket{\frac{1}{\floor{N\Delta}}\sum_{k=j\lfloor N\Delta\rfloor}^{(j+1)\lfloor N\Delta \rfloor-1} \sum_{\xi'} \left[ \mathsf{1}_{\{ \xi_k = \xi'\}} - \pi^{g^N_k}(\xi')\right]}^2}}^{1/2} \\
&\leq TC_T  \underbrace{\sqbracket{\mathbb{E}\sqbracket{\frac{1}{K}\sum_{j=0}^{K-1} \bracket{\frac{1}{\floor{N\Delta}}\sum_{k=j\lfloor N\Delta\rfloor}^{(j+1)\lfloor N\Delta \rfloor-1} \sum_{\xi'} \left[ \mathsf{1}_{\{ \xi_k = \xi'\}} - \pi^{g^N_k}(\xi')\right]}^2}}^{1/2}}_{\to 0}  \\
&\overset{n\to\infty}\to 0, \numberthis
\end{align*}
where step (CS) is by the Cauchy-Schwartz inequality. Combining \eqref{online convergence 2}, \eqref{critic l2 bound} and \eqref{l2 bound}, we have 
\begin{equation}
\lim_{N \to \infty} \e\left|I^N_{3,j} \right| = 0.
\end{equation}
Consequently $\mathbb{E}|I^{2,N}_t(\xi)| \to 0$, and so is $M^{3,N}_t(\xi)$. This completes the proof.
\end{proof}

Let $\rho^N$ denotes the probability measure of $\left(\mu_t^N, \nu_t^N, P_t^N, Q_t^N\right)_{0\le t \le T}$, which takes value in the set of probability measures $\M\left(D_E([0,T])\right)$. From the relative compactness result in Section \ref{relative compactness section}, we know that the sequence pf measures $\{ \rho^N \}_{N \in \mathbb{N}}$ contains a subsequence $\rho^{N_k}$ that converges weakly. Now we can prove the limit points of any convergence subsequence $\rho^{N_k}$ will satisfies the limiting ODEs \eqref{NN gradient flow}.
\begin{lemma}
Let $\rho^N$ be the probability measure of $(\mu^N, \nu^N, P^N, Q^N)$. We restrict ourselves to a convergent subsequence $\rho^{N_k}$ which converges to some limit point $\rho = (\mu, \nu, P, Q)$. Then $\rho$ is a Dirac measure on $D_E([0,T])$ such that $(\mu, \nu, P, Q)$ satisfies the limiting ODEs \eqref{NN gradient flow}.
\end{lemma}

\begin{proof}
For any sequence of time-points $0\leq s_1 < s_2 < ... < s_p \leq t$, functions $\varphi, \bar{\varphi} \in C^2_b(\R^{1+d})$, $\phi_1,...,\phi_p, \bar{\phi}_1,...,\bar{\phi}_p \in C_b(\R^{1+d})$ and $\psi_1,...,\psi_p,\bar{\psi}_1,...,\bar{\psi}_p \in C_b(\X\times\A)$, and consider a map $F: D_E([0,T]) \to \R^+$, defined as 
\begin{equation}
F(\mu,\nu,P,Q) = F_1(\mu) + F_2(\nu) + F_3(\mu,\nu,P,Q) + F_4(\mu,\nu,P,Q),
\end{equation}
where
\begin{align*}
F_1(\mu) &= \prod_{j=1}^p |\langle \bar{\phi}_j, \mu_{s_j} \rangle| \times |\langle \bar{\varphi}, \mu_t \rangle - \langle \bar{\varphi}, \nu_0 \rangle|, \\
F_2(\nu) &= \prod_{j=1}^p |\langle \phi_j, \nu_{s_j} \rangle| \times |\langle \varphi, \nu_t \rangle - \langle \varphi, \nu_0 \rangle|, \\
F_3(\mu,\nu,P,Q) &= \prod_{j=1}^p |\psi_j(Q_{s_j})| \times \sum_{\xi\in\X\times\A} \Bigg| Q_t(\xi) - Q_0(\xi) \\
&\phantom{=}- \alpha \int_0^t \sum_{\xi' = (x',a')} \left(r(\xi') + \gamma \sum_{z,a''} Q_s(z,a'')g_s(z,a'')p(z|\xi') - Q_s(x',a')\right) [A(\nu_s)]_{\xi, \xi'} \, \pi^{g_s}(\xi') \, ds \Bigg|, \\
F_4(\mu,\nu,P,Q) &= \prod_{j=1}^p |\bar{\psi}_j(P_{s_j})| \times \sum_{\xi\in\X\times\A} \Bigg| P_t(\xi) - P_0(\xi) \\
&\phantom{=} - \int_0^t \sum_{\xi' = (x',a')} \zeta_s \clip(Q_s(\xi')) \sigma^{g_s}(\xi') \Bigg([A(\mu_s)]_{\xi,\xi'} - \sum_{a''} f_s(x',a'') [A(\mu_s)]_{\xi,(x',a'')} \Bigg) \, ds \Bigg|, \\
f_t &= \mathrm{Softmax}(P_t), \quad g_t = \frac{\eta_t}{\#\bm{\A}} + (1-\eta_t) f_t.
\end{align*}
Then we have
\begin{equation*}
\mathbb{E}_{\rho^N}[F(\mu,\nu,P,Q)] = \mathbb{E}[F(\mu^N, \nu^N, P^N, Q^N)].
\end{equation*}

Let us analyse each terms of $\E\sqbracket{F(\mu^N, \nu^N, P^N, Q^N)}$ one by one. Firstly, \eqref{eq:bound_for_weak_convergence_mu} and the boundedness of $\bar{\phi}_j$ yields
\begin{align*}
\E[F_1(\mu^N)] 
&\leq C \E\abs{\la \bar{\varphi}, \mu^N_t \ra - \la \bar{\varphi}, \nu_0 \ra} \\
&\leq C \E\abs{\la \bar{\varphi}, \mu^N_t \ra - \la \bar{\varphi}, \mu^N_0 \ra} + C\E\abs{\la \bar{\varphi}, \mu^N_0 \ra - \la \bar{\varphi}, \nu_0 \ra} \\
&\leq \frac{C_T}{\sqrt{N}} + \frac{C_T}{N^{3/2}} \overset{N\to\infty}\to 0, \numberthis
\end{align*}
noting that
\begin{align*}
\E|\la \bar{\varphi}, \mu^N_0 \ra - \la \bar{\varphi}, \nu_0 \ra| &\leq \E \left|\frac{1}{N} \sum_{i=1}^N \bar{\varphi}(B^i_0, W^i_0) - \la \bar{\varphi}, \nu_0 \ra \right| \\
&\leq \left[\E \left|\frac{1}{N} \sum_{i=1}^N \bar{\varphi}(B^i_0, W^i_0) - \la \bar{\varphi}, \nu_0 \ra \right|^2 \right]^{1/2} \\
&= \frac{1}{\sqrt{N}} \sqrt{\E\left[\bar{\varphi}(B^1_0, W^1_0) - \la \bar{\varphi}, \nu_0 \ra \right]^2} \\
&\leq \frac{C}{\sqrt{N}}.
\end{align*}
Similarly, \eqref{eq:bound_for_weak_convergence_mu} and the boundedness of $\phi_j$ yields
\begin{align*}
\E[F_2(\nu^N)] 
&\leq C \E\abs{\la \varphi, \nu^N_t \ra - \la \varphi, \nu_0 \ra} \\
&\leq C \E\abs{\la \varphi, \nu^N_t \ra - \la \varphi, \nu^N_0 \ra}+ C\E\abs{\la \varphi, \nu^N_0 \ra - \la \varphi, \nu_0 \ra} \\
&\leq \frac{C_T}{\sqrt{N}} + \frac{C_T}{N^{3/2}} \overset{N\to\infty}\to 0. \numberthis
\end{align*}
To study the next two term, we define
\begin{equation*}
f^N_t = \mathrm{Softmax}(P^N_t), \quad \tilde{g}^N_t = \frac{\eta_t}{\#\bm{\A}} + (1-\eta_t) f^N_t,
\end{equation*}
\begin{align*}
E^{1,N}_t(\xi) &= \int_0^t \sum_{\xi'} [A(\nu^N_s)]_{\xi,\xi'} (\pi^{\tilde{g}^N_s}(\xi') - \pi^{g^N_s}(\xi')) \sqbracket{r(\xi') + \gamma \sum_{z,a''} Q^N_s(z,a'') \tilde{g}^N_s(z,a'')p(z|\xi') - Q^N_s(\xi')} \, ds, \\
E^{2,N}_t(\xi) &= \int_0^t \sum_{\xi'}[A(\nu^N_s)]_{\xi,\xi'} \pi^{g^N_s}(\xi') \sum_{z,a''} \Big[ \gamma Q^N_s(z,a'') (\tilde{g}^N_s(z,a'') - g^N_s(z,a''))p(z|\xi') \Big] \, ds.
\end{align*}
Then by \eqref{eq:Q_pre_limit}:
\begin{align*}
&\phantom{=}F_3(\mu^N, \nu^N, P^N, Q^N) \\
&= \prod_{j=1}^p |\psi_j(Q_{s_j})| \times \sum_{\xi \in \X \times \A} \Bigg| Q^N_t(\xi) - Q^N_0(\xi) \\
&\phantom{=} - \alpha \int_0^t \sum_{\xi'} [A(\nu^N_s)]^N_{\xi,\xi'} \, \pi^{\tilde{g}^N_s}(\xi') \, \Bigg[r(\xi') + \gamma \sum_{z,a''} Q^N_s(z,a'') \tilde{g}^N_s(z,a'')p(z|\xi') - Q^N_s(\xi') \Bigg] \, ds \Bigg| \\
&= \prod_{j=1}^p |\psi_j(Q_{s_j})| \times \sum_{\xi \in \X \times \A} \Bigg| Q^N_t(\xi) - Q^N_0(\xi) \\
&\phantom{=}- \alpha \int_0^t \sum_{\xi'} [A(\nu^N_s)]_{\xi,\xi'} \pi^{g^N_s}(\xi') \bigg[r(\xi') + \gamma \sum_{z,a''} Q^N_s(z,a'') g^N_s(z,a'')p(z|\xi') - Q^N_s(\xi') \bigg] \, ds + E^{1,N}_t(\xi) + E^{2,N}_t(\xi) \Bigg| \\
&\overset{\eqref{eq:Q_pre_limit}}= \prod_{j=1}^p |\psi_j(Q_{s_j})| \times \sum_{\xi \in \bm{\X}\times \bm{\A}} \alpha \abs{M^{1,N}_t(\xi) + M^{2,N}_t(\xi) + M^{3,N}_t(\xi) + E^{1,N}_t(\xi) + E^{2,N}_t(\xi) + O_p(N^{-1/2})}. \numberthis
\end{align*}
Since $\eta_t$ is Lipschitz, hence there exists a constant $C > 0$ such that
\begin{equation} \label{eq:diff_of_discretised_eta}
|\eta^N_{\lfloor Ns \rfloor} - \eta_s| = |\eta_{\frac{\lfloor Ns \rfloor}{N}} - \eta_s| \leq \frac{C}{N},
\end{equation}
and therefore, by Proposition \ref{prop:Lipschitzness_of_stationary_measures}, there exists another $C_T > 0$ such that
\begin{equation}
|\pi^{\tilde{g}^N_s}(\xi') - \pi^{g^N_s}(\xi')| \leq \frac{C_T}{N}. \numberthis \label{eq:difference_g_g_tilde}
\end{equation}
Therefore, we have
\begin{align*}
\E[E^{1,N}_t(\xi)] &\leq \frac{C_T}{N} \E\sqbracket{\int_0^t \sum_{\xi'} |[A(\nu^N_s)]_{\xi,\xi'}| \sqbracket{|r(\xi')| + \gamma \sum_{z,a''} |Q^N_s(z,a'')| |g^N_s(z,a'')| p(z|\xi') - Q^N_s(\xi')} \, ds} \\
&\leq \frac{1}{N} \E\sqbracket{TC_T \sup_{\xi}|Q^N_s(\xi)|} \leq \frac{C_T}{N},
\end{align*}
and
\begin{align*}
E^{2,N}_t(\xi) &= \frac{C}{N} \E\sqbracket{\int_0^t \sum_{\xi'} [A(\nu^N_s)]_{\xi,\xi'} \pi^{g^N_s}(\xi') \gamma \sum_{z,a''} |Q^N_s(z,a'')| p(z|\xi') \, ds} \leq \frac{1}{N} \E\sqbracket{TC_T \sup_{\xi}|Q^N_s(\xi)|} \leq \frac{C_T}{N}.
\end{align*}
Finally, we have
\begin{equation*}
\E[F_3(\mu^N, \nu^N, P^N, Q^N)] \leq C \sum_{\xi} \sqbracket{\E|M^{1,N}_t(\xi)| + \E|M^{2,N}_t(\xi)| + \E|M^{3,N}_t(\xi)| + \E|E^{1,N}_t(\xi)| + \E|E^{2,N}_t(\xi)|} \overset{N\to\infty}\to 0.
\end{equation*}
To study the final term, we define
\begin{align*}
E^{3,N}_t(\xi) = \int_0^t \zeta_s \sum_{\xi'} (\sigma_{\rho_0}^{\tilde{g}^N_s}(\xi') - \sigma_{\rho_0}^{g^N_s}(\xi')) \, \clip(Q^N_s(\xi')) \left[[A(\mu^N_s)]_{\xi, \xi'} - \sum_{a''} f^N_s(x',a'') [A(\mu^N_s)]_{\xi, (x', a'')} \right] ds, \\
E^{4,N}_t(\xi) = \int_0^t (\zeta^N_{\floor{Ns}} - \zeta_s) \sum_{\xi'} \sigma_{\rho_0}^{g^N_s}(\xi') \, \clip(Q^N_s(\xi')) \left[[A(\mu^N_s)]_{\xi, \xi'} - \sum_{a''} f^N_s(x',a'') [A(\mu^N_s)]_{\xi, (x', a'')} \right] ds,
\end{align*}
Then
\begin{align*}
&\phantom{=}F_4(\mu^N,\nu^N,P^N,Q^N) \\
&= \prod_{j=1}^p |\bar{\psi}_j (P_{s_j})| \times \sum_{\xi \in \X \times \A} \Bigg| P^N_t(\xi) - P^N_0(\xi) \\
&\phantom{=}- \int_0^t \sum_{\xi'} \zeta_s \clip(Q^N_s(\xi')) \sigma^{\tilde{g}^N_s}(\xi') \left[[A(\mu^N_s)]_{\xi, \xi'} - \sum_{a''} f^N_s(x',a'')[A(\mu^N_s)]_{\xi, (x', a'')} \right] \, ds \Bigg| \\
&= \prod_{j=1}^p |\bar{\psi}_j (P_{s_j})| \times \sum_{\xi\in\X\times\A} \Bigg| P^N_t(\xi) - P^N_0(\xi) \\
&\phantom{=}- \int_0^t \sum_{\xi'} \zeta_{\floor{Ns}} \clip(Q^N_s(\xi')) \sigma^{g^N_s}(\xi') \left[[A(\mu^N_s)]^N_{\xi, \xi'} - \sum_{a''} f^N_s(x',a'')[A(\mu^N_s)]_{\xi, (x', a'')} \right] \, ds + E^{3,N}_t(\xi) + E^{4,N}_t(\xi) \Bigg|,\\
&= \prod_{j=1}^p |\bar{\psi}_j (P_{s_j})| \times \sum_{\xi\in\X\times\A} |E^{3,N}_t(\xi) + E^{4,N}_t(\xi) + M^N_t(\xi) + O(N^{-1/2})|.
\end{align*}
By \eqref{eq:diff_of_discretised_eta} and statement (2) of Proposition \ref{prop:Lipschitzness_of_stationary_measures}, we have
\begin{equation*}
|\sigma^{\tilde{g}^N_s}_{\rho_0}(\xi') - \sigma^{g^N_s}_{\rho_0}(\xi')| \leq \frac{4}{(1-\gamma)^2 N}.
\end{equation*}
In addition, we have
\begin{align*}
\sup_{\xi,\xi'} \abs{[A(\mu^N_s)]_{\xi, \xi'} - \sum_{a''} f^N_s(x',a'') [A(\nu^N_s)]_{\xi, (x', a'')}} 
&\leq \sup_{\xi, \xi'} \sqbracket{\abs{[A(\nu^N_s)]_{\xi, \xi'}} + \sum_{a''} f^N_s(x',a'') \abs{[A(\nu^N_s)]_{\xi, (x', a'')}}} \\
&\leq C_T,
\end{align*}
as a result of $\bar{B}^N_{\xi,\xi',s}$ being uniformly bounded by Lemma \ref{lem:diff_of_B_bar_B} whenever $s \leq T$. Therefore for any $t \leq T$,
\begin{align*}
E^{3,N}_t(\xi) \leq T \times \frac{C_T}{N} \times 2 \times C_T = \frac{C_T}{N}.
\end{align*}
Similarly,
\begin{align*}
\abs{E^{4,N}_t(\xi)} &\leq C_T \int_0^T \abs{\zeta^N_{\floor{Ns}} - \zeta_s} \, ds \\
&\leq \sum^{\floor{NT}-1}_{k=0} \int_{k/N}^{(k+1)/N} \abs{\frac{1}{1+k/N} - \frac{1}{1+s}} \, ds \\
&\leq C_T \sum^{\floor{NT}-1}_{k=0} \frac{1}{N^2} = \frac{C_T}{N}.
\end{align*}
Combining with the boundedness of $\tilde{\phi}_p$, we have
\begin{equation}
F_4(\mu^N,\nu^N,P^N,Q^N) \leq C\sum_\xi \sqbracket{\E|E^{3,N}_t(\xi)| + \E|E^{4,N}_t(\xi)| + \E|M^{N}_t(\xi)| + O(N^{-1/2})} \overset{N\to\infty}\to 0.
\end{equation}
Combining the above analysis yields:
\begin{equation*}
\e_{\rho^N}[F(\mu,\nu,P,Q)] \overset{N\to\infty}\to 0.
\end{equation*}
But since $F$ is uniformly bounded, by bounded convergence theorem, we have
\begin{align*}
\E_{\rho}[F(\mu,\nu,P,Q)] = 0.
\end{align*}
This holds for any choice of the test functions $\varphi, \bar{\varphi}, \phi_j, \bar{\phi}_j, \psi_j, \bar{\psi}_j$, so we know that $\rho$ is a Dirac measure concentrated on a solution that satisfies the evolution equation.
\end{proof}

\subsection{Existence and uniqueness of solutions to limit ODEs}

\hspace{1.4em} To complete the proof, if suffices to show that there exists a unique solution for the ODEs \eqref{NN gradient flow}. Here we treat $(Q,P)$ as a vector of size $2M$ with $M=|\X\times\A|$. 
\begin{equation}
\label{limit ODE fix starting}
\frac{d}{dt} \begin{pmatrix} Q_t \\ P_t \end{pmatrix} = F(t,Q_t,P_t) = \begin{pmatrix} F_1(t,Q_t,P_t) \\ F_2(t,Q_t,P_t) \end{pmatrix}
\end{equation}
where the first $M$ entries $F(Q,P)$ are specified as
\begin{align*}
F_1(t,Q,P)(x,a) = \alpha \sum_{x', a'} \bar A_{x,a,x',a'} \pi^{g_t(P)}(x', a') \left(r(x', a') + \gamma \sum_{z, a''} Q(z, a'') [g_t(P)](z,a'') p(z| x', a') - Q(x', a') \right)
\end{align*}
and the remaining $M$ entries are specified as
\begin{align*}
F_2(t,Q,P)(x,a) = \sum_{x', a'} \zeta_t \clip(Q(x', a')) \left[A_{x, a, x', a'} - \sum_{a''} [f(P)](x',a'')  A_{x,a,x',a''} \right] \sigma^{g_t(P)}(x', a').
\end{align*}
Here the notation $f(P)$ and $g_t(P)$ denote the (probability) vectors in $\R^{M}$:
\begin{align*}
[f(P)](x,a) &= \frac{\exp\bracket{P(x,a)}}{\sum_{a''} \exp\bracket{P(x,a)}} \\
[g_t(P)](x,a) &= \frac{\eta_t}{\#\bm{\A}} + (1-\eta_t) [f(P)](x,a).
\end{align*}
We will show the global existence of a solution for $t \in [0,\infty)$ by taking the usual route of showing that $F(Q,P)$ is locally Lipschitz and linearly bounded.

\begin{lemma}
Let $\norm{\cdot}_\infty$ the the infinity norm of a vector in $\R^{2M}$, that is, the maximum of the absolute value of the entries of the vector. Then for all $R > 0$, there is a constant $C_R > 0$ that only depends on $R$ such that for all $(Q,P), (\tilde{Q}, \tilde{P})$ lying in the open $R$-ball, we have
\begin{equation}
\label{eq:local_lipschitz}
\norm{F(t,Q,P) - F(t,\tilde{Q},\tilde{P})}_\infty \leq C_R \norm{(Q,P) - (\tilde{Q}, \tilde{P})}_\infty, \quad \forall t \geq 0.  
\end{equation}
Moreover, there is a constant $C > 0$ such that for all $Q,P$, we have
\begin{equation}
\label{eq:linearly_bounded}
\norm{F(t,Q,P)}_\infty \leq C \norm{(Q,P)}_\infty + C, \quad \forall t \geq 0.
\end{equation}
Therefore, $F$ is locally Lipschitz and linearly bounded and for any fixed starting point $(Q_0, P_0)$, there exists the unique solution for ODE 
\eqref{limit ODE fix starting}.
\end{lemma}

\begin{remark}
We note that the choice of norm in equations \eqref{eq:local_lipschitz} and \eqref{eq:linearly_bounded} is not important. If these equations are true with infinity norm, then they will also be true if we replace the infinite norm with the usual Euclidean norm $\norm{\cdot}$, since they are equivalent due to $\X\times\A$ being finite-dimensional.
\end{remark}

\begin{proof}
Let us first prove equation \eqref{eq:linearly_bounded}. Note that the tensor $\bar A_{\xi,\xi'}$ is uniformly bounded by assumptions \ref{as:activation_function} and \ref{as:NN_condition}. Thus 
\begin{align*}
|F_1(t,Q,P)(x,a)| &\leq C \sum_{x', a'} \pi^{g_t(P)}(x',a') \bracket{|r(x', a')| + \gamma \sum_{z, a''} |Q(z, a'')| g(z,a'') p(z|x', a') + |Q(x', a')|} \\
&\leq C \sup_{x',a'} |r(x',a')| + C\gamma \sup_{z,a''} |Q(z,a'')| + C\gamma \sup_{x', a'} |Q(x',a')| \\
&\leq C + C \norm{(Q,P)}.
\end{align*}
It is also clear that
\begin{align*}
|F_2(t,Q,P)(x,a)| &\leq C \sup_{x, a} |\clip(Q(x,a))| \leq C
\end{align*}
This shows that $F$ is linearly bounded. 

To prove the local Lipschitz condition \eqref{eq:local_lipschitz}, note that for all $x,a$,
\begin{equation}
\begin{aligned}
\label{F1 diff}
&\abs{F_1(t, Q,P)(x, a) - F_1(t, \tilde{Q},\tilde{P})(x,a)} \\
\leq& \alpha \sum_{x', a'} |A_{x,a,x',a'}| \abs{\pi^{g_t(P)}(x', a') - \pi^{g_t(\tilde{P})}(x', a')} \underbrace{\abs{r(x', a') + \gamma \sum_{z, a''} Q(z, a'') [g_t(P)](z,a'') p(z| x', a') - Q(x', a')}}_{\leq C + (\gamma+1)R} \\
+& \alpha \sum_{x', a'} |A_{x,a,x',a'}|\pi^{g_t(\tilde{P})}(x', a') \abs{\sum_{z, a''} \gamma(Q(z, a'')[g_t(P)](z,a'') - \tilde{Q}(z,a'') [g_t(\tilde{P})](z,a''))p(z| x', a') - (Q(x', a') - \tilde{Q}(x', a'))}.
\end{aligned}
\end{equation}
Note that $\|(P,Q)\|_\infty \leq R$, we have $[f(P)](x,a) \geq e^{-R}$. Therefore, using statement (2) of Proposition \ref{prop:Lipschitzness_of_stationary_measures} and following arguments in Appendix \ref{SS:lipschitzness_of_the_stationary_measures}, we know that there exists $C_R > 0$ (which varies line by line) such that
\begin{align*}
\sup_{x,a} |[\pi^{g_t(P)}](x,a) - [\pi^{g_t(\tilde{P})}](x,a)|
&\leq C_R \sup_{x,a} |[g_t(P)](x,a) - [g_t(\tilde{P})](x,a)| \\
&= C_R \sup_{x,a} |[f(P)](x,a) - [f(\tilde{P})](x,a)| \\
&\leq C_R \norm{P - \tilde{P}}. \numberthis \label{F1 first term}
\end{align*}
Note that for all $z, a''$
\begin{equation}
\begin{aligned}
\label{F1 second term}
&\abs{Q(z, a'')[g_t(P)](z,a'') - \tilde{Q}(z,a'') [g_t(\tilde{P})](z,a'')} \\
\leq& \abs{Q(z, a'')} \cdot \abs{[g_t(P)](z,a'') - [g_t(\tilde{P})](z,a'')} + [g_t(\tilde{P})](z,a'') \cdot \abs{Q(z, a'') - \tilde{Q}(z,a'')} \\
\leq& CR\bracket{\sup_{z,a''} \abs{P(z,a'')-\tilde{P}(z,a'')}} + \sup_{z,a''}\abs{Q(z,a'')-\tilde{Q}(z,a'')} \\
\leq& CR \norm{(Q,P)-(\tilde{Q},\tilde{P})}.
\end{aligned}
\end{equation}
Combining \eqref{F1 diff}, \eqref{F1 first term} and \eqref{F1 second term}, we have 
\beq 
\abs{[F_1(t, Q, P)](x,a) - [F_1(t, \tilde{Q}, \tilde{P})]} \leq C_R \norm{(Q,P)-(\tilde{Q},\tilde{P})}.
\eeq
Similarly for $F_2$,
\begin{equation}
\begin{aligned}
\label{F2 diff}
& \abs{[F_2(t, Q,P)](x, a) - [F_2(t, \tilde{Q},\tilde{P})](x,a)} \\
\leq& \sum_{x', a'} \zeta_t \abs{\clip(Q(x', a'))-\clip(\tilde{Q}(x', a'))} \sigma^{g_t(P)}(x', a') \abs{A_{x, a, x', a'} - \sum_{a''} [f(P)](x',a'')  A_{x,a,x',a''}} \\
+& \sum_{x', a'} \zeta_t \abs{\clip(\tilde{Q}(x', a'))} \abs{\sigma^{g_t(P)}(x', a') - \sigma^{g_t(\tilde{P})}(x', a')} \abs{A_{x, a, x', a'} - \sum_{a''} [f(P)](x',a'')  A_{x,a,x',a''}} \\
+& \sum_{x', a'} \zeta_t \abs{\clip(\tilde{Q}(x', a'))} \sigma^{g_t(\tilde{P})}(x', a') \abs{\sum_{a''} ([f(P)](x',a'') - [f(\tilde{P})](x',a'')) A_{x,a,x',a''}} \\
\leq& C \norm{(Q,P)-(\tilde{Q},\tilde{P})}.
\end{aligned}
\end{equation}
We therefore show that $F$ is locally Lipschitz if we restrict $(Q,P)$ to be inside a $R$-ball for any $R<\infty$.

The linear boundedness of $F$ can guarantee that the solution grows almost exponentially. In fact, we have
\begin{equation}
\norm{(Q_t,P_t)} \leq \norm{(Q_0, P_0)} + \int_0^t (C + \norm{(Q_s, P_s)} C) \, ds \leq (\norm{(Q_0, P_0)} + Ct) + C \int_0^t \norm{(Q_s,P_s)} \, ds.
\end{equation}
which together with Gronwall's inequality derive 
\beq
\label{exp growth}
\norm{(Q_t,P_t)} \leq (\norm{(Q_0, P_0)} + Ct) e^{Ct}.
\eeq
Suppose the above evolution equation possesses two solutions $(Q, P)_t, (\tilde{Q},\tilde{P})_t$ that satisfies $Q_0 = \tilde{Q}_0$ and $P_0 = \tilde{P}_0$. Then we have
\begin{equation*}
\frac{d}{dt}\norm{(Q_t,P_t) - (\tilde{Q}_t, \tilde{P}_t)}^2 \leq 2 \norm{(Q_t,P_t) - (\tilde{Q}_t, \tilde{P}_t)} \cdot \norm{F(t,Q_t,P_t) - F(t,\tilde{Q}_t, \tilde{P}_t)}.
\end{equation*}
Using \eqref{F1 first term}, \eqref{F2 diff}, \eqref{exp growth} and replacing $R$ in \eqref{F1 first term} by the norm $\norm{(Q_t,P_t)}$ in \eqref{exp growth}, we can show that 
\beq
\frac{d}{dt}\norm{(Q_t,P_t) - (\tilde{Q}_t, \tilde{P}_t)}^2 \leq \underbrace{(C + (C + Ct) e^{Ct})}_{H(t)} \norm{(Q_t,P_t) - (\tilde{Q}_t, \tilde{P}_t)}^2.
\eeq
Therefore, by Gronwall's inequality, we have
\begin{equation*}
\norm{(Q_t,P_t) - (\tilde{Q}_t, \tilde{P}_t)}^2 \leq \norm{(Q_0,P_0) - (\tilde{Q}_0, \tilde{P}_0)}^2 \exp\bracket{\int_0^t H(s) \, ds} = 0,
\end{equation*}
which guarantee uniqueness.
\end{proof}

\subsection{Proof of convergence}
\hspace{1.4em} With the above preparations, now we can finish the proof of Theorem \ref{limit odes}. Rrecall the sequence of probability measure $\rho^N$ being the law of $\left(\mu_t^N, \nu_t^N, P_t^N, Q_t^N\right)_{0\le t \le T}$. We have shown by relative compactness that every subsequence of $\rho^N$ posesses a further subsequence that weakly converges to the $\rho = (\mu, \nu, P, Q)$, which is the unique solution of the limit ODEs \eqref{NN gradient flow}. Therefore by Prokhorov's Theorem (see \cite{billingsley2013convergence, ethier2009markov} for details), $\rho^N$ weakly converges to $\rho$, and thus we can conclude that the process $\left(\mu_t^N, \nu_t^N, P_t^N, Q_t^N\right)_{0\le t \le T}$ weakly converges to $\rho$.

\section{Analysis of the limiting ODE}
\label{s:analysis_of_the_limiting_ODE}

\hspace{1.4em} We have already set up the limit ODEs for the algorithm \eqref{alg:onlineNAC} and now we study the convergence of the limit ODEs \eqref{NN gradient flow}. To improve the readilibity, we first clarify some notations. 

\begin{itemize}
\item From their definitions in \eqref{value function}, $\bar{V}^{f}(x)$ and $V^{f}(x,a)$ are related via the formula
\begin{equation}
\bar{V}^{f}(x)=\sum_a V^{f}(x, a)f(x,a).
\end{equation}
\item Recalling the state and state-action visiting measures $\nu^{f}$ and $\sigma^{f}$ defined in \eqref{visiting}, we have
$\sigma_\mu^{f}(x,a)=f(x,a) \cdot \nu_\mu^{f}(x)$. By \cite{konda2002actor}, the stationary distribution of $\bm{\tilde{\mathcal{M}}}$ is the corresponding visitation measure of $\bm{\mathcal{M}}$. And for the MDP start from a fixed state $x_0$, the visiting measures are denoted by $\nu^f_{x_0}(\cdot), \sigma^f_{x_0}(\cdot, \cdot)$
\item Let the advantage function of policy $f$ denoted by
\begin{equation}
A^{f}(x, a) = V^{f}(x, a) - \bar{V}^{f}(x),\quad \forall (x, a) \in \bm{\mathcal{X}} \times \bm{\mathcal{A}},
\end{equation}
\end{itemize}

We recall that the gradient of a policy $f$ parametrised by some parameter $\theta$ can be evaluated in terms of the visiting measure \eqref{visiting} according to the policy gradient theorem \eqref{eq:policy_gradient}:
\beq
\label{policy gradient theorem}
\nabla_{\theta} J(f_{\theta}) = \sum_{x,a} \sigma^{f_{\theta}}(x,a) V^{f_{\theta}}(x,a) \nabla_{\theta} \log f_{\theta}(x,a),
\eeq
Assume that $f = \text{softmax}(P)$ be the softmax policy parametrised directly by the $(x,a)$-entries of $P$, denoted as $P(x,a)$, so that
\beq
\label{softmax NN policy}
f(x,a) = \frac{\exp\bracket{P(x,a)}}{\sum_{a''} \exp\bracket{P(x,a)}}.
\eeq
Then the gradient $\nabla_P J(\mathrm{softmax}(P))$ can be evaluated using the following formula.
\begin{lemma}
\label{softmax gradient}
Define $\displaystyle{\partial_{x,a} J(f) := \frac{\partial J(f)}{\partial P(x,a)}}$, where $f = \mathrm{softmax}(P)$,   Then
\begin{equation}
\label{policy gradient}
\partial_{x,a} J(f) =  \sigma_{\rho_0}^{f}(x,a) A^{f}(x,a).
\end{equation}
\end{lemma}

\begin{proof}
By the Policy Gradient Theorem \eqref{policy gradient theorem}, we have 
\begin{align*}
\partial_{x,a} J(f) &=  \sum_{x', a'} \nu_{\rho_0}^{f_\theta}(x') f(x', a') \mathsf{1}_{\{x' = x\}} \left[ \mathsf{1}_{\{a' = a\}} - f(x', a) \right] V^{f}(x', a') \\
&= \sum_{a'} \nu_{\rho_0}^{f}(x)f(x, a')  \left[ \mathsf{1}_{\{ a' = a\}} - f(x, a) \right] V^{f}(x, a') \\
&= \nu_{\rho_0}^{f}(x)f_\theta(x, a) V^{f}(x, a) -  \nu_{\rho_0}^{f}(x)f(x, a)  \left[ \sum_{a'}f(x, a') V^{f}(x, a') \right]\\
&= \nu_{\rho_0}^{f}(x)f(x, a) A^{f}(x, a)\\
&= \sigma_{\rho_0}^{f}(x, a) A^{f}(x, a). \numberthis \label{element}
\end{align*}
\end{proof}	

\subsection{Critic Convergence (Proof of Theorem \ref{thm:global_critic_convergence})}
We shall now prove the convergence of the critic \eqref{critic convergence}, which states that the critic model will converge to the action-value function during training. We first derive an ODE for the difference between the critic and the value function. Then, we use a comparison lemma, a two time-scale analysis, and the properties of the learning and exploration rates to prove the convergence of the critic to the value function. We shall focus on the case when $\eta_t = (1+t)^{-\varepsilon}$. The case when $\eta_t = (1+(\log(1+t)^2)^{-1}$ is deferred to Appendix \ref{SS:case_when_eta_is_log}.

The comparison lemma to be used is as follows:
\begin{lemma}[Comparison Lemma] \label{lem:comparison_new}
Consider a differentiable map $t \in [0,\infty) \mapsto Y_t$. Suppose $Y_t \geq 0$, and there exist constants $n_0, \mathsf{C}_1, \mathsf{C}_2, \mathsf{C}_3 > 0$ such that for all $t \geq 0$,
\begin{equation}
\frac{dY_t}{dt} \leq - \frac{\mathsf{C}_1 Y_t}{(1+t)^{\gamma_1}} + \frac{\mathsf{C}_2}{(1+t)^{\gamma_2}},
\end{equation}
where $\gamma_2 > 0$, then
\begin{itemize}
\item for $ \gamma_1 \in(0,1)$:
\begin{equation}
Y_t = O\left(\frac{1}{(1+t)^{\gamma_2-\gamma_1}} \right),
\end{equation}
\item for $\gamma_1 = 1$ and $\mathsf{C}_1-\gamma_2 \neq -1$:
\begin{equation}
Y_t = O\left(\frac{1}{(1+t)^{(\gamma_2-1) \wedge \mathsf{C}_1}} \right),
\end{equation}
\item for $\gamma_1 = 1$ and $\mathsf{C}_1-\gamma_2 = -1$:
\end{itemize}
\begin{equation}
Y_t = O\left(\frac{\log(1+t)}{(1+t)^{\mathsf{C}_1}} \right).
\end{equation}
\end{lemma}

\begin{proof}
See subsection \ref{SS:proof_of_comparison_principle}
\end{proof}

\begin{proof}[Proof of Theorem \ref{thm:global_critic_convergence}]
We shall decompose $\phi_t = \tilde{\phi}_t +  V^{g_t} - V^{f_t}$, where $\tilde{\phi_t} = Q_t - V^{g_t}$. Then, we have

\begin{align}
\frac{d\tilde{\phi}_t}{dt}(\xi) 
&=\alpha \sum_{\xi'} A_{\xi,\xi'} \pi^{g_t}(\xi') \Bigg[-\tilde{\phi}_t(\xi') + \gamma \sum_{z, a''} \tilde{\phi}_t(z, a'') g_t(z,a'') p(z|\xi') \Bigg] - \frac{d}{dt} V^{g_t}(\xi). \label{eq:dphi_dt_new}
\end{align}
Define
$$\tilde{Y}_t = \frac{1}{2} \tilde{\phi}_t A^{-1} \tilde{\phi}_t,$$
then
then
\begin{equation}
\frac{d\tilde{Y}_t}{dt} = - \alpha \sum_{\xi'} \pi^{g_t}(\xi') (\tilde{\phi}_t(\xi'))^2 + \alpha \gamma \sum_{\xi'} \tilde{\phi}_t(\xi') \pi^{g_t}(\xi') \sum_{z, a''} \tilde{\phi}_t(z, a'') g_t(z, a'') p(z | \xi') \, - \, \tilde{\phi}_t^{\top} A^{-1} \frac{dV^{g_t}}{dt}. \label{eq:Yeqn0new}
\end{equation}
The second term on the last line of \eqref{eq:Yeqn0new} becomes:
\begin{align}
&\phantom{=} \bigg{|} \sum_{\xi'=(x',a')} \tilde{\phi}_t(\xi') \pi^{g_t}(\xi') \sum_{z, a''} \tilde{\phi}_t(z, a'') g_t(z, a'') p(z\mid\xi') \bigg{|} \notag \\
&= \bigg{|} \sum_{\xi'=(x',a')} \sum_{z, a''} \tilde{\phi}_t(z, a'') \tilde{\phi}_t(\xi')  g_t(z, a'') p(z\mid \xi') \pi^{g_t}(\xi')  \bigg{|} \notag \\
&\leq \sum_{\xi'=(x',a')} \sum_{z, a''}  \bigg{|} \tilde{\phi}_t(z, a'') \tilde{\phi}_t(\xi') \bigg{|} g_t(z, a'') p(z | \xi') \pi^{g_t}(\xi') \notag \\
&\leq \frac{1}{2} \sum_{\xi'=(x',a')} \sum_{z, a''} \bigg{(} \tilde{\phi}_t(z, a'')^2 + \tilde{\phi}_t(\xi')^2 \bigg{)} g_t(z, a'') p(z | \xi') \pi^{g_t}(\xi') \notag \\
&= \frac{1}{2} \sum_{z, a''} \tilde{\phi}_t(z, a'')^2 \sum_{\xi'=(x',a')} g_t(z, a'') p(z | \xi') \pi^{g_t}(\xi') + \frac{1}{2} \sum_{\xi'=(x',a')} \tilde{\phi}_t(\xi')^2 \pi^{g_t}(\xi')  \sum_{z, a''}  g_t(z, a'') p(z | \xi') \notag \\
&= \frac{1}{2} \sum_{z, a''} \tilde{\phi}_t(z, a'')^2 \pi^{g_t}(z, a'')  + \frac{1}{2}  \sum_{\xi'=(x',a')} \tilde{\phi}_t(\xi')^2 \pi^{g_t}(\xi') \notag \\
&= \sum_{\xi=(x',a')} \phi_t(\xi')^2 \pi^{g_t}(\xi').     
\end{align}
where we have used Young's inequality, the fact that $\displaystyle \sum_{z, a''} g_t(z, a'') p(z | \xi')  = 1$ for each $\xi'=(x',a')$, and $\displaystyle \sum_{\xi'=(x',a')}  g_t(z, a'') p(z | \xi') \pi^{g_t}(\xi')  = \pi^{g_t}(z, a'')$. Therefore, with the lower bound $\min_{\xi'} \pi^{g_t}(\xi') \geq C_{\mathsf{1}} \eta_t^{n_0}$ from \eqref{eq:lower_bound_MDP_general}, we  have
\begin{align}
\frac{d\tilde{Y}_t}{dt} 
&\leq -\alpha(1-\gamma) \sum_{\xi'} \pi^{g_t}(\xi') (\tilde{\phi}(\xi'))^2 - \tilde{\phi}^\top A^{-1} \frac{dV^{g_t}}{dt} \nonumber \\
&\leq -\alpha(1-\gamma)C_{\mathsf{1}} \eta_t^{n_0} \sum_{\xi'} (\tilde{\phi}(\xi'))^2 - \tilde{\phi}^\top A^{-1} \frac{dV^{g_t}}{dt}. \label{eq:Yeqn_new}
\end{align}
Moreover, we have
\begin{align}
\tilde{\phi}_t^\top A^{-1} \frac{dV^{g_t}}{dt} 
&\leq \frac{CM}{\sigma_{\min}(A)} \|\tilde{\phi}_t\| \zeta^{\mathsf{actor}}_t \nonumber \\
&= \left(\left(\alpha(1-\gamma) C_{\mathsf{1}} \right)^{1/2} \eta^{n_0/2}_t \right) \left(\frac{CM}{(\alpha (1-\gamma) C_{\mathsf{1}})^{1/2} \sigma_{\min}(A)} \frac{\zeta_t}{\eta^{n_0/2}_t} \|\tilde{\phi}_t\| \right) \nonumber \\
&\leq \frac{\alpha(1-\gamma) C_{\mathsf{1}}}{2} \eta^{n_0}_t \|\tilde{\phi}_t\|^2 + \frac{C^2 M^2}{\alpha (1-\gamma) C_{\mathsf{1}} \sigma_{\min}^{2}(A)} \frac{\zeta_t}{\eta^{n_0}_t}
\label{eq:zetaBound_new}
\end{align}
Combining \eqref{eq:Yeqn_new} with \eqref{eq:zetaBound_new}, we have 
\begin{align}
\frac{d\tilde{Y}_t}{dt} &\leq -\frac{\alpha(1-\gamma)C_{\mathsf{1}}}{2} \eta_t^{n_0} \|\tilde{\phi}_t \|^2 + \frac{C^2 M^2}{\alpha (1-\gamma) C_{\mathsf{1}} \sigma_{\min}^{2}(A)} \frac{\zeta_t^2}{\eta^{n_0}_t} \nonumber \\ 
&\leq -\frac{\alpha(1-\gamma) C_{\mathsf{1}}}{2(1+t)^{n_0\varepsilon}} \|\tilde{\phi}_t\|^2 + \frac{C^2 M^2}{\alpha (1-\gamma) C_{\mathsf{1}} \sigma_{\min}^2(A)} \frac{1}{(1+t)^{2\beta - n_0\varepsilon}}.
\end{align}
Assume that $\beta_a > n_0\varepsilon$, then by Lemma \ref{lem:comparison_new}, we have
\begin{equation}
\tilde{Y}_t = O\left(\frac{1}{(1+t)^{2(\beta-n_0\varepsilon)}} \right).
\end{equation}
This also means that 
\begin{equation}
\|\tilde{\phi}_t \|^2 \leq \frac{1}{\sigma_{\min}(A^{-1})} \tilde{Y}_t = \sigma_{\max}(A) \tilde{Y}_t = O\left(\frac{1}{(1+t)^{2(\beta-n_0\varepsilon)}}\right).
\end{equation}
We shall now bound $V^{g_t} - V^{f_t}$. Note that
\begin{align*}
V^{g_t}(\xi) - V^{f_t}(\xi) &= \gamma \sum_{\xi'} (V^{g_t}(\xi') - V^{f_t}(\xi')) g_t(\xi') p(x'|\xi)  + \gamma \sum_{\xi'} V^{f_t}(\xi') (g_t(\xi')-f_t(
\xi') p(x'|\xi) \\
&=\gamma \sum_{\xi'} (V^{g_t}(\xi') - V^{f_t}(\xi')) g_t(\xi') p(x'|\xi)  + \gamma \eta_t \sum_{\xi'} V^{f_t}(\xi') \left(\frac{1}{\#\bm{\A}}-f_t(
\xi') \right) p(x'|\xi).
\end{align*}
Therefore
\begin{align*}
|V^{g_t}(\xi) - V^{f_t}(\xi)| &\leq \gamma \sum_{\xi'} |V^{g_t}(\xi') - V^{f_t}(\xi')| g_t(\xi') p(x'|\xi) + \gamma \eta_t \sum_{\xi'} |V^{f_t}(\xi')| \left(\frac{1}{\#\bm{\A}} + f_t(
\xi') \right) p(x'|\xi) \\
&\leq \gamma \max_{\xi'} |V^{g_t}(\xi')-V^{f_t}(\xi')| + \frac{2\gamma}{1-\gamma} \eta_t.
\end{align*}
Thus, by taking maximum and re-arranging, we have
\begin{equation}
\max_{\xi'} |V^{g_t}(\xi')-V^{f_t}(\xi')| \leq \frac{2\gamma}{(1-\gamma)^2},
\end{equation}
and that
\begin{equation}
\|V^{g_t} - V^{f_t} \| \leq \frac{2\gamma\sqrt{M}}{(1-\gamma)^2} \eta_t =\frac{2\gamma \sqrt{M}}{(1-\gamma)^2} \frac{1}{(1+t)^\varepsilon}.
\end{equation}
Therefore, by triangle inequality, we have
\begin{equation}
\|\phi_t\| = O\left(\frac{1}{(1+t)^{(\beta-n_0\varepsilon) \wedge \varepsilon}}\right).
\end{equation}
\end{proof}

\subsection{Actor Convergence (Proof of Theorem \ref{thm:global_actor_convergence})}
Now we show that the actor converges to a stationary point. We introduce the following notation:
\begin{align}
\widehat{\nabla}_P J(f_t) &:= \sum_{x'', a''} \sigma_{\rho_0}^{f_t}(x'' , a'') Q_t(x'', a'') \nabla_P \log f_t(x'', a''), \\
\widehat{\partial}_{x,a} J(f_t)&:= \sum_{x'', a''} \sigma_{\rho_0}^{f_t}(x'', a'') Q_t(x'', a'') \partial_{P(x,a)} \log f_t(x'', a''),
\end{align}
with $\hat{\partial}_{x,a} J(f_t)$ being the $(x,a)$-entries of the vector $\widehat{\nabla}_P J(f_t)$. Since $f_t = \text{softmax}(P_t)$, hence
\begin{align*}
\widehat{\partial}_{x,a} J(f_t) 
&= \sum_{x'', a''} \sigma_{\rho_0}^{f_t}(x'',a'') Q_t(x'',a'') \mathsf{1}_{\{x'' = x\}} \left[ \mathsf{1}_{\{a'' = a\}} - f_t(x'', a) \right]   \\
&= \sigma_{\rho_0}^{f_t}(x,a)\left[ Q_t(x,a) - \sum_{a''} Q_t(x,a'') f_t(x,a'') \right].
\end{align*}
Using the same method as in \cite{wang2021global} and the following lemmas, we can prove $\|\nabla_P J(f_t)\| \to 0$ as $t \to \infty$. Again, here we only include the case when $\eta_t = (1+t)^{-\varepsilon}$. The case for $\eta_t = (1+(\log(1+t))^2)^{-1}$ is similar, and is deferred to subsection \ref{SS:case_when_eta_is_log}.

\begin{proof}[Proof of Theorem \ref{thm:global_actor_convergence_rate}]
We first prove that whenever $\beta \in (\frac{n_0+1}{n_0+3}, 1)$, then there is $\varepsilon > 0$ such that $\beta + ((\beta - n_0 \varepsilon) \wedge \varepsilon) > 1$. We note that
\begin{align*}
\beta + 2((\beta - n_0 \varepsilon) \wedge \varepsilon) &> 1 \\
\iff 2((\beta - n_0 \varepsilon) \wedge \varepsilon) &> 1-\beta \\
\iff 2\beta - 2n_0 \varepsilon &> 1-\beta\text{ and } \varepsilon > \frac{1-\beta}{2} \\
\iff \varepsilon &\in \left(\frac{1-\beta}{2}, \frac{3\beta-1}{2n_0} \right).
\end{align*}
The interval in the RHS is non-empty iff
$$\frac{1-\beta}2 < \frac{3\beta-1}{2n_0} \iff \beta > \frac{n_0+1}{n_0+3}.$$

Let us now establish the convergence of $\inf_{t\in [0,T]\|}  \|\nabla_P J(f_t)$. Since $\|Q_t - V^{f_t} \| \overset{t \to \infty}\to 0$, there exists a $T_0$ such that whenever $t\geq T_0$, $\clip(Q_t) = Q_t$ (and $|Q_t(x,a)| \leq 2/(1-\gamma)$), which implies for any $\xi = (x,a)$:
\begin{align*}
[A \widehat{\nabla}_P J(f_t)]_\xi &= \sum_{x',a'} A_{\xi,(x',a')} \left[Q_t(x',a') - \sum_{a''} Q_t(x',a'') f_t(x',a'') \right] \sigma^{f_t}_{\rho_0}(x',a') \\
&= \sum_{x',a'} A_{\xi,(x',a')} Q_t(x',a') \sigma^{f_t}_{\rho_0}(x',a') - \sum_{x',a',a''} A_{(x,a),(x',a')} Q_t(x',a'') f_t(x',a'') \sigma^{f_t}_{\rho_0}(x',a') \\
&\overset{(*)}= \sum_{x',a'} A_{\xi,(x',a')} Q_t(x',a') \sigma^{f_t}_{\rho_0}(x',a') - \sum_{x',a',a''} A_{\xi,(x',a')} Q_t(x',a'') f_t(x',a') \sigma^{f_t}_{\rho_0}(x',a'') \\
&\overset{(**)}= \sum_{x',a'} \mathsf{clip}(Q_t(x',a')) \sigma^{f_t}_{\rho_0}(x',a') \left[A_{\xi, (x',a')} - \sum_{a''} A_{\xi,(x',a'')} f_t(x',a'') \right] \\
&= \sum_{x',a'} \mathsf{clip}(Q_t(x',a')) \sigma^{g_t}_{\rho_0}(x',a') \left[A_{\xi, (x',a')} - \sum_{a''} A_{\xi,(x',a'')} f_t(x',a'') \right] + \\
&\phantom{=}+ \sum_{x',a'} \mathsf{clip}(Q_t(x',a')) [\sigma^{g_t}_{\rho_0}(x',a') - \sigma^{f_t}_{\rho_0}(x',a')] \left[A_{\xi, (x',a')} - \sum_{a''} A_{\xi,(x',a'')} f_t(x',a'') \right] \\
&= \frac{1}{\zeta_t} \frac{dP_t}{dt}(\xi) - E_t(\xi)
\end{align*}
with
\begin{itemize}
\item $E_t(\xi)$ defined by the following:
\begin{align*}
E_t(\xi) &= \sum_{x',a'} \clip(Q_t(x',a')) [\sigma^{g_t}_{\rho_0}(x',a') - \sigma^{f_t}_{\rho_0}(x',a')] \left[A_{\xi, (x',a')} - \sum_{a''} A_{\xi,(x',a'')} f_t(x',a'') \right],
\end{align*}
\item $(*)$ holds as $\sigma^{f_t}_{\rho_0}(x',a') f_t(x',a'') = \nu^{f_t}_{\rho_0}(x') f_t(x',a') f_t(x',a'') = \sigma^{f_t}_{\rho_0}(x',a'') f_t(x',a')$, and
\item $(**)$ holds by swapping the indices $a'$ and $a''$ in the second term and replacing $Q_t$ with $\mathsf{clip}(Q_t)$.
\end{itemize}
Rearranging yields the following:
\begin{equation}
\frac{dP_t}{dt} = \zeta_t A \widehat{\nabla}_P J(f_t) - \zeta_t E_t,
\end{equation}
where $E_t$ is a vector having $\xi$-entries $E_t(\xi)$. Therefore we have
\begin{align*}
\frac{dJ(f_t)}{dt} &= \zeta_t(\nabla_P J(f_t))^\top A(\nabla_P J(f_t)) + \zeta_t(\nabla_P J(f_t))^\top A (\widehat{\nabla}_P J(f_t) - \nabla_P J(f_t)) - \zeta_t \sum_{\xi} \partial_\xi J(f_t) E_t(\xi).
\end{align*}
As $A$ is a positive definite with finite entries, $0 < \sigma_{\min}(A) \leq \sigma_{\max}(A) < \infty$, so
\begin{equation}
\frac{d}{dt} J(f_t) \geq \zeta_t \sigma^2_{\min}(A) \|\nabla_P J(f_t)\|^2 - \zeta_t \|\nabla_P J(f_t)\| \Big[\sigma^2_{\max}(A) \|\widehat{\nabla}_P J(f_t) - \nabla_P J(f_t)\| + \|E_t\| \Big]. \label{eq:pre-lower-bound}
\end{equation}
Recall that 
\begin{equation}
\partial_{x.a} J(f_t) - \widehat{\partial}_{x,a} J(f_t) = \sigma^{f_t}_{\rho_0}(x,a) \left[[Q_t - V^{f_t}](x,a) - \sum_{a''} [Q_t - V^{f_t}](x,a'')\right],
\end{equation}
so there are constants $C > 0$ (changing line by line) such that
$$\|\nabla_P J(f_t) - \widehat{\nabla}_P J(f_t) \| \leq C\|V^{f_t} - Q_t\| \leq \frac{C}{(1+t)^{(\beta-n_0\varepsilon) \wedge \varepsilon}}.$$
Furthermore, using statement (2) of Proposition \ref{prop:Lipschitzness_of_stationary_measures}, there are constants $C > 0$ such that
\begin{align*}
|E_t(\xi)| \leq C \eta_t \left[\sum_\xi \left|f(\xi) - \frac{1}{\#\mathcal{A}}\right| \right] \leq C \eta_t,
\end{align*}
so $\displaystyle{\|E_t\| \leq C\eta_t \leq \frac{C}{(1+t)^{(\beta-n_0\varepsilon) \wedge \varepsilon}}}$. \\

Substituting \eqref{eq:pre-lower-bound-new}, we know that there is a constant $C>0$ such that
\begin{align*}
\frac{d}{dt} J(f_t) &\geq \frac{1}{(1+t)^{\beta}} \sigma^2_{\min}(A) \|\nabla_P J(f_t)\|^2 - \frac{C}{(1+t)^{\beta+(\beta-n_0\varepsilon) \wedge \varepsilon}} \|\nabla_P J(f_t)\|, \quad t \geq T_0.
\end{align*}
Rearranging and utilising Young's inequality yields, for all $T \geq T_0$,
\begin{align*}
\frac{\sigma^2_{\min}(A)}{(1+t)^{\beta}}  \|\nabla_P J(f_t)\|^2 &\leq \frac{dJ(f_t)}{dt} + \frac{C}{(1+t)^{\beta+(\beta-n_0\varepsilon) \wedge \varepsilon}} \|\nabla_P J(f_t)\| \\
&\leq \frac{dJ(f_t)}{dt} + \frac{\sigma^2_{\min}(A)}{2(1+t)^{\beta}} \|\nabla_P J(f_t)\|^2 + \frac{C}{2 \sigma^2_{\min}(A)} \frac{1}{(1+t)^{\beta+2((\beta-n_\varepsilon) \wedge \varepsilon)}}. \\
\implies \frac{\|\nabla_P J(f_t)\|^2}{(1+t)^{\beta}} &\leq \frac{2}{\sigma^2_{\min} (A)} \frac{dJ(f_t)}{dt} + \frac{C}{2 \sigma^4_{\min}(A)} \frac{1}{(1+t)^{\beta+2((\beta-n_\varepsilon) \wedge \varepsilon)}}.
\end{align*}
Recall from Lemma \ref{softmax gradient} that
\begin{equation*}
\partial_{x,a} J(f_t) = \sigma^{f_t}_{\rho_0}(x,a) A^{f_t}(x,a).
\end{equation*}
As $|A^{f_t}(x,a)| \leq 2/(1-\gamma)$, the following bound holds:
\begin{equation}
\|\nabla_P J(f_t)\|^2 = \sum_{x,a} (\sigma^{f_t}_{\rho_0}(x,a))^2 (A^{f_t}(x,a))^2 \leq \frac{2}{1-\gamma}.
\end{equation}
Therefore,
\begin{align}
\int_0^T \frac{\|\nabla_P J(f_t)\|^2}{(1+t)^{\beta}} \, dt 
&= \int_0^{T_0} \frac{\|\nabla_P J(f_t)\|^2}{(1+t)^{\beta}} \, dt + \int_{T_0}^T \frac{\|\nabla_P J(f_t)\|^2}{(1+t)^{\beta}} \, dt \\
&\leq \frac{2T_0}{1-\gamma} + \frac{J(f_T) - J(f_{T_0})}{2\sigma^2_{\min}(A)} + \frac{C}{2\sigma^4_{\min}(A)} \int_0^T \frac{1}{(1+t)^{\beta+2((\beta-n_0 \varepsilon) \wedge \varepsilon)}} \, dt \nonumber \\
&\leq \frac{2T_0 \sigma^2_{\min}(A) + 1}{(1-\gamma)\sigma^2_{\min}(A)} + \frac{C}{2\sigma^4_{\min}(A)} \int_0^T \frac{1}{(1+t)^{\beta+2((\beta-n_0 \varepsilon) \wedge \varepsilon)}} \, dt, \label{eq:weighted_time_average_of_squared_norm}
\end{align}
noting that $|J(f)| = |\sum_x \rho_0(x) \bar{V}^f(x)| \leq 1/(1-\gamma)$ by \eqref{eq:trivial_bound}. \\

Since we assumed that $\beta+ 2((\beta-n_0\varepsilon) \wedge\varepsilon) > 1$, so is $\beta+2((\beta-n_0\varepsilon) \wedge\varepsilon)$, and the integral in the RHS converges. \\

Therefore, the weighted time-average of the squared gradient converges to zero:
\begin{align}
\inf_{t \in [0,T]} \|\nabla_P J(f_t) \|^2 \leq \frac{\int_0^T\zeta_t \|\nabla_P J(f_t)\|^2 \, dt}{\int_0^T \zeta_t \, dt} &\leq \frac{C'}{(1+T)^{1-\beta_a} - 1} \overset{T\to+\infty}\to 0.
\end{align}
\end{proof}

\begin{proof}[Proof of Theorem \ref{thm:global_actor_convergence}]
Using the exact same arguments in Theorem \ref{thm:global_actor_convergence_rate}, we can show that whenever $\beta \in (\frac{n_0+1}{n_0+2},1)$, then there is $\varepsilon > 0$ such that $\beta+((\beta-n_0\varepsilon) \wedge \varepsilon) > 1$. To show that $\displaystyle \lim_{t\to \infty}\|\nabla_P J(f_t)\| = 0$, assume the contrary; that is $\displaystyle \limsup_{t\to \infty}\|\nabla_P J(f_t)\| > 0$. Then we can find a constant $\epsilon_1>0$ and two increasing sequences $\{a_n\}_{n\ge 1}, \{b_n\}_{n\ge 1}$ such that 
\begin{align*}
a_1 <b_1 <a_2 <b_2 <a_3 <b_3 < \cdots,\\
\|\nabla_P J(f_{a_n})\| < \frac{\epsilon_1}{2},\quad \|\nabla_P J(f_{b_n})\| > \epsilon_1.
\end{align*}
Define the following cycle of stopping times:
\begin{align}
t_n &:= \sup\{s | s\in (a_n, b_n),\  \|\nabla_P J(f_s)\| < \frac{\epsilon_1}{2} \},\\
i(t_n) &:= \inf\{s | s\in (t_n, b_n),\  \|\nabla_\theta J(f_s)\| > \epsilon_1 \}.
\end{align}
Note that $ \|\nabla_P J(f_t)\| $ is continuous with respect to $t$, so
we have 
\begin{align*}
&a_n \le t_n < i(t_n) \le b_n \\
&\|\nabla_P J(f_{t_n})\| = \frac{\epsilon_1}{2}, \quad \|\nabla_P J(f_{i(t_n)})\| =\epsilon_1\\
&\frac{\epsilon_1}{2} \le \|\nabla_P J(f_s)\| \le \epsilon_1, \quad s\in(t_n, i(t_n)). \numberthis \label{property}
\end{align*}
Notice that there is a constant $L > 0$ such that gradient is $L$-Lipschitz, see e.g., the proof of \cite[Lemma 7]{mei2020global}, we have for any $t_n$
\begin{align*}
\frac{\epsilon_1}{2} &=  \|\nabla_P J(f_{i(t_n)})\| - \|\nabla_P J(f_{t_n})\|\\
&\le \| \nabla_P J(f_{i(t_n)}) - \nabla_P J(f_{t_n}) \|\\
&\le L\| P_{i(t_n)} - P_{t_n} \|\\
&\le C \int_{t_n}^{i(t_n)} \zeta_s\| \nabla_P J(f_s)\| ds + C \int_{t_n}^{i(t_n)} \zeta_s \|\widehat{\nabla}_P J(f_s) - \nabla_P J(f_s)\|ds \\
&\le C \epsilon_1 \int_{t_n}^{i(t_n)} \frac{1}{(1+t)^\beta} ds + C \int_{t_n}^{i(t_n)} \frac{1}{(1+t)^{\beta + ((\beta - n_0\varepsilon) \wedge \varepsilon)}} ds.
\end{align*}
Since $\int_0^\infty (1+t)^{-(\beta + ((\beta - n_0\varepsilon) \wedge \varepsilon))} \, dt < +\infty$, it follows that 
\begin{equation}
\label{key2}
\frac{1}{2L} \le \liminf_{n\to \infty} \int_{t_n}^{i(t_n)} \zeta_s ds.
\end{equation}
Using \eqref{property}, we see that 
\begin{equation*}
J(f_{\bar\theta_{i(t_n)}}) - J(f_{\bar\theta_{t_n}}) \ge C_1 \left(\frac{\epsilon_1}{2} \right)^2 \int_{t_n}^{i(t_n)} \zeta_s ds - C_2\int_{t_n}^{i(t_n)} \zeta_s\eta_s ds.
\end{equation*}
Due to the convergence of $J(f_{\theta_{t_n}})$ and the assumption of the learning rate, this implies that 
\begin{equation}
\lim_{n\to \infty} \int_{t_n}^{i(t_n)} \zeta_s ds = 0,
\end{equation}
which contradicts Equation \eqref{key2}, and thus the convergence to the stationary point is proven. 
\end{proof}

\section{Conclusion}
We have developed a convergence analysis for online neural network actor-critic algorithms in the limit as the number of hidden units and training steps $\rightarrow \infty$. In the algorithm, the neural network is initialised using the Neural Tangent Kernel (NTK) initialisation and an epsilon-greedy softmax policy is used for exploration. It is proven that the time-rescaled algorithm evolutions converges weakly to a limit ODE system using weak convergence techniques combined with an appropriate Poisson equation. The limit ODE involves an NTK kernel, and we study the dynamics as the training time $t \rightarrow \infty$. 

In our analysis of the limit ODE, we prove that the critic network converges to the action-value function for the current policy and the actor network converges to a stationary point of the objective function (the expected discounted reward). A crucial element necessary for the proof is the careful selection of the exploration rate in the epsilon-greedy softmax policy, highlighting the role of exploration in learning an optimal policy. Convergence rates are also proven for both the actor and critic ODEs.

There are many interesting future research directions which could build upon the results in this paper. We have focused our analysis on the very specific case where $\sigma$ is the sigmoid function $\sigma(z)=(1+e^{-z})^{-1}$, and did not generalise too far from the original Cybenko's condition \cite{cybenko1989approximation}, which ensures the universal approximation of functions by neural networks. It has been shown in \cite{sirignano2021asymptotics, sirignano2022scalinglimitneuralnetworks}) that this is sufficient to ensure that the NTK is positive definite. It has since been established in more recent research (e.g.,  \cite{Carvalho2024NTKpositive}) that the NTK is positive definite even when a much more mild condition is imposed on $\sigma$. This, in principle, paves the way for us to extend our analysis for more general activation functions. However, several important technical challenges remain. For unbounded ReLUs, there will be additional challenges that must be addressed in order to establish bounds for the parameter updates. These challenges due to the unbounded nature of ReLUs could potentially be addressed by clipping the critic in the critic updates. However, it is unclear how this would affect the convergence of the limit critic to a valid action value function as the training time $t \rightarrow \infty$. Modifying the Taylor expansion arguments when deriving the ODE, such that it is also valid when the neural network is non-differentiable at certain points (which is the case for ReLU activation functions). 

It would also be interesting to prove an extension of the current results to a general Polish state space. A new condition is needed for the MDP to ensure that the induced/auxiliary Markov chain by a reasonably large class of policies is (geometric) ergodic. It is also not trivial to ensure that the stationary measures $\pi^f$ and $\sigma^f_{\rho_0}$ remain Lipschitz with respect to the policies $f$. The space that the action-value functions are studied in is important. For example, we could consider the reproducing kernel Hilbert space (RKHS) of the associated Neural Tangent Kernel (NTK). \cite{cayci2022finitetime, wang2019neural} provides more details on how this could be done. The main difficulty would be that it is now more difficult to specify the path space $D_E([0,T])$, where $E$ is the Polish space in which the action value functions live. The norm of the Polish space must be carefully chosen for the convergence to be meaningful. A key challenge is that a completely different proof would have to be used for the limit ODEs as $t \rightarrow \infty$ since the quantity $Y_t := \phi_t^\top A^{-1} \phi_t$ in the current critic convergence proof (subsection 5.1) cannot be used in infinite dimensions since the infinite-dimensional NTK kernel cannot be inverted. 

Finally, it would be interesting to consider how different normalization scalings affect the limit differential equations. As mentioned in the Introduction, the normalisation scaling for the neural network approximation determines if the limit differential equation is an ODE or a McKean-Vlasov equation (mean-field limit), and the step size determines if a Brownian motion is present. It would also be interesting to explore the dynamics of Stochastic Gradient Langevin Dynamics for actor-critic algorithms.

\section*{Acknowledgement}
This research has been supported by the EPSRC Centre for Doctoral Training in Mathematics of Random Systems: Analysis, Modelling and Simulation (EP/S023925/1).

\appendix
\section{Pre-limit Evolutions of the Actor and Critic Networks}
\label{S:pre_limit_evolution}

Recall that the Online Actor-Critic algorithm (Algorithm \ref{alg:onlineNAC}) could therefore be written as followed, given the choice of the learning rates $\lambda = 1$:
\begin{align*}
C^{i}_{k+1} &= C^{i}_k + \frac{1}{N\sqrt{N}} \left(r(\xi_k) + \gamma \sum_{a''} Q^N_k(x_{k+1}, a'') g^N_k(x_{k+1},a'') - Q^N_k(\xi_k) \right) \sigma \left(W^i_k \cdot \xi_k \right), \\
W^i_{k+1} &= W^i_k + \frac{1}{N\sqrt{N}}  \left( r(\xi_k) + \gamma \sum_{a''} Q^N_k(x_{k+1}, a'') g^N_k(x_{k+1},a'') - Q^N_k(\xi_k) \right) C^{i}_k \sigma' \left( W^i_k \cdot \xi_k \right) \xi_k, \\
B^i_{k+1} &= B^i_k + \frac{\zeta^N_k}{N\sqrt{N}} \clip(Q^N_k(\tilde{\xi}_k)) \bracket{\sigma(U^i \cdot (\tilde{\xi}_k)) - \sum_{a''} f^N_k(\tilde{x}_k,a'') \sigma(U^i \cdot (\tilde{x}_k, a''))}, \\
U^i_{k+1} &= U^i_k + \frac{\zeta^N_k}{N\sqrt{N}} \clip(Q^N_k(\tilde{\xi}_k)) \bracket{B^i_k \sigma'(U^i_k \cdot \tilde{\xi}_k) \tilde{\xi}_k - \sum_{a''} f^N_k(\tilde{x}_k,a'') B^i_k \sigma'(U^i_k \cdot \tilde{\xi}_k) (\tilde{x}_k, a'')}. \numberthis \label{eq:updates_repeated}
\end{align*}

In computing the evolution of the Actor and Critic networks, we will
\begin{itemize}
    \item prove that the increments of the parameters are bounded,
    \item prove a-priori $L^2$ bounds for the outputs of the actor and critic networks, and 
    \item specify the size of the error terms in the pre-limit evolution, and
    \item rewrite the pre-limit evolution in terms of fluctuation terms
\end{itemize}

\subsection{Bounds for the increments of the parameters}

\begin{lemma}[A-priori bounds of size of increments of parameters]
\label{NN parameter bound}
For any fixed $T>0$, any $k$ such that $k \leq TN$ and $i \in [N] = \{1,...,N\}$, there exists a constant $C_T<\infty$ that only depends on $T$ such that 
\begin{equation}
\max\bracket{\left|C^i_k \right|, \e\|W_k^i\|, |B^i_k|, \e\| U_k^i \|} < C_T, 
\end{equation}
and that
\begin{equation}
\max\bracket{\left|C^i_{k+1} - C^i_k \right|, \left\|W^i_{k+1} - W^i_k \right\|} \leq \frac{C_T}{N}.
\end{equation}
Moreover,
\begin{equation}
\max\bracket{\left|B^i_{k+1} - B^i_k \right|, \left\|U^i_{k+1} - U^i_k \right\|} < \frac{C_T}{N^{3/2}}
\end{equation}
\end{lemma}

\begin{proof}
We start with the observations that
\begin{equation}
    \max_{\xi \in \bm{\X}\times \bm{\A}} |Q^N_k(\xi)| = \max_{\xi \in \bm{\X}\times \bm{\A}} \abs{\frac{1}{\sqrt{N}} \sum_{i=1}^N C^i_k \sigma(W^i_k \cdot \xi)} \leq \frac{1}{\sqrt{N}} \sum_{i=1}^N |C^i_k|
\end{equation}
(as $\sigma$ is bounded by 1 by Assumption \ref{as:activation_function}), and
\begin{align*}
\max_{x \in \bm{\X}} \abs{\sum_{a''} Q^N_k(x,a'') g^N_k(x,a'')} &\leq \max_{x \in \bm{\X}} \sum_{a''} \abs{Q^N_k(x,a'')} g^N_k(x,a'') \\
&\leq \max_{\xi \in \bm{\X} \times \bm{\A}} \abs{Q^N_k(\xi)} \sum_{a''} g^N_k(x,a'') = \max_{\xi \in \bm{\X} \times \bm{\A}} \abs{Q^N_k(\xi)}. \numberthis
\end{align*}
We may then recursive bound for $|C^i_k|$ using the above observations:
\begin{align*}
\left|C^i_{k+1} - C^i_k\right| &\leq \frac{\alpha^N}{\sqrt{N}} \abs{r(\xi_k) + \gamma \sum_{a''} Q^N_k(x_{k+1},a'') g^N_k(x_{k+1},a'') - Q^N_k(\xi_k)} \cdot \abs{\sigma(W^i_k \cdot \xi_k)} \\
&\leq \frac{\alpha}{N^{3/2}} \left( \abs{r(\xi_k)} + (1 + \gamma) \max_{\xi \in \bm{\X} \times \bm{\A} } \abs{Q^N_k(\xi)}\right) \cdot \abs{\sigma (W^i_k \cdot \xi_k)} \\
&{\le} \frac{\alpha}{N^{3/2}} \bracket{1 + (\gamma+1) \frac{1}{\sqrt{N}} \sum_{i=1}^N |C^i_k|} \\
&= \frac{\alpha}{N^{3/2}} + \frac{\alpha}{N^2} \sum_{i=1}^N |C^i_k|. \numberthis
\end{align*}
By recursively using the triangle inequality, and recalling that $C^i_0$ is a bounded random variable, we have
\begin{align*}
|C^i_k| \leq |C^i_0| + \sum_{j=1}^k (|C^i_j - C^i_{j-1}|) 
&\leq 1 + \sum_{j=1}^k \bracket{\frac{\alpha}{N^{3/2}} + \frac{\alpha}{N^2} \sum_{i=1}^N |C^i_{j-1}|} \\
&= 1 + \frac{\alpha T}{N^{1/2}} + \frac{\alpha}{N^2} \sum_{j=1}^k \sum_{i=1}^N |C^i_{j-1}|. \numberthis \label{eq:recursion_bound_Cik}
\end{align*}
Define 
\begin{equation}
\displaystyle{m_k^N = \frac{1}{N} \sum_{i=1}^N |C^i_k|}.
\end{equation}
Then
\begin{equation}
m_k^N \leq \frac{1}{N} \sum_{i=1}^N \bracket{1 + \frac{\alpha T}{N^{1/2}} + \frac{\alpha}{N^2} \sum_{j=1}^k \sum_{l=1}^N |C^l_{j-1}|} \le (1+\alpha T) + \frac{\alpha}{N} \sum_{j=1}^k m_{j-1}^N.
\end{equation}
By the discrete Gronwall's lemma and using $k \le TN$,
\begin{equation}
m_k^N \leq (1+\alpha T)\exp\left(\frac{\alpha k}{N} \right) \leq (1+\alpha T) \exp(\alpha T) =: C_T.
\end{equation}
Plugging this into \eqref{eq:recursion_bound_Cik} yields
\begin{equation}
\label{C bound}
\left|C^i_k\right| \le \left| C_0^i \right| + \frac{C}{N^{1/2}} + \frac{C}{N}\sum_{j=1}^k m_{j-1}^N
\le \left| C_0^i \right| + \frac{C}{N^{1/2}} + C_T \le C_T,
\end{equation}
We could bootstrap with this a-priori bound to show that 
\begin{equation}
    |C^i_{k+1} - C^i_k| \leq \frac{C}{N^{3/2}} + N \times \frac{C}{N^2} \times C_T \leq \frac{C_T}{N}.
\end{equation}
We can similarly get the bound for $\norm{W_k^i}$. In fact, 
\begin{align*}
\norm{W^i_{k+1} - W^i_k} &\leq \frac{\alpha^N}{\sqrt{N}} \abs{r(\xi_k) + \gamma \sum_{a''} Q^N_k(x_{k+1},a'') g^N_k(x_{k+1},a'') - Q^N_k(\xi_k)} \cdot \abs{C^{i}_k \sigma'\left( W^i_k \cdot \xi_k \right)} \|\xi_k\| \\
&\leq \frac{C_T}{N^{\frac32}} \bracket{C + (\gamma+1) N^{-\frac12} \sum_{i=1}^N |C^i_k|} \overset{\eqref{C bound}}{\le} \frac{C_T}{N}, \numberthis \label{eq:diff_of_W}
\end{align*}
Taking expectation and using Assumptions \ref{as:activation_function} and \ref{as:NN_condition} yields
\begin{equation}
\E\norm{W^i_{k}} \leq \E\norm{W^i_{0}} + \sum_{j=0}^{k-1} \e \norm{W^i_{k+1} - W^i_k} \leq C_T.
\end{equation}

For the boundedness of parameters in the actor network, observe that
\begin{align*}
|B^i_{k+1} - B^i_k| &\leq \zeta^N_k N^{-\frac32} |\clip(Q^N_k(\tilde{\xi}_k))| \cdot \sup_{a''} \abs{\sigma(\tilde{x}_k, a'')} \cdot \bracket{1 + \sum_{a''} f^N_k(\tilde x_k, a'')} < \frac{C}{N^{3/2}} \numberthis \label{eq:diff_of_B}
\end{align*}
then by telescoping series, we have for all $k \leq NT$
\begin{equation}
|B^i_k| \leq |B^i_0| + C\frac{k}{N^{\frac32}} \leq C + C \frac{T}{N^{\frac12}} \le C_T.
\end{equation}
As the state-action space is finite, we also have 
\begin{equation}
\norm{U^i_{k+1} - U^i_k} \leq \zeta^N_k N^{-\frac32} \left|\clip(Q^N_k(\tilde{\xi}_k))\right| |B^i_k| \left(1+\sum_{a''} f^N_k (\tilde{x}_k, a'')\right) \cdot \sup_{\xi \in \X \times \A} \norm{\xi} \le \frac{C_T}{N^{3/2}},
\end{equation}
which yields
\begin{equation}
\e\norm{U^i_k} \leq \e\norm{U^i_0} + C_T \frac{k}{N^{\frac 32}} \leq C + \frac{C_T}{N^{\frac12}} \leq  C_T, \quad \forall k \le TN.
\end{equation}
\end{proof}

\begin{lemma}[Increments of entries in the pre-limit kernels] \label{lem:diff_of_B_bar_B}
For all $k \leq NT$,
\begin{equation}
\max_{\xi, \xi' \in \X \times \A} \Big[ |[A(\nu^N_{k+1})]_{\xi,\xi'} - [A(\nu^N_k)]_{\xi,\xi'}| \vee |[A(\mu^N_{k+1})]_{\xi,\xi'} - [A(\mu^N_k)]_{\xi,\xi'}| \Big] \leq \frac{C_T}{N},
\end{equation}
where the kernels $A(\cdot)$ are defined in \eqref{eq:kernel_matrix}. Consequently, one could show by method of telescoping series that for all $k \leq NT$
\begin{equation}
\label{pre kernel bound}
\max_{\xi, \xi' \in \X \times \A} \Big[\big|[A(\nu^N_k)]_{\xi, \xi'} \big| \vee \big|[A(\mu^N_k)]_{\xi, \xi'} \big| \Big] \leq C_T,
\end{equation}
\end{lemma}

\begin{proof}
The proof for the case of the kernel $A(\mu^N_k)$ is exactly the same as the proof for the case of $A(\nu^N_k)$, both of which utilise our a priori bounds of the increments of the parameters. To the end, for all $\xi,\xi' \in \X \times \A$, we have
\begin{align*}
    &\phantom{=}\abs{\la \sigma(w \cdot \xi') \sigma(w \cdot \xi), \, \nu^N_{k+1} - \nu^N_k \ra} \\
    &\leq \frac{1}{N} \sum_{i=1}^N \abs{\sigma(W^i_{k+1} \cdot \xi') \sigma(W^i_{k+1} \cdot \xi) - \sigma(W^i_k \cdot \xi') \sigma(W^i_k \cdot \xi)} \\
    &\leq \frac{1}{N} \sum_{i=1}^N \sqbracket{\abs{\sigma(W^i_{k+1} \cdot \xi') - \sigma(W^i_k \cdot \xi')} \abs{\sigma(W^i_{k+1} \cdot \xi)} + \abs{\sigma(W^i_{k+1} \cdot \xi) - \sigma(W^i_k \cdot \xi)} \abs{\sigma(W^i_k \cdot \xi')}} \\
    &\leq \frac{1}{N} \sum_{i=1}^N (|\xi'| + |\xi|) \norm{W^i_{k+1} - W^i_k} \leq \frac{C_T}{N}. \numberthis \label{eq:B_increment_term_1}
\end{align*}
Similarly,
\begin{align*}
    &\phantom{=}\abs{\la c^2 \sigma'(w \cdot \xi') \sigma'(w \cdot \xi), \, \nu^N_{k+1} - \nu^N_k \ra} \\
    &\leq \frac{1}{N} \sum_{i=1}^N \abs{(C^i_{k+1})^2 \sigma(W^i_{k+1} \cdot \xi') \sigma(W^i_{k+1} \cdot \xi) - (C^i_k)^2 \sigma(W^i_k \cdot \xi') \sigma(W^i_k \cdot \xi)} \\
    &\leq \frac{1}{N} \sum_{i=1}^N \Big[\abs{(C^i_{k+1})^2 - (C^i_k)^2} \abs{\sigma(W^i_{k+1} \cdot \xi')} \abs{\sigma(W^i_{k+1} \cdot \xi)} \\
    &\phantom{=}+ (C^i_k)^2\abs{\sigma(W^i_{k+1} \cdot \xi') - \sigma(W^i_k \cdot \xi')} \abs{\sigma(W^i_{k+1} \cdot \xi)} + (C^i_k)^2 \abs{\sigma(W^i_{k+1} \cdot \xi) - \sigma(W^i_k \cdot \xi)} \abs{\sigma(W^i_k \cdot \xi')} \Big] \numberthis
\end{align*}
We have the control 
\begin{equation}
    \abs{(C^i_{k+1})^2 - (C^i_k)^2} \leq \abs{C^i_{k+1} - C^i_k}^2 + 2|C^i_k||C^i_{k+1} - C^i_k| \leq \frac{C^2_T}{N^2} + \frac{2C^2_T}{N} \leq \frac{C_T}{N}.
\end{equation}
By combining this with our previous analyses, we have
\begin{equation}
    \abs{\la c^2 \sigma'(w \cdot \xi') \sigma'(w \cdot \xi), \, \nu^N_{k+1} - \nu^N_k \ra}
    \leq \frac{1}{N} \sum_{i=1}^N \sqbracket{\frac{C_T}{N} + C_T \times \frac{C_T}{N} + C_T \times \frac{C_T}{N}} = \frac{C_T}{N}. \label{eq:B_increment_term_2}
\end{equation}
Summing up \eqref{eq:B_increment_term_1} and \eqref{eq:B_increment_term_2} yields $|[A(\nu^N_{k+1})]_{\xi,\xi'} - [A(\nu^N_k)]_{\xi,\xi'}| \leq C_T/N$, uniformly in $\xi,\xi'$. It remains to show that there is a $C>0$, independent of $T$, such that $|[A(\nu^N_0)]_{\xi,\xi'}| \leq C$. This is clearly true by the sure boundedness of $\sigma(\cdot), \sigma'(\cdot)$ and $C^i_0$ as guaranteed in the Assumptions \ref{as:activation_function} and \ref{as:NN_condition}. Therefore, we could consider the telescoping sum
\begin{equation}
\big|[A(\nu^N_k)]_{\xi,\xi'} \big| \leq \big| [A(\nu^N_0)]_{\xi,\xi'} \big| + \sum_{j=0}^{k-1} \big| [A(\nu^N_{j+1})]_{\xi,\xi'} - [A(\nu^N_j)]_{\xi,\xi'} \big| \leq C + N \times \frac{C_T}{N} \leq C_T,
\end{equation}
which completes our proof.
\end{proof}

\subsection{\texorpdfstring{$L^2$}{} bounds of network outputs}

Using Lemma \ref{NN parameter bound} and Lemma \ref{lem:diff_of_B_bar_B}, we can now establish the bounds for the neural networks. 

\begin{lemma}[A-priori $L^2$ bound for the outputs of the critic network] \label{lem:maximal_Q_bounded}
For all $k$ such that $k \leq TN$, there is a $C_T<\infty$ such that 
\begin{equation}
\E\sqbracket{\max_{(x,a) \in \X \times \A} |Q^N_k(x,a)|^2} < C_T.
\end{equation}
\end{lemma}

\begin{proof}
We first prove the statement for $k = 0$. Since $C^i_0$ and $\sigma(W^i_0 \cdot \xi)$ are both bounded by 1, we have
\begin{align*}
\E\sqbracket{\max_{\xi \in \X \times \A} |Q^N_0(\xi)|^2} \leq \E\sqbracket{\sum_{\xi \in \X \times \A} |Q^N_0(\xi)|^2}
&\leq \sum_{\xi \in \X \in \A} \E\sqbracket{\frac{1}{\sqrt{N}}\sum_{i=1}^N C^i_0 \sigma(W^i_0 \cdot \xi)}^2 \\
&\overset{(a)}{\leq} \frac{C}{N} \sum_{i=1}^N \E\sqbracket{C^i_0 \sigma(W^i_0 \cdot \xi)}^2 \leq C < \infty, \numberthis \label{eq:expectation_max_Q0}
\end{align*}
We now provide an $L^2$ control over the maximum increments of the outputs $Q^N_k(\xi)$. Recall that
\begin{equation}
\label{Q update bound}
Q^N_{k+1}(\xi) - Q^N_k(\xi) = \frac{1}{\sqrt{N}} \sum_{i=1}^N \sqbracket{(C^i_{k+1} - C^i_k) \sigma(W^i_{k+1} \cdot \xi) + C^i_k (\sigma(W^i_{k+1} \cdot \xi) - \sigma(W^i_{k+1} \cdot \xi))},
\end{equation}
so
\begin{align*}
\abs{Q^N_{k+1}(\xi) - Q^N_k(\xi)}^2
&\overset{\mathsf{(CS)}}\leq \frac{2}{N} \sqbracket{\bracket{\sum_{i=1}^N (C^i_{k+1} - C^i_k) \sigma(W^i_{k+1} \cdot \xi)}^2 + \bracket{\sum_{i=1}^N C^i_k (\sigma(W^i_{k+1} \cdot \xi) - \sigma(W^i_k \cdot \xi))}^2} \\
&\overset{\mathsf{(CS)}}\leq \frac{2}{N} \Bigg[\Bigg( \sum_{i=1}^N (C^i_{k+1} - C^i_k)^2 \Bigg) \Bigg(\sum_{i=1}^N \bracket{\sigma(W^i_{k+1} \cdot \xi)}^2 \Bigg) \\
&\phantom{=}+ \Bigg( \sum_{i=1}^N (C^i_k)^2 \Bigg) \Bigg( \sum_{i=1}^N \bracket{\sigma(W^i_{k+1} \cdot \xi) - \sigma(W^i_k \cdot \xi)}^2 \Bigg) \Bigg] \\
&\leq 2 \sqbracket{\sum_{i=1}^N (C^i_{k+1} - C^i_k)^2 + C_T \sum_{i=1}^N \bracket{\sigma(W^i_{k+1} \cdot \xi) - \sigma(W^i_k \cdot \xi)}^2} \numberthis \label{eq:apriori_bound_of_increment_of_Q}
\end{align*}
To proceed, we note that by Jensen inequality (and the fact that $\sum_{a''} g^N_k(x_{k+1},a'') = 1$),
\begin{align*}
\bracket{\sum_{a''} Q^N_k(x_{k+1},a'') g^N_k(x_{k+1},a'')}^2 &\leq \sum_{a''} \bracket{Q^N_k(x_{k+1},a'')}^2 g^N_k(x_{k+1},a'') \\
    &\leq \sum_{a''} \max_{\xi \in \bm{\X}\times\bm{\A}} |Q^N_k(\xi)|^2 g^N_k(x_{k+1},a'') = \max_{\xi \in \bm{\X}\times\bm{\A}} |Q^N_k(\xi)|^2,
    \numberthis
\end{align*}
so
\begin{align*}
\left|C^i_{k+1} - C^i_k \right|^2 &\leq \frac{(\alpha^N)^2}{N} \bracket{r(\xi_k) + \gamma \sum_{a''} Q^N_k(x_{k+1},a'') g^N_k(x_{k+1},a'') - Q^N_k(\xi_k)}^2  \bracket{\sigma(W^i_k \cdot \xi_k)}^2 \\
&\overset{\mathsf{(CS)}}\leq \frac{3\alpha}{N^3}\sqbracket{(r(\xi_k))^2 + \gamma^2 \bracket{\sum_{a''} Q^N_k(x_{k+1},a'') g^N_k(x_{k+1},a'')}^2 + \bracket{Q^N_k(\xi_k)}^2} \\
&\leq \frac{3\alpha}{N^3} \bracket{1 + (1+\gamma^2) \max_{\xi \in \bm{\X}\times\bm{\A}}|Q^N_k(\xi)|^2} \numberthis
\end{align*}
Making use of the mean-value inequality (and the fact that $|\sigma'| \leq 1$ by assumption \ref{as:activation_function}), one could show similarly
\begin{align*}
&\phantom{=} |\sigma(W^i_{k+1} \cdot \xi) - \sigma(W^i_k \cdot \xi)|^2 \\
&\leq |(W^i_{k+1} - W^i_k) \cdot \xi|^2 \\
&\leq \frac{(\alpha^N)^2}{N} \bracket{r(\xi_k) + \gamma \sum_{a''} Q^N_k(x_{k+1},a'') g^N_k(x_{k+1},a'') - Q^N_k(\xi_k)}^2  \bracket{\sigma(W^i_k \cdot \xi_k)}^2 (C^i_k)^2 (\xi_k \cdot \xi)^2 \\
&\leq \frac{3\alpha C_T^2}{N^3} \bracket{1 + (1+\gamma^2) \max_{\xi \in \bm{\X}\times\bm{\A}}|Q^N_k(\xi)|^2}, \numberthis
\end{align*}
noting that $(\xi_k \cdot \xi)^2$ is bounded by some constant $C$ as $\xi, \xi_k$ are elements from the finite set $\bm{\X}\times \bm{\A}$. Substituting into \eqref{eq:apriori_bound_of_increment_of_Q} yields
\begin{equation}
\abs{Q^N_{k+1}(\xi) - Q^N_k(\xi)}^2 \leq \frac{C_T}{N^2} \bracket{1 + (1+\gamma^2) \max_{\xi \in \bm{\X}\times\bm{\A}}|Q^N_k(\xi)|^2}. \label{eq:4.41}
\end{equation}
Therefore for any $\xi$ and $k \leq NT$,
\begin{align*}
    |Q^N_k(\xi)|^2 &= \bracket{Q^N_0(\xi) + \sum_{j=0}^{k-1} (Q^N_{j+1}(\xi) - Q^N_j(\xi))}^2 \\
    &\overset{\mathsf{(CS)}}= 2\bracket{Q^N_0(\xi)}^2 + 2\bracket{\sum_{j=0}^{k-1} (Q^N_{j+1}(\xi) - Q^N_j(\xi))}^2 \\
    &\overset{\mathsf{(CS)}}\leq 2\bracket{Q^N_0(\xi)}^2 + N\sum_{j=0}^{k-1} (Q^N_{j+1}(\xi) - Q^N_j(\xi))^2 \\
    &\leq 2\max_{\xi \in \bm{\X}\times\bm{\A}} |Q^N_0(\xi)|^2 + \frac{C_T}{N} \sum_{j=0}^{k-1} \bracket{1 + (1+\gamma^2) \max_{\xi \in \bm{\X}\times \bm{\A}}|Q^N_j(\xi)|^2}. \numberthis
\end{align*}
Taking maximum then expectation yields
\begin{align*}
\E\sqbracket{\max_{\xi \in \bm{\X}\times \bm{\A}}|Q^N_k(\xi)|^2} &\leq 2 \E\sqbracket{\max_{\xi \in \bm{\X}\times\bm{\A}} |Q^N_0(\xi)|^2} + \frac{C_T}{N} + \frac{C_T}{N} \sum_{j=0}^{k-1} \E\sqbracket{\max_{\xi \in \bm{\X}\times \bm{\A}}|Q^N_j(\xi)|^2} \\
&\leq C_T + \frac{C_T}{N} \sum_{j=0}^{k-1} \E\sqbracket{\max_{\xi \in \bm{\X}\times \bm{\A}}|Q^N_j(\xi)|^2}. \numberthis
\end{align*}
We conclude by discrete Gronwall's lemma that for all $k \leq TN$:
\begin{equation}
    \E\sqbracket{\max_{\xi \in \bm{\X}\times \bm{\A}}|Q^N_k(\xi)|^2} \leq C_T \exp\bracket{C_T \frac{k}{N}} \leq C_T < +\infty.
\end{equation}
\end{proof}

\begin{lemma}[A-priori $L^2$ bound for the outputs of the actor network]
\label{lem:maximal_P_bounded}
For all $k$ such that $k \leq NT$, there is a $C_T<\infty$ such that 
\begin{equation}
\e \sqbracket{\max_{(x,a) \in \X \times \A} \left|P^N_k(x,a)\right|^2} < C_T.
\end{equation}
\end{lemma}

\begin{proof}
Again we first prove the statement for $k=0$. Since $B^i_0$ and $\sigma(W^i_0 \cdot \xi)$ are bounded by 1,
\begin{align*}
\E\sqbracket{\max_{\xi \in \X \times \A} |P^N_0(\xi)|^2} \leq \E\sqbracket{\sum_{\xi \in \X \times \A} |P^N_0(\xi)|^2} 
&\leq \sum_{\xi \in \X \in \A} \E\sqbracket{\frac{1}{\sqrt{N}}\sum_{i=1}^N B^i_0 \sigma(U^i_0 \cdot \xi)}^2 \\
&\leq \frac{C}{N} \sum_{i=1}^N \E\sqbracket{B^i_0}^2 \leq C < \infty. \numberthis
\end{align*}
The increments could again be controlled by noting
\begin{align*}
    |P^N_{k+1}(\xi) - P^N_k(\xi)| &\leq \frac{1}{\sqrt{N}} \sum_{i=1}^{N} \left[(B^i_{k+1}-B_k^i) \sigma(U_{k+1}^i \cdot \xi) + (\sigma(U_{k+1}^i \cdot \xi) - \sigma(U_{k}^i \cdot \xi)) B_k^i \right] \\
    &\leq \frac{1}{\sqrt{N}} \sum_{i=1}^{N} \left[\abs{B^i_{k+1}-B_k^i} \abs{\sigma(U_{k+1}^i \cdot \xi)} + \abs{\sigma(U_{k+1}^i \cdot \xi) - \sigma(U_{k}^i \cdot \xi)} \abs{B_k^i} \right] 
\end{align*}
By the mean-value inequality and the fact that both $\sigma$ and $\sigma'$ are bounded by 1 by assumption \ref{as:activation_function},
\begin{equation}
    |P^N_{k+1}(\xi) - P^N_k(\xi)| \leq \frac{1}{\sqrt{N}} \sum_{i=1}^{N} \frac{C_T}{N^{3/2}} = \frac{C_T}{N^{2}}.
\end{equation}
Therefore for all $\xi$,
\begin{align*}
    |P^N_k(\xi)|^2 = \bracket{P^N_0(\xi) + \sum_{j=0}^{k-1} (P^N_{j+1}(\xi) - P^N_j(\xi))}^2 &\leq 2 \max_{\xi \in \bm{\X}\times \bm{\A}}|P^N_0(\xi)|^2 + 2N \sum_{j=0}^{k-1}(P^N_{j+1}(\xi) - P^N_j(\xi))^2 \\ 
    &\leq 2 \max_{\xi \in \bm{\X}\times \bm{\A}} |P^N_0(\xi)|^2 + \frac{C_T}{N^2}. \numberthis
\end{align*}
Taking supremum then expectation yields the result.
\end{proof}

\subsection{Taylor Expansions for the Pre-limit Evolutions}

The evolution of the actor and critic network $Q_k^N$ could therefore be studied by using Taylor's expansions. For the critic network, one has:
\begin{align*}
Q^N_{k+1}(\xi) &= Q^N_k(\xi)+\frac{1}{\sqrt{N}} \sum_{i=1}^{N} \left[(C^i_{k+1}-C_k^i) \sigma(W_{k+1}^i \cdot \xi) + (\sigma(W_{k+1}^i \cdot \xi) - \sigma(W_{k}^i \cdot \xi)) C_k^i \right] \\ 
&= Q^N_k(\xi)+\frac{1}{\sqrt{N}} \sum_{i=1}^{N} \bigg[(C^i_{k+1}-C_k^i) \bracket{\sigma(W^i_k \cdot \xi) + \sigma'(W^{i,*}_k \cdot \xi) (W^i_{k+1} - W^i_k) \cdot \xi} \\
&\phantom{=}+ C_k^i \bracket{\sigma'(W^{i}_k \cdot \xi) (W^i_{k+1} - W^i_k) \cdot \xi + \frac{1}{2} \sigma''(W^{i,**}_k \cdot \xi) ((W^i_{k+1}-W^i_k) \cdot \xi)^2} \bigg], \numberthis
\end{align*}
where $W^{i,*}_k$ and $W^{i,**}_k$ are points in the line segment connecting the points $W^i_k$ and $W^i_{k+1}$. Substituting the parameter updates \eqref{eq:updates_repeated}, we have the following pre-limit evolution:
\begin{align*} 
Q^N_{k+1}(\xi)
&= Q^N_k(\xi)+\frac{1}{\sqrt{N}} \sum_{i=1}^{N} \left[(C^i_{k+1}-C_k^i) \sigma(W_k^i \cdot \xi)+\sigma'(W_k^i \cdot \xi) C_k^i (W^i_{k+1}-W_k^i) \cdot \xi \right] \\ 
&\phantom{=}+ \frac{1}{\sqrt{N}} \sum_{i=1}^{N} \bigg[\sigma'(W^{i,*}_k \cdot \xi) (C^i_{k+1}-C_k^i) ((W^i_{k+1} - W^i_k) \cdot \xi)
+ \frac{C_k^i \sigma''(W^{i,**}_k \cdot \xi)}{2} ((W^i_{k+1}-W^i_k) \cdot \xi)^2 \bigg]. \\
&= Q^N_k(\xi) + \frac{\alpha}{N} \left[ r(\xi_k) + \gamma \sum_{a''} Q^N_k(x_{k+1},a'') g^N_k(x_{k+1},a'') - Q^N_k(\xi_k) \right] \\
&\phantom{=}\times \frac{1}{N} \sum_{i=1}^N [\sigma(W^i_k \cdot \xi_k) \sigma(W^i_k \cdot \xi) + (C^i_k)^2 \sigma'(W^i_k \cdot \xi) \sigma(W^i_k \cdot \xi_k) (\xi \cdot \xi_k)] \\
&\phantom{=}+ \frac{1}{\sqrt{N}} \sum_{i=1}^{N} \bigg[\sigma'(W^{i,*}_k \cdot \xi) (C^i_{k+1}-C_k^i) ((W^i_{k+1} - W^i_k) \cdot \xi)
+ \frac{C_k^i \sigma''(W^{i,**}_k \cdot \xi)}{2} ((W^i_{k+1}-W^i_k) \cdot \xi)^2 \bigg]. \numberthis
\end{align*}
Use the definition of $A(\nu)$ as given in \eqref{eq:kernel_matrix}, we have
\begin{align*} 
[A(\nu^N_k)]_{\xi, \xi'} 
&= \langle \sigma(w \cdot \xi')\sigma(w \cdot  \xi) + c^2 \sigma'(w \cdot \xi')\sigma(w \cdot \xi) (\xi' \cdot \xi), \nu_k^N \rangle \\
&= \frac{1}{N} \sum_{i=1}^{N} \left[ \sigma(W_k^i \cdot \xi')\sigma(W_k^i \cdot  \xi) + (C_k^i)^2 \sigma'(W_k^i \cdot \xi')\sigma'(W_k^i \cdot  \xi) (\xi' \cdot \xi) \right].
\end{align*}
Therefore
\begin{align*}
Q^N_{k+1}(\xi) &= Q^N_k(\xi)+\frac{\alpha}{N} \left[ r(\xi_k) + \gamma \sum_{a''} Q^N_k(x_{k+1},a'') g^N_k(x_{k+1},a'') - Q^N_k(\xi_k) \right] [A(\nu^N_k)]_{\xi,\xi'}\\
&\phantom{=} + \frac{1}{\sqrt{N}} \sum_{i=1}^{N} \bigg[\sigma'(W^{i,*}_k \cdot \xi) (C^i_{k+1}-C_k^i) ((W^i_{k+1} - W^i_k) \cdot \xi)
+ \frac{C_k^i \sigma''(W^{i,**}_k \cdot \xi)}{2} ((W^i_{k+1}-W^i_k) \cdot \xi)^2 \bigg]. \numberthis
\end{align*}
For the actor network, one has
\begin{align*}
P^N_{k+1}(\xi) &= P^N_k(\xi) + \frac{1}{\sqrt{N}} \sum_{i=1}^{N} \left[(B^i_{k+1}-B_k^i) \sigma(U_{k+1}^i \cdot \xi) + (\sigma(U_{k+1}^i \cdot \xi) - \sigma(U_{k}^i \cdot \xi)) B_k^i \right]\\
&= P^N_k(\xi) + \frac{1}{\sqrt{N}} \sum_{i=1}^{N} \bigg[(B^i_{k+1}-B_k^i) \bracket{\sigma(U^i_k \cdot \xi) + \sigma'(U^{i,*}_k \cdot \xi) (U^i_{k+1} - U^i_k) \cdot \xi} \\
&\phantom{=}+ B_k^i \bracket{\sigma'(U^{i}_k \cdot \xi) (U^i_{k+1} - U^i_k) \cdot \xi + \frac{1}{2} \sigma''(U^{i,**}_k \cdot \xi) ((U^i_{k+1}-U^i_k) \cdot \xi)^2} \bigg] \\
&= P^N_k(\xi) + \sum_{i=1}^{N} \frac{\zeta^N_k}{N^2} \clip(Q^N_k(\tilde{\xi}_{k})) \Bigg[\sigma(U_k^i \cdot \xi)\bracket{\sigma(U^i_k \cdot \tilde{\xi}_k) - \sum_{a''} f^N_k(\tilde{x}_k, a'') \sigma(U^i_k \cdot (\tilde{x}_k, a''))} \\
&\phantom{=}+ (B_k^i)^2 \sigma'(U_k^i \cdot \xi) \bracket{\sigma'(U^i_k \cdot \tilde{\xi}_k) (\xi \cdot \tilde{\xi}_k) - \sum_{a''} f^N_k(\tilde{x}_k, a'') \sigma'(U^i_k \cdot (\tilde{x}_k, a'')) ((\tilde{x}_k, a'') \cdot \xi) } \Bigg] \\
&\phantom{=}+ \frac{1}{\sqrt{N}} \sum_{i=1}^{N} \bigg[\sigma'(U^{i,*}_k \cdot \xi) (B^i_{k+1}-B_k^i) ((U^i_{k+1} - U^i_k) \cdot \xi) + \frac{B^i_k \sigma''(U^{i,**}_k \cdot \xi)}{2} ((U^i_{k+1}-U^i_k) \cdot \xi)^2 \bigg] \numberthis \label{eq:4.6}
\end{align*}
where $U^{i,*}_k$ and $U^{i,**}_k$ are points in the line segment connecting the points $U^i_k$ and $U^i_{k+1}$. Again, we the notation $A(\nu)$ as defined in \eqref{eq:kernel_matrix} to simplify the evolution equation \eqref{eq:4.6}:
\begin{align*} 
[A(\mu^N_k)]_{\xi,\xi} 
&= \langle \sigma(u \cdot \xi')\sigma(u \cdot \xi) + b^2 \sigma'(u \cdot \xi')\sigma'(u \cdot  \xi) (\xi' \cdot \xi), \mu^N_k \rangle \\
&= \frac{1}{N} \sum_{i=1}^{N} \left[ \sigma(U_k^i \cdot \xi')\sigma(U_k^i \cdot  \xi) + (B_k^i)^2 \sigma'(U_k^i \cdot \xi')\sigma'(U_k^i \cdot  \xi) (\xi' \cdot \xi) \right], \numberthis \label{eq:bar_B}
\end{align*}
so,
\begin{align*}
P^N_{k+1}(\xi) &= P^N_k(\xi) + \frac{\zeta^N_k}{N} \clip(Q^N_k(\tilde{\xi}_k)) \sqbracket{[A(\mu^N_k)]_{\xi, \tilde{\xi}_k} - \sum_{a''} f^N_k(\tilde{x}_k, a'') [A(\mu^N_k)]_{\xi,(\tilde{x}_k,a'')}} \\
&\phantom{=}+ \frac{1}{\sqrt{N}} \sum_{i=1}^{N} \bigg[\sigma'(U^{i,*}_k \cdot \xi) (B^i_{k+1}-B_k^i) ((U^i_{k+1} - U^i_k) \cdot \xi) + \frac{B^i_k \sigma''(U^{i,**}_k \cdot \xi)}{2} ((U^i_{k+1}-U^i_k) \cdot \xi)^2 \bigg] \numberthis
\end{align*}

We are now ready to prove the Proposition \ref{prop:pre-limit evolution}.

\begin{proof}[Proof of Proposition \ref{prop:pre-limit evolution}]
We begin by noting for all $\xi$,
\begin{align*}
&\phantom{=}\abs{Q^N_{k+1}(\xi) - Q^N_k(\xi) - \frac{\alpha}{N} \bracket{r(\xi_k) + \gamma \sum_{a''} Q^N_k(x_{k+1},a'') g^N_k(x_{k+1},a'') - Q^N_k(\xi_k)} [A(\nu^N_k)]_{\xi,\xi_k}} \\
&= \frac{1}{\sqrt{N}} \abs{\sum_{i=1}^N \sigma'(W^{i,*}_k \cdot \xi) (C^i_{k+1}-C^i_k)(W^i_{k+1}-W^i_k) \cdot \xi + \frac{\sigma''(W^{i,**}_k \cdot \xi) C^i_k}{2}((W^i_{k+1}-W^i_k)\cdot \xi)^2} \\
&\overset{\mathsf{(CS)}}\leq \frac{1}{\sqrt{N}} \sum_{i=1}^N \sqbracket{\abs{C^i_{k+1}-C^i_k} \norm{W^i_{k+1}-W^i_k} \|\xi\| + C_T \norm{W^i_{k+1}-W^i_k}^2 \|\xi\|^2} \\
&\leq \frac{1}{\sqrt{N}} \sum_{i=1}^N \frac{C_T}{N^{3}} \bracket{1+(1+\gamma) \max_{\xi \in \bm{\X}\times\bm{\A}}|Q^N_k(\xi)|}^2 \leq \frac{C_T}{N^{5/2}} \bracket{1 + (1+\gamma)^2 \max_{\xi \in \bm{\X}\times\bm{\A}} |Q^N_k(\xi)|^2} \numberthis
\end{align*}
Taking maximum and expectation yields
\begin{align*}
&\phantom{=}\E\sqbracket{\max_\xi \abs{Q^N_{k+1}(\xi) - Q^N_k(\xi) - \frac{\alpha}{N} \bracket{ r(\xi_k) + \gamma \sum_{a''} Q^N_k(x_{k+1},a'') g^N_k(x_{k+1},a'') - Q^N_k(\xi_k)} [A(\nu^N_k)]_{\xi,\xi_k,k}}} \\
&\leq \frac{C_T}{N^{5/2}} \E\sqbracket{1 + (1+\gamma)^2 \max_{\xi \in \bm{\X}\times\bm{\A}} |Q^N_k(\xi)|^2} \leq \frac{C_T}{N^{5/2}}. \numberthis
\end{align*}
Similarly, for all $\xi$,
\begin{align*}
&\phantom{=}\abs{P^N_{k+1}(\xi) - P^N_k(\xi) - \frac{\zeta^N_k}{N} \clip(Q^N_k(\tilde{\xi}_k)) \bracket{[A(\mu^N_k)]_{\xi, \tilde{\xi}_k} - \sum_{a''} f^N_k(\tilde{x}_k, a'') [A(\mu^N_k)]_{\xi,(\tilde{x}_k,a'')}}} \\
&=\frac{1}{\sqrt{N}} \abs{\sum_{i=1}^N \sigma'(U^{i,*}_k \cdot \xi) (B^i_{k+1}-B^i_k)(U^i_{k+1}-U^i_k) \cdot \xi + \frac{\sigma''(U^{i,**}_k \cdot \xi) C^i_k}{2}((U^i_{k+1}-U^i_k)\cdot \xi)^2} \\
&\overset{\mathsf{(CS)}}\leq \frac{1}{\sqrt{N}} \sum_{i=1}^N \sqbracket{\abs{B^i_{k+1}-B^i_k} \norm{U^i_{k+1}-U^i_k} \|\xi\| + C_T \norm{U^i_{k+1}-U^i_k}^2 \|\xi\|^2} \leq \frac{C_T}{N^{5/2}}.
\end{align*}
This completes the proof.
\end{proof}

Using the notation as introduced in the Definition \ref{def:stochastic_boundedness}, one could write
\begin{align}
Q^N_{k+1}(\xi) &= Q^N_k(\xi)+\frac{\alpha}{N} \left[ r(\xi_k) + \gamma \sum_{a''} Q^N_k(x_{k+1},a'') g^N_k(x_{k+1},a'') - Q^N_k(\xi_k) \right] [A(\nu^N_k)]_{\xi,\xi_k} + O_p(N^{-5/2}). \label{eq:evolution_of_q_2} \\
P^N_{k+1}(\xi) &= P^N_k(\xi) + \frac{\zeta^N_k}{N} \clip(Q^N_k(\tilde{\xi}_k)) \sqbracket{ [A(\mu^N_k)]_{\xi, \tilde{\xi}_k, k} - \sum_{a''} f^N_k(\tilde{x}_k, a'')[A(\mu^N_k)]_{\xi,(\tilde{x}_k,a''),k}} + O(N^{-5/2}). \label{eq:evolution_of_p_2}
\end{align}

\section{Proofs for Relative Compactness}
\subsection{Proof of compact containment (Lemma \ref{lem:compact_containment})} \label{SS:proof_compact_containment}

\begin{proof}[Proof of Lemma \ref{lem:compact_containment}]
Let $K_L = [-L,L]^{1+d}$ denotes a compact subset in $\R^{1+d}$. We then see that for any $t \geq 0$ and $N \in \N$,
\begin{align*}
\E\sqbracket{\nu_t^N\bracket{\R^{1+d} \setminus K_L}} &= \frac{1}{N} \sum_{i=1}^N \p\bracket{(C^i_{\floor{Nt}}, W^i_{\floor{Nt}}) \in \R^{1+d} \setminus K_L} \\
&\leq \frac{1}{N} \sum_{i=1}^N \p\bracket{\abs{C^i_{\floor{Nt}}} + \norm{W^i_{\floor{Nt}}} \geq L} \leq \frac{C_T}{L}, \numberthis \label{eq:simple_bound_of_nunt}
\end{align*}
where the final step is by $\abs{C^i_{\floor{Nt}}} + \norm{W^i_{\floor{Nt}}}$ is integrable (from Lemma \ref{NN parameter bound}) and Chebyshev's inequality. We define the following subset of $\mathcal{M}\bracket{\R^{1+d}}$
\begin{equation}
\hat{K}_L = \overline{\set{\nu \in \mathcal{M}\bracket{\R^{1+d}} \,\bigg|\, \nu\bracket{\R^{1+d} \setminus K_{(L+j)^2}} < \frac{1}{\sqrt{L+j}} \text{ for all } j}},
\end{equation}
which is a closure of a tight family of measures and thus being a compact subset of $\M(\R^{1+d})$. Observe that
\begin{align*}
\p\bracket{\nu^N_t \notin \hat{K}_L} &\leq \p\bracket{\exists j \text{ s.t. } \nu_t^N(\R^{1+d} \setminus K_{(L+j)^2}) > \frac{1}{\sqrt{L+j}}} \\
&\leq \sum_{j=1}^\infty \p\bracket{\nu_t^N(\R^{1+d} \setminus K_{(L+j)^2}) > \frac{1}{\sqrt{L+j}}} \\
&\overset{(a)}{\leq} \sum_{j=1}^\infty \frac{\E\sqbracket{\nu_t^N(\R^{1+d} \setminus K_{(L+j)^2})}}{(L+j)^{-1/2}} \\
&\overset{(b)}{\leq} \sum_{j=1}^\infty \frac{C_T}{(L+j)^{3/2}} < \infty,
\end{align*}
where step $(a)$ is from Chebyshev's inequality and step $(b)$ from \eqref{eq:simple_bound_of_nunt}. By dominated convergence theorem for infinite sum, we see that $\sum_{j\geq 1} (L+j)^{-3/2} \to 0$ as $L \to +\infty$, thus for any $\eta > 0$ there is an $L$ such that 
\begin{equation*}
\sup_{N \in \N, t \in [0,T]} \p\bracket{\nu^N_t \notin \hat{K}_L} < \frac{\eta}{4}.
\end{equation*}
With the exact same argument, we can also make $L$ large enough such that
\begin{equation*}
\sup_{N \in \N, t \in [0,T]} \p\bracket{\mu^N_t \notin \hat{K}_L} < \frac{\eta}{4}.
\end{equation*}
As we have shown in Lemma \ref{lem:maximal_Q_bounded} and \ref{lem:maximal_P_bounded} that the $L^2$ norm of $P$ and $Q$ are locally bounded, so by Chebyshev's inequality we know for each $\eta > 0$, there exists $B > 0$ such that
\begin{equation*}
\sup_{N \in \N, t \in [0,T]} \p\bracket{Q^N_t \notin [-B,B]^{M}} < \frac{\eta}{4},
\end{equation*}
and 
\begin{equation*}
\sup_{N \in \N, t \in [0,T]} \p\bracket{P^N_t \notin [-B,B]^{M}} < \frac{\eta}{4}.
\end{equation*}
Therefore, for each $\eta > 0$, there is a compact set $\mathcal{K} := \hat{K}_L \times \hat{K}_L \times [-B,B]^{M} \times [-B,B]^{M} \subseteq E$ such that 
\begin{equation*}
\sup_{N \in \N, 0 \leq t \leq T} \p\sqbracket{\bracket{\mu_t^N, \nu_t^N, P_t^N, Q_t^N} \notin \K} < \eta,
\end{equation*}
which completes the proof.
\end{proof}

\subsection{Evolution of empirical measure}
\label{SS:evolution_of_empirical_measure}
The evolution of the empirical measure $\nu^N_k$ 
can be characterized in terms of their projection onto test functions $\varphi \in C^2_b(\R^{1+M})$, by Taylor's expansion
\begin{align*}
\la \varphi, \nu^N_{k+1} \ra - \la \varphi, \nu^N_k \ra 
&= \frac{1}{N} \sum_{i=1}^N (\varphi(C^i_{k+1},W^i_{k+1}) - \varphi(C^i_k,W^i_k)) \\
&= \frac{1}{N} \sum_{i=1}^N \bigg[\de_c \varphi(C^i_k,W^i_k) (C^i_{k+1} - C^i_k) + \de_w \varphi(C^i_k,W^i_k) \cdot (W^i_{k+1} - W^i_k) \\
&\phantom{=}+ \frac{1}{2} \Big(\de^2_c \varphi(C^{i,*}_k,W^{i,*}_k) (C^i_{k+1} - C^i_k)^2 +  (C^i_{k+1} - C^i_k) \de^2_{cw} \varphi(C^{i,**}_k,W^{i,**}_k) (W^i_{k+1} - W^i_k) \\
&\phantom{=}+ (W^i_{k+1} - W^i_k) \cdot \de^2_{w} \varphi(C^{i,***}_k,W^{i,***}_k) (W^i_{k+1} - W^i_k) \Big)\bigg], \numberthis \label{eq:Q_measure_pre_limit}
\end{align*}
where $\bracket{C^{i,*}_k, W^{i,*}_k}, \bracket{C^{i,**}_k, W^{i,**}_k}, \bracket{C^{i,***}_k, W^{i,***}_k}$ are points lying on the line segments connecting between $\bracket{C^{i}_k, W^{i}_k}$ and $\bracket{C^{i}_{k+1}, W^{i}_{k+1}}$. Substituting \eqref{NACCriticupdates} into \eqref{eq:Q_measure_pre_limit}, we have 
\begin{align*}
\la \varphi, \nu^N_{k+1} \ra - \la \varphi, \nu^N_k \ra 
&= \frac{1}{N} \sum_{i=1}^N \sqbracket{\de_c \varphi(C^i_k,W^i_k) (C^i_{k+1} - C^i_k) + \de_w \varphi(C^i_k,W^i_k) \cdot (W^i_{k+1} - W^i_k)} + O_p(N^{-2}) \\
&= \alpha N^{-
\frac52} \bracket{r(\xi_k) + \gamma \sum_{a''} Q^N_k(x_{k+1},a'') g^N_k(x_{k+1},a'') - Q^N_k(\xi_k)} \\
&\pheq \times \sum_{i=1}^N \bracket{\de_c \varphi(C^i_k, W^i_k)\sigma(W^i_k \cdot \xi_k)) + C^i_k \sigma'(W^i_k \cdot \xi_k) \de_w \varphi(C^i_k, W^i_k) \xi_k} + O_p(N^{-2}) \\
&= \alpha N^{-\frac32} \bracket{r(\xi_k) + \gamma \sum_{a''} Q^N_k(x_{k+1},a'') g^N_k(x_{k+1},a'') - Q^N_k(\xi_k)} \\
&\pheq \times \la \de_c \varphi(c,w) \sigma(w \cdot \xi_k) + c \sigma'(w \cdot \xi_k) \de_w \varphi(c,w) \xi_k, \nu^N_k \ra + O_p(N^{-3}). \numberthis
\end{align*}
Therefore, the time-rescled empirical measure $\nu^N_t := \nu^N_{\floor{Nt}}$ satisfies
\begin{align*}
\la \varphi, \nu^N_t \ra - \la \varphi, \nu^N_0 \ra 
&= \alpha N^{-\frac32} \sum_{k=0}^{\floor*{Nt}-1} \bracket{r(\xi_k) + \gamma \sum_{a''} Q^N_k(x_{k+1},a'') g^N_k(x_{k+1},a'') - Q^N_k(\xi_k)}  \\
&\pheq \times \la \de_c \varphi(c,w) \sigma(w \cdot \xi_k) + c \sigma'(w \cdot \xi_k) \de_w \varphi(c,w) \xi_k, \nu^N_k \ra + O_p(N^{-2}). \numberthis
\end{align*}
We can similarly characterise the evolution of the empirical measure $\mu^N_k$ in terms of their projection onto any test functions $\varphi \in C^2_b(\R^{1+M})$:
\begin{align*}
\la \varphi, \mu^N_{k+1} \ra - \la \varphi, \mu^N_k \ra 
&= \frac{1}{N} \sum_{i=1}^N \bigg[\de_b \varphi(B^i_k,U^i_k) (B^i_{k+1} - B^i_k) + \de_u \varphi(B^i_k,U^i_k) \cdot (U^i_{k+1} - U^i_k) \bigg]+ O_p(N^{-2}) \\
&= \frac{1}{N^{\frac52}} \sum_{i=1}^N \zeta^N_k \clip(Q^N_k(\tilde{\xi}_k)) \Bigg[ \sigma(U^i_k \cdot \tilde{\xi}_k) (\partial_b \varphi(B^i_k, U^i_k) - B^i_k \partial_w\varphi(B^i_k, U^i_k) \cdot \xi_k) \\
&- \sum_{a''} f^N_k(\tilde{x}_k, a'') \sigma'(U^i_k \cdot (\tilde{x}_k, a'')) \bracket{\partial_b \varphi(B^i_k, U^i_k) - B^i_k \partial_w\varphi(B^i_k, U^i_k) \cdot (\tilde{x}_k, a'')} \Bigg] + O_p(N^{-2}) \\
&= \frac{1}{N^{\frac32}} \xi^N_k \clip(Q^N_k(\tilde{\xi}_k)) \Big[ \la \sigma(u \cdot \tilde{\xi}_k) (\partial_b \varphi(b,u) - b \partial_w\varphi(b,u) \cdot \xi_k), \mu^N_k \ra \\
&- \sum_{a''} f^N_k(\tilde{x}_k, a'') \la \sigma'(u \cdot (\tilde{x}_k, a'')) \bracket{\partial_b \varphi(b,u) - b \partial_w\varphi(b,u) \cdot (\tilde{x}_k, a'')}, \mu^N_k \ra \Big]+ O_p(N^{-2}), \numberthis \label{eq:P_pre_limit}
\end{align*}
and hence
\begin{align}
\label{eq:prelimit_evolution_of_mu}
\la \varphi, \mu^N_t \ra - \la \varphi, \mu^N_0 \ra 
&= \frac{1}{N^{\frac32}} \sum_{k=0}^{\floor{Nt}-1} \xi^N_k \clip(Q^N_k(\tilde{\xi}_k)) \Big[ \la \sigma(u \cdot \tilde{\xi}_k) (\partial_b \varphi(b,u) - b \partial_w\varphi(b,u) \cdot \xi_k), \mu^N_k \ra \\
&- \sum_{a''} f^N_k(\tilde{x}_k, a'') \la \sigma'(u \cdot (\tilde{x}_k, a'')) \bracket{\partial_b \varphi(b,u) - b \partial_w\varphi(b,u) \cdot (\tilde{x}_k, a'')}, \mu^N_k \ra \Big]+ O_p(N^{-1}).
\end{align}

\subsection{Proof of path regularities (Lemmas \ref{lem:regularity_of_nu} and \ref{lem:regularity_of_Q})}
\label{SS:proof_of_path_regularities}

\begin{proof}[Proof of Lemma \ref{lem:regularity_of_nu}]
We start by the following Taylor's expansion for $0 \leq s < t \leq T$:
\begin{align*}
&\pheq \abs{\left\langle f, \nu_{t}^{N}\right\rangle-\left\langle f, \nu_{s}^{N}\right\rangle} \\
&=\left|\left\langle f, v_{\lfloor N t\rfloor}^{N}\right\rangle-\left\langle f, v_{\lfloor N s\rfloor}^{N}\right\rangle\right| \\
&\leq \frac{1}{N} \sum_{i=1}^{N}\left|f\left(C_{\lfloor N t\rfloor}^{i}, W_{\lfloor N t\rfloor}^{i}\right)-f\left(C_{\lfloor N s\rfloor}^{i}, W_{\lfloor N s\rfloor}^{i}\right)\right| \\
&\leq \frac{1}{N} \sum_{i=1}^{N} \abs{\partial_{c} f\bracket{\bar{C}_{\lfloor N t\rfloor}^{i}, \bar{W}_{\lfloor N t\rfloor}^{i}}} \abs{C_{\lfloor N t\rfloor}^{i}-C_{\lfloor N s\rfloor}^{i}} 
+ \frac{1}{N} \sum_{i=1}^{N} \norm{\partial_{w} f\left(\bar{C}_{\lfloor N t\rfloor}^{i}, \bar{W}_{\lfloor N t\rfloor}^{i}\right)} \norm{W_{\lfloor Nt \rfloor}^{i}-W_{\lfloor Ns \rfloor}^{i}}, \numberthis \label{eq:taylor_for_nu}
\end{align*}
where $\bar{C}^{i}, \bar{W}^{i}$ are in the segments connecting $C_{\lfloor Ns \rfloor}^{i}$ to $C_{\lfloor Nt \rfloor}^{i}$ and $W_{\left\lfloor N_{s}\right\rfloor}^{i}$ to $W_{\lfloor N t\rfloor}^{i}$ respectively.\\

Let's now establish a bound on $\abs{C_{\floor{Nt}}^{i}-C_{\floor{Ns}}^{i}}$ for $s<t \leq T$ with $0<t-s \leq \delta<1$.
\begin{align*}
&\E\sqbracket{\abs{C_{\floor{Nt}}^{i}-C_{\floor{Ns}}^{i}} \,|\, \mathcal{F}_{s}^{N}}
=\mathbb{E}\sqbracket{\abs{\sum_{k=\floor{Ns}}^{\floor{Nt}-1} \bracket{C_{k+1}^{i}-C_{k}^{i}}} \,|\, \mathcal{F}_{s}^{N}} \\
\leq& \E\sqbracket{\sum_{k=\floor{Ns}}^{\floor{Nt}-1} \frac{\alpha}{N^{\frac32}} \abs{r(\xi_k) + \gamma \sum_{a''} Q^N_k(x_{k+1},a'') g^N_k(x_{k+1}, a'') - Q^N_k(\xi_k)} \cdot \abs{\sigma(W^i_k \cdot \xi_k)} \,|\,  \mathcal{F}_{s}^{N}} \\
\leq& \frac{\alpha C}{N^{\frac32}} \sum_{k=\floor{Ns}}^{\floor{Nt}-1} \bracket{C + (\gamma+1) \E\sqbracket{\sup_{\xi \in \X \times \A} |Q^N_k(\xi)|}} \\ 
\overset{(a)}{\leq}& \frac{C(\floor{Nt} - \floor{Ns})}{N^{\frac32}} \bracket{C+(\gamma+1) C_T^{1/2}} \\
\leq& \frac{C_T(N(t-s)+1)}{N^{\frac32}} 
\leq \frac{C_T}{\sqrt{N}} \delta + \frac{C_T}{N^{\frac32}}. \numberthis \label{eq:Lipschitz_c} 
\end{align*}
where step $(a)$ is by Lemma \ref{lem:maximal_Q_bounded}. 
Similarly for $\norm{W_{\floor{Nt}}^{i}-W_{\floor{Ns}}^{i}}$ for any $s < t \leq T$ with $0 < t-s \leq \delta < 1$,
\begin{align*}
&\phantom{=}\E\sqbracket{\norm{W_{\floor{Nt}}^{i}-W_{\floor{Ns}}^{i}} \mid \mathcal{F}_{s}^{N}} \\
&= \E\sqbracket{\norm{\sum_{k=\floor{Ns}}^{\floor{Nt}-1}\bracket{W_{k+1}^{i}-W_{k}^{i}}} \,|\, \mathcal{F}_{s}^{N}} \\
&\leq \E\sqbracket{\sum_{k=\floor{Ns}}^{\floor{Nt}-1} \frac{\alpha}{N^{\frac32}} \abs{r(x_k, a_k) + \gamma \sum_{a''} Q^N_k(x_{k+1}, a'') g^N_k(x_{k+1}, a'') - Q^N_k(x_k, a_k)} \cdot \abs{C^i_k} \cdot \abs{\sigma'(W^i_k \cdot (x_k, a_k))} \,|\,  \mathcal{F}_{s}^{N}} \\
&\leq \frac{\alpha C_T}{N^{\frac32}} \sum_{k=\floor{Ns}}^{\floor{Nt}-1} \bracket{C + (\gamma+1) \E\sqbracket{\sup_{\xi \in \X \times \A} |Q^N_k(\xi)|}} \\
&\leq \frac{C_T}{\sqrt{N}} \delta + \frac{C_T}{N^{3/2}}, \numberthis \label{eq:Lipschitz_w}
\end{align*}
where we have used the bound in Lemma \ref{NN parameter bound} and \ref{lem:maximal_Q_bounded} again. Combine \eqref{eq:Lipschitz_c}, \eqref{eq:Lipschitz_w} and \eqref{eq:taylor_for_nu}, we have for any $0 \leq s < t \leq T$ with $ 0<t-s\leq \delta < 1$
\begin{equation}
\label{eq:bound_for_weak_convergence_nu}
\E\sqbracket{\abs{\la f, \nu^N_t \ra - \la f, \nu^N_s \ra} \mid \mathcal{F}_{s}^{N}} \leq \frac{C_T}{\sqrt{N}}\delta + \frac{C_T}{N^{3/2}} \leq C_T\delta + \frac{C_T}{N^{3/2}}.
\end{equation}

Similarly for $\mu_t^N$, we have by Taylor's expansion that for $0 \leq s < t \leq T$ with $0 \leq s < t \leq T$ that 
\begin{align*}
&\pheq \abs{\left\langle f, \mu_{t}^{N}\right\rangle-\left\langle f, \mu_{s}^{N}\right\rangle} \\
&=\left|\left\langle f, \mu_{\lfloor N t\rfloor}^{N}\right\rangle - \left\langle f, \mu_{\lfloor N s\rfloor}^{N}\right\rangle\right| \\
&\leq \frac{1}{N} \sum_{i=1}^{N}\left|f\left(B_{\lfloor N t\rfloor}^{i}, U_{\lfloor N t\rfloor}^{i}\right)-f\left(B_{\lfloor N s\rfloor}^{i}, U_{\lfloor N s\rfloor}^{i}\right)\right| \\
&\leq \frac{1}{N} \sum_{i=1}^{N} \abs{\partial_{b} f\bracket{\bar{B}_{\lfloor N t\rfloor}^{i}, \bar{U}_{\lfloor N t\rfloor}^{i}}} \abs{B_{\lfloor N t\rfloor}^{i}-B_{\lfloor N s\rfloor}^{i}} 
+ \frac{1}{N} \sum_{i=1}^{N} \norm{\partial_{u} f\left(\bar{B}_{\lfloor N t\rfloor}^{i}, \bar{U}_{\lfloor N t\rfloor}^{i}\right)}  \norm{U_{\lfloor Nt \rfloor}^{i}-U_{\lfloor Ns \rfloor}^{i}}, \numberthis \label{eq:taylor_for_mu} 
\end{align*}
and 
\begin{align*}
\E\sqbracket{\abs{B_{\floor{Nt}}^{i}-B_{\floor{Ns}}^{i}} \,|\, \mathcal{F}_{s}^{N}}
& \leq \frac{C}{N^{\frac32}} \sum_{k=\floor{Ns}}^{\floor{Nt}-1} \E |\clip(Q^N_k(\tilde{x}_k,\tilde{a}_k))|
\leq \frac{C(N(t-s)+1)}{N^{\frac32}}
\leq \frac{C_T}{\sqrt{N}} \delta + \frac{C_T}{N^{\frac32}} \\
\E\sqbracket{\norm{U_{\floor{Nt}}^{i}-U_{\floor{Ns}}^{i}} \,|\, \mathcal{F}_{s}^{N}} 
&\leq\E\sqbracket{\sum_{k=\floor{Ns}}^{\floor{Nt}-1} CN^{-\frac32} |\clip(Q^N_k(\tilde{\xi}_k))| |B^i_k|} 
\leq \frac{C_T}{\sqrt{N}} \delta + \frac{C_T}{N^{\frac32}}. \numberthis \label{B U diff}
\end{align*}
With the fact that the terms $\abs{\partial_b f(\bar{B}^i_{\floor{Nt}}, \bar{U}^i_{\floor{Nt}})}$ and $\norm{\partial_w f(\bar{B}^i_{\floor{Nt}}, \bar{U}^i_{\floor{Nt}})}$ are bounded, we have that that for $0 \leq s < t \leq T$ with $0<t-s\leq \delta < 1$
\begin{equation}
\label{eq:bound_for_weak_convergence_mu}
\E\sqbracket{\abs{\la f, \mu^N_t \ra - \la f, \mu^N_s \ra} \mid \mathcal{F}_{s}^{N}} \leq \frac{C_T}{\sqrt{N}}\delta + \frac{C_T}{N^{3/2}} \leq C_T\delta + \frac{C_T}{N^{3/2}}.
\end{equation}
\end{proof}

\begin{proof}[Proof of Lemma \ref{lem:regularity_of_Q}]
Recalling the assumption that the state-action space is finite, it suffices to prove a uniform bound for the increments of the outputs $P^N(\xi), Q^N(\xi)$. In particular, by \eqref{eq:evolution_of_q_2} we have 
\begin{align*}
\E \sqbracket{\sup_{\xi}|Q^N_{k+1}(\xi) - Q^N_k(\xi)|}
&\leq \frac{\alpha}{N} \E\left| r(\xi_k) + \gamma \sum_{a''} Q^N_k(x_{k+1},a'') - Q^N_k(\xi_k) \right| |[A(\nu^N_k)]_{\xi,\xi_k}| + \frac{C_T}{N^{3/2}} \\
&\leq \frac{C_T}{N} + \frac{C_T}{N^{3/2}}, \numberthis
\end{align*}
and that by \eqref{eq:evolution_of_p_with_tensor} we have
\begin{align*}
\sup_\xi |P^N_{k+1}(\xi) - P^N_k(\xi)| &\leq \frac{\zeta^N_k}{N} |\clip(Q^N_k(\tilde{\xi}_k))| \abs{[A(\mu^N_k)]_{\xi, \tilde{\xi}_k} - \sum_{a''} f^N_k(\tilde{x}_k', a'') [A(\mu^N_k)]_{\xi,(\tilde{x}_k,a'')}} + \frac{C_T}{N^{3/2}} \\
&\leq \frac{C_T}{N} + \frac{C_T}{N^{3/2}}. \numberthis
\end{align*}
In fact, one could prove a stronger inequality.
\begin{align*}
&\phantom{=}\E\sqbracket{\sup_\xi \abs{Q^N_t(\xi) - Q^N_s(\xi)}} \\
&\leq \sum_{k=\floor{Ns}}^{\floor{Nt}-1} \E \sqbracket{\sup_\xi \abs{Q_{k+1}(\xi) - Q_{k}(\xi)}} \\
&\leq \sum_{k=\floor{Ns}}^{\floor{Nt}-1} \sup_\xi \sqbracket{\frac{1}{\sqrt{N}} \sum_{i=1}^N \abs{(C^i_{k+1}-C_k^i) \sigma(W_k^i \cdot \xi)+\sigma'(W_k^i \cdot \xi)\xi^\top (W^i_{k+1}-W_k^i)C_k^i}+O_p(N^{-5/2})} \\
&\leq \sum_{k=\floor{Ns}}^{\floor{Nt}-1} \sqbracket{\frac{C}{\sqrt{N}}  \sum_{i=1}^N \bracket{|C^i_{k+1} - C^i_k| + \norm{W^i_{k+1} - W^i_k}} + O_p(N^{-5/2})}. \numberthis
\end{align*}
Taking expectations and using the bounds \eqref{eq:Lipschitz_c} and \eqref{eq:Lipschitz_w}, we have
\begin{align*}
\E\sqbracket{\sup_\xi \abs{Q^N_t(\xi) - Q^N_s(\xi)} \,|\, \F^N_s} 
&\leq \sum_{k=\floor{Ns}}^{\floor{Nt}-1} \sqbracket{\frac{C}{\sqrt{N}} \sum_{i=1}^N \bracket{\E\sqbracket{|C^i_{k+1} - C^i_k| + \norm{W^i_{k+1} - W^i_k} \,|\, \F^N_s}} + \E[O_p(N^{-5/2})]} \\
&\leq \frac{C}{\sqrt{N}} \sum_{i=1}^N \bracket{\frac{C_T}{\sqrt{N}}\delta + \frac{C_T}{N^{3/2}}} \leq C_T \delta + \frac{C_T}{N}. \numberthis
\end{align*}
With exactly the same arguments, we can derive
\begin{align*}
\abs{P^N_t(\xi) - P^N_s(\xi)} 
&=\abs{P^N_{\floor{Nt}}(\xi) - P^N_{\floor{Ns}}(\xi)} 
&\leq \sum_{k=\floor{Ns}}^{\floor{Nt}-1} \sqbracket{\frac{C}{\sqrt{N}} \sum_{i=1}^N \bracket{|B^i_{k+1} - B^i_k| + \norm{U^i_{k+1} - U^i_k}} + O(N^{-5/2})},
\end{align*}
which together with \eqref{B U diff} derive
\begin{align*}
\E\sqbracket{\sup_\xi \abs{P^N_t(\xi) - P^N_s(\xi)} \,|\, \F^N_s} 
&\leq C_T \delta + \frac{C_T}{N}.
\end{align*}
\end{proof}

\section{Further Analyses of the Critic and Actor Convergence}

\subsection{Proof of Comparison Principle (Lemma \ref{lem:comparison_new})} \label{SS:proof_of_comparison_principle}

\begin{proof}[Proof of Lemma \ref{lem:comparison_new}]
Consider the differential equation
\begin{equation}
\frac{dZ_t}{dt} = -\frac{\mathsf{C}_1}{(1+t)^{\gamma_1}} Z_t + \frac{\mathsf{C}_2}{(1+t)^{\gamma_2}}, \quad Z_0=Y_0.
\end{equation}
\paragraph{Case $\gamma_1 \in (0,1)$: } This admits the solution 
\begin{equation}
Z_t = \frac{Y_0}{\exp\left(\frac{\mathsf{C}_1}{1-\gamma_1} (1+t)^{1-\gamma_1} \right)} + \frac{\int_0^t  \mathsf{C}_2 (1+s)^{-\gamma_2} \exp\left(\frac{\mathsf{C}_1}{1-\gamma_1} (1+s)^{1-\gamma_1}\right) \, ds}{\exp\left(\frac{\mathsf{C}_1}{1-\gamma_1} (1+t)^{1-\gamma_1} \right)}
\end{equation}
This enables us to compute the asymptotics of the second term using the L'H\^{o}pital's rule:
\begin{align*}
&\phantom{=} \lim_{t\to\infty} \frac{\int_0^t  \mathsf{C}_2 (1+s)^{-\gamma_2} \exp\left(\frac{\mathsf{C}_1}{1-\gamma_1} (1+s)^{1-\gamma_1}\right) \, ds}{(1+t)^{\gamma_1-\gamma_2} \exp\left(\frac{\mathsf{C}_1}{1-\gamma_1} (1+t)^{1-\gamma_1}\right)} \\
&= \lim_{t\to\infty} \frac{\mathsf{C}_2 (1+t)^{-\gamma_2} \exp\left(\frac{\mathsf{C}_1}{1-\gamma_1} (1+t)^{1-\gamma_1}\right)}{(\gamma_1-\gamma_2)(1+t)^{\gamma_1-\gamma_2-1} \exp\left(\frac{\mathsf{C}_1}{1-\gamma_1} (1+t)^{1-\gamma_1}\right) + \mathsf{C}_1(1+t)^{-\gamma_2} \exp\left(\frac{\mathsf{C}_1}{1-\gamma_1} (1+t)^{1-\gamma_1}\right)} \\
&= \lim_{t\to\infty} \frac{\mathsf{C}_2}{(\gamma_1-\gamma_2)(1+t)^{\gamma_1-1} + \mathsf{C}_1} \\
&= \frac{\mathsf{C}_2}{\mathsf{C}_1}.
\end{align*} 
We now show that the first term is signficantly smaller than the second term. To do so, we note the inequality $e^x\geq (1/n!) x^n$ for any $x \geq 0$. Therefore, by choosing $n = 1 + \lceil (\gamma_2-\gamma_1)/(1-\gamma_1) \rceil$, we have $(1-\gamma_1) n > \gamma_2-\gamma_1$, so
\begin{align*}
0\leq \lim_{t\to\infty} \frac{(1+t)^{\gamma_2-\gamma_1} Y_0}{\exp\left(\frac{\mathsf{C}_1}{1-\gamma_1} (1+t)^{1-\gamma_1} \right)} 
&\leq \lim_{t\to\infty} \frac{n! (1-\gamma_1)^n Y_0}{\mathsf{C}_1^n (1+t)^{n(1-\gamma_1) - (\gamma_2-\gamma_1)}} = 0.
\end{align*}
Therefore
\begin{equation}
Z_t \asymp \frac{1}{(1+t)^{\gamma_2-\gamma_1}}.
\end{equation}

\paragraph{Case $\gamma_1=1$:} This admits the solution
\begin{align}
Z_t &= \frac{Z_0}{(1+t)^{\mathsf{C}_1}} + \frac{\int_0^t \mathsf{C}_2 (1+s)^{\mathsf{C}_1-\gamma_2}}{(1+t)^{\mathsf{C}_1}} \nonumber \\
&= \begin{cases} \displaystyle{\frac{Y_0}{(1+t)^{\mathsf{C}_1}} + \frac{\mathsf{C}_2}{\mathsf{C}_1-\gamma_2+1} \frac{1}{(1+t)^{\gamma_2-1}}} & \mathsf{C}_1-\gamma_2 \neq -1 \\ \displaystyle{\frac{Y_0}{(1+t)^{\mathsf{C}_1}} + \frac{\mathsf{C}_2}{\mathsf{C}_1-\gamma_2+1} \frac{\log(1+t)}{(1+t)^{\mathsf{C}_1}}} & \mathsf{C}_1-\gamma_2 = 1
\end{cases}.
\end{align}
The asymptotics thus follows.

\paragraph{Comparison Principle} Finally, we establish a comparison principle between $Y_t$ and $Z_t$. We note that
\begin{align*}
\frac{d(Y_t-Z_t)}{dt} \leq -\frac{\mathsf{C}_1}{(1+t)^{\gamma_1}} (Y_t-Z_t), \quad Y_0-Z_0=0.
\end{align*}
Therefore, using an integrating factor,
\begin{align*}
0 \geq \begin{cases} \displaystyle{\frac{d}{dt} \left[(Y_t-Z_t) \exp\left(\frac{\mathsf{C}_1}{1-\gamma_1} (1+t)^{1-\gamma_1} \right) \right]} & \gamma_1 \in (0,1) \\ \displaystyle{\frac{d}{dt} \left[(Y_t-Z_t) (1+t) \right]} & \gamma_1 =1 \end{cases},
\end{align*}
and thus
\begin{align*}
Y_t - Z_t &\leq \begin{cases} (Y_0-Z_0) \exp\left(-\frac{\mathsf{C}_1}{1-\gamma_1} (1+t)^{1-\gamma_1} \right) & \gamma_1 \in (0,1) \\ (Y_0-Z_0) (1+t)^{-1} & \gamma_1 = 1 \end{cases} \\
&= 0.
\end{align*}
As $Y_0 \geq 0$, the desired asymptotics thus followed.
\end{proof}

\subsection{Case when \texorpdfstring{$\eta_t = (1+(\log(1+t))^2)^{-1}$}{}} \label{SS:case_when_eta_is_log}

We first establish a comparison lemma, similar to \ref{lem:comparison_new}.

\begin{lemma}[Comparison Lemma] \label{lem:comparison}
Consider a differentiable map $t \in [0,\infty) \mapsto Y_t$. Suppose $Y_t \geq 0$, and there exist constants $n_0, \mathsf{C}_1, \mathsf{C}_2 > 0$ such that for all $t \geq 0$,
\begin{equation}
\frac{dY_t}{dt} \leq -\mathsf{C}_1 \eta_t^{n_0} Y_t + \mathsf{C}_2 \eta_t^2,
\end{equation}
where, as defined in \eqref{eq:learning rates},
\begin{equation*}
\eta_t = \frac{1}{1+\log^2(1+t)}.
\end{equation*}
then 
\begin{equation}
Y_t = O\left(\frac{\zeta_t}{\eta_t^{n_0-2}}\right) = O\left(\frac{(1+(\log(1+t))^2)^{n_0-2}}{1+t} \right).
\end{equation}
\end{lemma}

\begin{proof}
Consider the differential equation
\begin{equation}
\frac{dZ_t}{dt} = -\mathsf{C}_1 \eta^{n_0}_t Z_t + \mathsf{C}_2 \eta_t^2, \quad Z_0=Y_0.
\end{equation}
This admits the solution 
\begin{equation}
Z_t = \frac{Y_0}{\exp\left(\int_0^t \mathsf{C}_1 \eta_s^{n_0} \, ds\right)} + \frac{\int_0^t \mathsf{C}_2 \eta_t^2 \exp\left(\int_0^s \mathsf{C}_1 \eta_\tau^{n_0} \, d\tau\right) \, ds}{\exp\left(\int_0^t \mathsf{C}_1 \eta_s^{n_0} \, ds\right)}
\end{equation}
As there exists $t_0 \geq 0$ such that for any $t \geq t_0$, $\eta^2_t \geq 1/t$, it follows that both $\int_0^t \mathsf{C}_1 \eta_s^{n_0} \, ds$ and $\exp\left(\int_0^t \mathsf{C}_1 \eta_s^{n_0} \, ds\right)$ diverges to $+\infty$ as $t \to \infty$. Therefore, we can use L'H\^{o}pital's rule to establish our asymptotics. In particular, we first have the following preliminary asymptotics:
\begin{align}
\lim_{t\to\infty} \frac{1}{(1+t) \eta^{n_0}_t} \int_0^t \mathsf{C}_1  \eta_s^{n_0} \, ds
&= \lim_{t\to\infty} \frac{\int_0^t \mathsf{C}_1 \eta_s^{n_0} \, ds}{(1+t) (1+(\log(1+t))^2)^{-n_0}}  \nonumber \\
&= \lim_{t\to\infty} \frac{\mathsf{C}_1 \eta_t^{n_0}}{(1+(\log(1+t))^2)^{-n_0} - 2n_0 (1+(\log(1+t))^2)^{-n_0-1} \log(1+t)}  \nonumber \\
&= \lim_{t\to\infty} \frac{\mathsf{C}_1}{1-\frac{2n_0 \log(1+t)}{1+(\log(1+t))^2}} = \mathsf{C}_1.
\end{align}
This means
\begin{align}
\int_0^t \mathsf{C}_1 \eta_s^{n_0} \asymp (1+t) \eta_t^{n_0} = \frac{1+t}{(1+(\log(1+t))^2)^{n_0}}.
\end{align}
This enables us to compute the asymptotics of the second term:
\begin{align*}
\lim_{t\to\infty} \frac{\eta_t^{n_0-2} \int_0^t \mathsf{C}_2 \eta_s^2 \exp\left(\int_0^s \mathsf{C}_1  \eta_\tau^{n_0} \, d\tau\right) \, ds}{\zeta_t \exp\left(\int_0^t \mathsf{C}_1 \eta_s^{n_0} \, ds\right)} &= \lim_{t\to\infty} \frac{\int_0^t \mathsf{C}_2 \eta_s^2 \exp\left(\int_0^s \mathsf{C}_1  \eta_\tau^{n_0} \, d\tau\right) \, ds}{(1+t)^{-1} (1+(\log(1+t))^2)^{n_0-2} \exp\left(\int_0^t \mathsf{C}_1 \eta_s^{n_0} \, ds\right)}.
\end{align*}
Note that the numerator satisfies
\begin{align*}
\int_0^t \mathsf{C}_2 \eta_s^2 \exp\left(\int_0^s \mathsf{C}_1 \eta_\tau^{n_0} \, d\tau \right) \, ds \geq \int_0^t\mathsf{C}_2 \eta_s^2 \, ds \overset{n\to\infty}\to +\infty, 
\end{align*}
and the denominator satisfies
\begin{align}
&\phantom{=}\frac{(1+(\log(1+t))^2)^{n_0-2}}{1+t} \exp\left(\int_0^t \mathsf{C}_1 \eta_\tau^{n_0} \, d\tau \right) \nonumber \\
&\geq \frac{(1+(\log(1+t))^2)^{n_0-2}}{1+t} \left(1+\int_0^t \mathsf{C}_1 \eta_\tau^{n_0} \, d\tau + \frac{1}{2} \left(\int_0^t \mathsf{C}_1 \eta_\tau^{n_0} \, d\tau \right)^2 \right) \nonumber \\
&\geq \frac{(1+(\log(1+t))^2)^{n_0-2}}{2(1+t)}  \left(\int_0^t \mathsf{C}_1 \eta_\tau^{n_0} \, d\tau \right)^2 \nonumber \\
&\asymp (1+t) \eta_t^{n_0+2} \overset{t\to\infty}\to \infty. \label{eq:second_term_denom}
\end{align}
Therefore, we can use L'H\^{o}pital's rule to obtain the asymtotics for the second term. For the case $n_0 \neq 2$, we have
\begin{align*}
&\phantom{=} \lim_{t\to\infty} \frac{\int_0^t \mathsf{C}_2 \eta_s^2 \exp\left(\int_0^s \mathsf{C}_1  \eta_\tau^{n_0} \, d\tau\right) \, ds}{(1+t)^{-1} (1+(\log(1+t))^2)^{n_0-2} \exp\left(\int_0^t \mathsf{C}_1 \eta_s^{n_0} \, ds\right)} \\
&= \lim_{t\to\infty} \frac{\mathsf{C}_2 \eta_t^2 \exp\left(\int_0^t \mathsf{C}_1 \eta_s^{n_0} \, d\tau\right)}{\exp\left(\int_0^t \mathsf{C}_1 \eta_s^{n_0} \, ds\right) \left[-\frac{(1+\log(1+t))^2)^{n_0-2}}{(1+t)^2} + \frac{2(n_0-2)(1+\log(1+t))^2)^{n_0-3}}{(1+t)^2} + \frac{(1+\log(1+t))^2)^{n_0-2}}{1+t} \int_0^t \mathsf{C}_1 \eta_s^{n_0} \, ds\right] } \\
&= \frac{\mathsf{C}_2}{\lim_{t\to\infty} \left[-\frac{(1+\log(1+t))^2)^{n_0}}{(1+t)^2} + \frac{2(n_0-2)(1+\log(1+t))^2)^{n_0-1}}{(1+t)^2} + \frac{(1+\log(1+t))^2)^{n_0}}{1+t} \int_0^t \mathsf{C}_1 \eta_s^{n_0} \, ds\right]} \\
&= \frac{\mathsf{C}_2}{\mathsf{C}_1};
\end{align*} 
and for the case $n_0 = 2$, we have
\begin{align*}
&\phantom{=} \lim_{t\to\infty} \frac{\int_0^t \mathsf{C}_2 \eta_s^2 \exp\left(\int_0^s \mathsf{C}_1  \eta_\tau^{n_0} \, d\tau\right) \, ds}{(1+t)^{-1} (1+(\log(1+t))^2)^{n_0-2} \exp\left(\int_0^t \mathsf{C}_1 \eta_s^{n_0} \, ds\right)} \\
&= \lim_{t\to\infty} \frac{\int_0^t \mathsf{C}_2 \eta_s^2 \exp\left(\int_0^s \mathsf{C}_1 \eta_\tau^2 \, d\tau\right) \, ds}{(1+t)^{-1} \exp\left(\int_0^t \mathsf{C}_1 \eta_s^2 \, ds\right)} \\
&= \lim_{t\to\infty} \frac{\mathsf{C}_2 \eta_t^2 \exp\left(\int_0^t \mathsf{C}_1 \eta_s^2 \, d\tau\right)}{\exp\left(\int_0^t \mathsf{C}_1 \eta_s^2 \, ds\right) \left[-\frac{1}{(1+t)^2} + \frac{1}{1+t} \int_0^t \mathsf{C}_1 \eta_s^2 \, ds\right] } \\
&= \frac{\mathsf{C}_2}{\lim_{t\to\infty} \left[-\frac{(1+\log(1+t))^2)^2}{(1+t)^2} +  \frac{(1+\log(1+t))^2)^2}{1+t} \int_0^t \mathsf{C}_1 \eta_s^2 \, ds\right]} \\
&= \frac{\mathsf{C}_2}{\mathsf{C}_1}.
\end{align*}
Therefore, the second term admits the following asymptotic:
\begin{equation}
\frac{\int_0^t \mathsf{C}_2 \eta_s^2 \exp\left(\int_0^s \mathsf{C}_1  \eta_\tau^{n_0} \, d\tau\right) \, ds}{\exp\left(\int_0^t \mathsf{C}_1 \eta_s^{n_0} \, ds\right)} \asymp \frac{\zeta_t}{\eta_t^{n_0-2}} = \frac{(1+(\log(1+t))^2)^{n_0-2}}{1+t}.
\end{equation}
Finally, for the first term, \eqref{eq:second_term_denom} yields
\begin{align*}
\lim_{t\to\infty} \frac{\eta_t^{n_0-2} Y_0}{\zeta_t \exp\left(\int_0^t \mathsf{C}_1 \eta_s^{n_0} \, ds \right)} 
&\leq \lim_{t\to\infty} \frac{2Y_0}{\mathsf{C}_1^2 (1+t) \eta_t^{n_0+2}} = 0,
\end{align*}
so 
\begin{align*}
\frac{Y_0}{\exp\left(\int_0^t \mathsf{C}_1 \eta_s^{n_0} \, ds\right)} = o\left(\frac{\zeta_t}{\eta_t^{n_0-2}}\right).
\end{align*}
Combining yields
\begin{equation}
Z_t \asymp \frac{\zeta_t}{\eta_t^{n_0-2}}.
\end{equation}
Finally, we establish a comparison principle between $Y_t$ and $Z_t$. We note that
\begin{align*}
\frac{d(Y_t-Z_t)}{dt} \leq -\mathsf{C}_1 \eta_t^{n_0} (Y_t-Z_t), \quad Y_0-Z_0=0.
\end{align*}
Therefore, using an integrating factor,
\begin{align*}
\frac{d}{dt} \left[(Y_t-Z_t) \exp\left(\int_0^t \mathsf{C}_1 \eta_s^{n_0} \, ds \right)\right] \leq 0,
\end{align*}
and thus
\begin{align*}
Y_t-Z_t \leq (Y_0-Z_0) \exp\left(-\int_0^t \mathsf{C}_1 \eta_s^{n_0} \, ds \right) = 0.
\end{align*}
As $Y_0 \geq 0$, so 
\begin{equation*}
Y_t = O\left(\frac{\zeta_t}{\eta_t^{n_0-2}}\right) = O\left(\frac{(1+(\log(1+t))^2)^{n_0-2}}{1+t} \right)
\end{equation*}
as desired.
\end{proof}

\begin{remark}
We note that this is a strict improvement from the bound in Lemma 5.2 of \cite{wang2021global}.
\end{remark}

Recall that both the value functions $V^{g_t}$ and $V^{f_t}$ satisfies the Bellman equation
\begin{align}
r(x,a) + \gamma \sum_{z, a''} V^{g_t}(z,a'') g_t(z, a'') p(z| x,a) - V^{g_t}(x,a) &= 0. \\
r(x,a) + \gamma \sum_{z, a''} V^{f_t}(z,a'') f_t(z, a'') p(z| x,a) - V^{f_t}(x,a) &=0. \label{eq:Bellman equation}
\end{align}
Define the difference
\begin{equation}
\phi_t = Q_t - V^{f_t}. 
\end{equation}
Without loss of generality, we initialise the ODE as $Q_0 = 0$. We can then finish the proof for the convergence for the critic.

\begin{proof}[Proof of Theorem \ref{thm:global_critic_convergence}]
Rearranging the Bellman equation \eqref{eq:Bellman equation} gives:
\begin{align}
r(x,a) = -\gamma \sum_{z, a''} V^{f_t}(z,a'') f_t(z, a'') p(z| x,a) + V^{f_t}(x,a). \label{eq:Bellman equation_arr}
\end{align}
Substituting this into \eqref{NN gradient flow} gives:
\begin{align*}
\frac{d\phi_t}{dt}(\xi) 
&= \alpha \sum_{\xi'=(x',a')} A_{\xi,\xi'} \pi^{g_t}(\xi') \left[r(\xi') - Q_t(\xi') + \gamma \sum_{z, a''} Q_t(z, a'') g_t(z,a'') p(z|\xi') \right] - \frac{d}{dt} V^{g_t}(\xi) \\  
&=\alpha \sum_{\xi'=(x',a')} A_{\xi,\xi'} \pi^{g_t}(\xi') \left[-\phi_t(\xi') + \gamma \sum_{z, a''} \Big( Q_t(z, a'') g_t(z,a'') - V^{f_t}(z,a'') f_t(z,a'') \Big) p(z|\xi') \right] - \frac{d}{dt} V^{g_t}(\xi) \\
&=\alpha \sum_{\xi'=(x',a')} A_{\xi,\xi'} \pi^{g_t}(\xi') \Bigg[-\phi_t(\xi') + \gamma \sum_{z, a''} \phi_t(z, a'') g_t(z,a'') p(z|\xi')  \\
&\phantom{=}+ \gamma \sum_{z, a''} V^{f_t}(z, a'') \big(g_t(z,a'') - f_t(z,a'') \big)  p(z|\xi') \Bigg] - \frac{d}{dt} V^{g_t}(\xi) \\
&=\alpha \sum_{\xi'=(x',a')} A_{\xi,\xi'} \pi^{g_t}(\xi') \Bigg[-\phi_t(\xi') + \gamma \sum_{z, a''} \phi_t(z, a'') g_t(z,a'') p(z|\xi')  \\
&\phantom{=}+ \gamma \eta_t \sum_{z, a''} V^{f_t}(z, a'') \left(\frac{1}{\#\bm{\A}} - f_t(z,a'') \right) p(z|\xi') \Bigg] - \frac{d}{dt} V^{g_t}(\xi).
\end{align*}
Rearranging the Bellman equation \eqref{eq:Bellman equation} again gives:
\begin{equation}
\gamma \sum_{z,a''} V^{f_t}(z,a'') f_t(z,a'') p(z|x,a) = V^{f_t}(x,a) - r(x,a),
\end{equation}
so
\begin{align}
\frac{d\phi_t}{dt}(\xi) 
&=\alpha \sum_{\xi'=(x',a')} A_{\xi,\xi'} \pi^{g_t}(\xi') \Bigg[-\phi_t(\xi') + \gamma \sum_{z, a''} \phi_t(z, a'') g_t(z,a'') p(z|\xi') \nonumber \\
&\phantom{=}+ \gamma \eta_t \sum_{z, a''} V^{f_t}(z, a'') \frac{p(z|\xi')}{\#\bm{\A}}  - \eta_t (V^{f_t}(\xi')-r(\xi')) \Bigg] - \frac{d}{dt} V^{g_t}(\xi). \label{eq:dphi_dt}
\end{align}
Define the process
\begin{equation}
Y_t =  \frac{1}{2} \phi_t^{\top} A^{-1} \phi_t,
\end{equation}
then
\begin{align}
\frac{dY_t}{dt} &= - \alpha \sum_{\xi'} \pi^{g_t}(\xi') (\phi_t(\xi'))^2 + \alpha \gamma \sum_{\xi'} \pi^{g_t}(\xi') \phi_t(\xi')\sum_{z, a''} \phi_t(z, a'') g_t(z, a'') p(z | \xi') \nonumber  \\
&\phantom{=}+ \, \alpha \gamma \eta_t \sum_{\xi'} \pi^{g_t}(\xi') \phi_t(\xi') \sum_{z, a''} V^{f_t}(z, a'') \frac{p(z | \xi')}{\#\bm{\A}} + \alpha \eta_t \sum_{\xi'} \pi^{g_t}(\xi') \phi_t(\xi') (r(\xi') - V^{f_t}(\xi)) \nonumber \\
&\phantom{=}- \phi_t^{\top} A^{-1} \frac{dV^{g_t}}{dt}. \label{eq:Yeqn0}
\end{align}
We shall now bound different terms in \eqref{eq:Yeqn0}:

\paragraph{Second Term} The absolute value of this term could be bounded as follows.
\begin{equation}
\begin{aligned}
&\bigg{|} \sum_{x', a'} \phi_t(x', a') \pi^{g_t}(x', a') \sum_{z, a''} \phi_t(z, a'') g_t(z, a'') p(z | x', a') \bigg{|} \notag \\
=& \bigg{|} \sum_{x', a'} \sum_{z, a''} \phi_t(z, a'')  \phi_t(x', a')  g_t(z, a'') p(z | x', a') \pi^{g_t}(x', a')  \bigg{|} \notag \\
\leq& \sum_{x', a'} \sum_{z, a''}  \bigg{|} \phi_t(z, a'')  \phi_t(x', a') \bigg{|}  g_t(z, a'') p(z | x', a') \pi^{g_t}(x', a')  \notag \\
\leq& \frac{1}{2}  \sum_{x', a'} \sum_{z, a''} \bigg{(} \phi_t(z, a'')^2 +  \phi_t(x', a')^2 \bigg{)}  g_t(z, a'') p(z | x', a') \pi^{g_t}(x', a') \notag \\
=& \frac{1}{2}  \sum_{z, a''}  \phi_t(z, a'')^2  \sum_{x', a'}   g_t(z, a'') p(z | x', a') \pi^{g_t}(x', a') + \frac{1}{2}  \sum_{x', a'} \phi_t(x', a')^2 \pi^{g_t}(x', a')  \sum_{z, a''}    g_t(z, a'') p(z | x', a') \notag \\
=&   \frac{1}{2}  \sum_{z, a''}  \phi_t(z, a'')^2 \pi^{g_t}(z, a'')   + \frac{1}{2}  \sum_{x', a'} \phi_t(x', a')^2 \pi^{g_t}(x', a') \notag \\
=&  \sum_{x', a'} \phi_t(x', a')^2 \pi^{g_t}(x', a')
\end{aligned}
\end{equation}
where we have used Young's inequality, the fact that $\displaystyle \sum_{z, a''}    g_t(z, a'') p(z | x', a')  = 1$ for each $(x', a')$, and $\displaystyle \sum_{x', a'}   g_t(z, a'') p(z | x', a') \pi^{g_t}(x', a')  = \pi^{g_t}(z, a'')$.

\paragraph{Third Term} The absolute value of this term could be bounded as follows.
\begin{align}
&\phantom{=}\alpha \gamma \eta_t \left|\sum_{\xi'} \pi^{g_t}(\xi') \phi_t(\xi') \sum_{z, a''} V^{f_t}(z, a'') \frac{p(z | \xi')}{\#\bm{\A}} \right| \nonumber \\
&\leq \alpha \gamma \eta_t \left(\sum_{\xi'} \pi^{g_t}(\xi') (\phi_t(\xi'))^2 \right)^{1/2} \left(\sum_{\xi'} \pi^{g_t}(\xi') \left(\sum_{z,a''} V^{f_t}(z,a) \frac{p(z|\xi')}{\#\bm{\A}} \right)^2\right)^{1/2} \nonumber \\
&\leq \alpha \gamma \eta_t \left(\sum_{\xi'} \pi^{g_t}(\xi') (\phi_t(\xi'))^2 \right)^{1/2} \left(\sum_{\xi'} \pi^{g_t}(\xi') \sum_{z,a''} (V^{f_t}(z,a))^2 \frac{p(z|\xi')}{\#\bm{\A}}\right)^{1/2} \nonumber \\
&\leq \frac{\alpha \gamma}{1-\gamma} \eta_t \left(\sum_{\xi'} \pi^{g_t}(\xi') (\phi_t(\xi'))^2 \right)^{1/2} \\
&\leq \frac{\alpha(1-\gamma)}{4} \sum_{\xi'} \pi^{g_t}(\xi') (\phi_t(\xi'))^2 + \frac{\alpha\gamma}{(1-\gamma)^2} \eta_t^2.
\end{align}

\paragraph{Fourth Term}
By Cauchy-Schwarz, we have:
\begin{align}
&\phantom{=}\alpha \eta_t \left|\sum_{\xi'} \pi^{g_t}(\xi') \phi_t(\xi') (r(\xi')-V^{f_t}(\xi')) \right| \nonumber \\
&\leq \alpha \eta_t \left(\sum_{\xi'} \pi^{g_t}(\xi') (\phi_t(\xi'))^2 \right)^{1/2} \left(\sum_{\xi'} \pi^{g_t}(\xi') (r(\xi') - V^{f_t}(\xi'))^2\right)^{1/2} \nonumber \\
&\leq \frac{\alpha \sqrt{2} (2-\gamma)}{1-\gamma} \eta_t \left(\sum_{\xi'} \pi^{g_t}(\xi') (\phi_t(\xi'))^2 \right)^{1/2} \\
&\leq \frac{\alpha(1-\gamma)}{4} \sum_{\xi'} \pi^{g_t}(\xi') (\phi_t(\xi'))^2 + \frac{2\alpha(2-\gamma)}{(1-\gamma)^2} \eta_t^2.
\end{align}

So far, we have
\begin{equation}
\frac{dY_t}{dt} \leq -\frac{\alpha(1-\gamma)}{2} \sum_{\xi} \pi^{g_t}(\xi) (\phi_t(\xi))^2 + \frac{\alpha(4-\gamma)}{(1-\gamma)^2} \eta_t^2 - \phi_t^{\top} A^{-1} \frac{dV^{g_t}}{dt}.
\end{equation}
As \begin{equation*}
\min_{\xi} \pi^{g_t}(\xi) \geq C_{\mathsf{1}} \eta_t^{n_0},
\end{equation*}
where the constants $C_{\mathsf{1}}, n_0 > 0$ are given in the Proposition \ref{prop:geometric2}, so
\begin{align}
\frac{dY_t}{dt} 
&\leq -\frac{\alpha(1-\gamma) C_{\mathsf{1}}}{2} \eta_t^{n_0} \sum_{\xi} (\phi_t(\xi))^2 + \frac{\alpha(4-\gamma)}{(1-\gamma)^2} \eta_t^2 - \phi_t^{\top} A^{-1} \frac{dV^{g_t}}{dt}.
\label{eq:Yeqn}
\end{align}

\paragraph{Final Term} We shall now evaluate $dV^{g_t}/dt$ using Policy Gradient Theorem. To justify the use of this theorem, we shall first show that $dg_t/dt$ is uniformly bounded in $t$. To begin, note that
\begin{align*}
\frac{d\eta_t}{dt} = -\frac{2 \log (1+t)}{(1+t) (1+ (\log(1+t))^2 )^2},
\end{align*}
so 
\begin{align*}
\abs{\frac{d\eta_t}{dt}} \leq \frac{2}{(1+t)(1+\log(1+t))} \leq 2\zeta_t.
\end{align*}
Therefore,
\begin{align*}
\frac{df_t}{dt}(x,a) &= -\frac{e^{P_t(x,a)} \sum_{a'} e^{P_t(x,a')} \frac{dP_t(x,a')}{dt}}{\left(\sum_{a''} e^{P_t(x,a'')}\right)^2} + \frac{e^{P_t(x,a)} \frac{dP_t(x,a')}{dt}}{\sum_{a'} e^{P_t(x,a')}} \\
&= -f_t(x,a) \sum_{a'} f_t(x,a') \frac{dP_t}{dt}(x,a') + f_t(x,a) \frac{dP_t}{dt}(x,a),
\end{align*}
and, by the limit ODEs in \eqref{NN gradient flow}, it holds for any $(x,a)$:
\begin{equation}
\label{theta uniform bound}
\left| \frac{dP_t}{dt}(x,a) \right| 
= \left| \sum_{x', a'} \zeta_t \clip(Q_t(x', a')) \left[A_{x, a, x', a'} - \sum_{a''} f_t(x',a'')  A_{x,a,x',a''} \right] \sigma^{f_t}(x', a')\right| \le C\zeta_t.
\end{equation}
Therefore, we have
\begin{align*}
\abs{\frac{df_t}{dt}(\xi)} \leq f_t(x,a) \sum_{a'} f_t(x,a') C\zeta_t + Cf_t(x,a) \zeta_t \leq C\zeta_t.
\end{align*}
As
\begin{align*}
\frac{dg_t}{dt} = \left(\frac{1}{\#\bm{\A}} - f_t(\xi) \right) \frac{d\eta_t}{dt} - \frac{df_t}{dt}(\xi),
\end{align*}
so for all $\xi$, we have
\begin{align*}
\left|\frac{dg_t}{dt}(\xi) \right| \leq C\zeta_t.
\end{align*}
We can therefore invoke the Policy Gradient Theorem, which, in our case, reads
\begin{align*}
\frac{d\bar{V}^{g_t}}{dt}(x_0) = \sum_{\xi=(x,a)} \nu^{g_t}_{x_0}(x) \frac{dg_t}{dt}(\xi) V^{g_t}(\xi),
\end{align*}
and, by differentiating the Bellman equation,
\begin{align*}
\frac{dV^{g_t}(\xi)}{dt} 
&= \gamma \sum_{x'} \frac{d\bar{V}^{g_t}(x')}{dt} p(x'|\xi) \\
&= \gamma \sum_{x', \xi'' = (x'',a'')} \nu^{g_t}_{x'}(x'') \frac{dg_t}{dt}(\xi'') V^{g_t}(\xi'') p(x'|\xi),
\end{align*}
so
\begin{equation}
\abs{\frac{dV^{g_t}}{dt}(\xi)} \leq C\frac{\gamma}{1-\gamma} \zeta_t.
\end{equation}
Therefore,
\begin{align}
\phi_t^\top A^{-1} \frac{dV^{g_t}}{dt} 
&\leq \frac{CM}{\sigma_{\min}(A)} \|\phi_t\| \zeta_t \nonumber \\
&= \left(\left(\frac{\alpha(1-\gamma) C_{\mathsf{1}}}{2} \right)^{1/2} \eta^{n_0/2}_t\right) \left(\frac{\sqrt{2} CM}{(\alpha (1-\gamma) C_{\mathsf{1}})^{1/2} \sigma_{\min}(A)} \frac{\zeta_t}{\eta^{n_0/2}_t} \|\phi_t\| \right) \nonumber \\
&\leq \frac{\alpha(1-\gamma) C_{\mathsf{1}}}{4} \eta^{n_0}_t \|\phi_t\|^2 + \frac{C^2 M^2}{\alpha (1-\gamma) C_{\mathsf{1}} \sigma_{\min}^{2}(A)} \frac{\zeta^2_t}{\eta^{n_0}_t}. \label{eq:zetaBound}
\end{align}
Combining \eqref{eq:Yeqn} with \eqref{eq:zetaBound}, we have 
\begin{equation}
\frac{dY_t}{dt} \leq -\frac{\alpha(1-\gamma) C_{\mathsf{1}}}{4} \eta_t^{n_0} \|\phi_t\|^2 + \frac{\alpha(4-\gamma)}{(1-\gamma)^2} \eta_t^2 + \frac{C^2 M^2}{\alpha (1-\gamma) C_{\mathsf{1}} \sigma_{\min}^2(A)} \frac{\zeta^2_t}{\eta^{n_0}_t}.
\end{equation}
Since $A$ is positive definite, its smallest eigenvalue $\sigma_{\min}(A) > 0$, and
\begin{equation} \label{eq:relate_phi_with_Y}
\frac{1}{\sigma_{\max}(A)} \|\phi_t\|^2 \leq Y_t \leq \frac{1}{\sigma_{\min}(A)} \|\phi_t\|^2,
\end{equation}
where $\sigma_{\max}$ is the largest eigenvalue of $A$. Hence,
\begin{equation*} 
-\|\phi_t\|^2 \leq -\sigma_{\min}(A) Y_t,
\end{equation*}
and
\begin{equation}
\frac{dY_t}{dt} \leq -\frac{\alpha(1-\gamma) C_{\mathsf{1}} \sigma_{\min}(A)}{4} \eta_t^{n_0} Y_t + \frac{\alpha(4-\gamma)}{(1-\gamma)^2} \eta_t^2 + \frac{C^2 M^2}{\alpha (1-\gamma) C_{\mathsf{1}} \sigma_{\min}^2(A)} \frac{\zeta^2_t}{\eta^{n_0}_t} .\label{eq:critic}
\end{equation}
Finally, note that there exists a constant $c$ such that $\eta_t^2 \geq \zeta_t^2/\eta_t^{n_0}$ for any $t \geq 0$, so it is true that
\begin{equation}
\frac{dY_t}{dt} \leq -\mathsf{C}_1 \eta_t^{n_0} Y_t + \mathsf{C}_2 \eta_t^2, \quad t \geq 0,
\end{equation}
where
\begin{equation*}
\mathsf{C_1} = \frac{\alpha(1-\gamma)C_{\mathsf{1}} \sigma_{\min}(A)}{4}, \quad \mathsf{C}_2= \frac{c\alpha(4-\gamma)}{(1-\gamma)^2} \vee \frac{C^2 M^2}{\alpha (1-\gamma) C_{\mathsf{1}} \sigma^2_{\min}(A)}.
\end{equation*}
Therefore, by the comparison lemma (Lemma \ref{lem:comparison}) and \eqref{eq:relate_phi_with_Y}, there exists constant $C \geq 0$ such that 
\begin{equation}
\|\phi_t\|^2 \leq C \frac{(1+(\log(1+t))^2)^{n_0}}{1+t}.
\end{equation}
\end{proof}

\begin{proof}[Proof of Theorem \ref{thm:global_actor_convergence_rate}]
Follow from the original calculations, we know that there exists $T_0 \geq 0$ such that for any $T\geq T_0$,
\begin{equation}
\frac{d}{dt} J(f_t) \geq \zeta_t \sigma^2_{\min}(A) \|\nabla_P J(f_t)\|^2 - \zeta_t \|\nabla_P J(f_t)\| \Big[\sigma^2_{\max}(A) \|\widehat{\nabla}_P J(f_t) - \nabla_P J(f_t)\| + \|E_t\| \Big]. \label{eq:pre-lower-bound-new}
\end{equation}
Recall that 
\begin{equation}
\partial_{x,a} J(f_t) - \widehat{\partial}_{x,a} J(f_t) = \sigma^{f_t}_{\rho_0}(x,a) \left[[Q_t - V^{f_t}](x,a) - \sum_{a''} [Q_t - V^{f_t}](x,a'')\right],
\end{equation}
so there are constant $C>0$ (changing line by line) such that
$$\|\nabla_P J(f_t) - \widehat{\nabla}_P J(f_t) \| \leq C\|V^{f_t} - Q_t\| \leq \frac{C}{\eta_t^{n_0/2-1} \sqrt{1+t}} \leq C\eta_t.$$
Furthermore, using statement (2) of Proposition \ref{prop:Lipschitzness_of_stationary_measures}, there are constants $C > 0$ such that
\begin{align*}
|E_t(\xi)| \leq C \eta_t \left[\sum_\xi \left|f(\xi) - \frac{1}{\#\mathcal{A}}\right| \right] \leq C \eta_t,
\end{align*}
so $\|E_t\| \leq C\eta_t$. \\

Substituting \eqref{eq:pre-lower-bound}, we know that there are constant $C>0$ such that
\begin{align*}
\frac{d}{dt} J(f_t) &\geq \zeta_t \sigma^2_{\min}(A) \|\nabla_P J(f_t)\|^2 - C\zeta_t \eta_t \|\nabla_P J(f_t)\|.
\end{align*}
Rearranging and utilising Young's inequality yields
\begin{align*}
\zeta_t \sigma^2_{\min}(A) \|\nabla_P J(f_t)\|^2 &\leq \frac{dJ(f_t)}{dt} + C\zeta_t \eta_t \|\nabla_P J(f_t)\| \\
&\leq \frac{dJ(f_t)}{dt} + \frac{\sigma^2_{\min}(A)}{2} \zeta_t \|\nabla_P J(f_t)\|^2 + \frac{C}{2 \sigma^2_{\min}(A)} \zeta_t \eta^2_t. \\
\implies \frac{\sigma^2_{\min}(A)}{2} \zeta_t \|\nabla_P J(f_t)\|^2 &\leq \frac{dJ(f_t)}{dt} + \frac{C}{2 \sigma^2_{\min}(A)} \zeta_t \eta^2_t \\
\implies \zeta_t \|\nabla_P J(f_t)\|^2 &\leq \frac{2}{\sigma^2_{\min} (A)} \frac{dJ(f_t)}{dt} + \frac{C}{2 \sigma^4_{\min}(A)} \zeta_t \eta^2_t.
\end{align*}
Integrating both sides yields
\begin{align}
\int_0^T\zeta_t \|\nabla_P J(f_t)\|^2 \, dt 
&\leq \int_0^{T_0} \zeta_t \|\nabla_P J(f_t)\|^2 \, dt + \int_{T_0}^T\zeta_t \|\nabla_P J(f_t)\|^2 \, dt  \\
&\leq \frac{2T_0}{1-\gamma} + \frac{J(f_T) - J(f_0)}{2\sigma^2_{\min}(A)} + \frac{C}{2\sigma^4_{\min}(A)} \int_0^T \zeta_t \eta^2_t \, dt\\
&\leq \frac{2T_0 \sigma^2_{\min}(A)+1}{(1-\gamma)\sigma^2_{\min}(A)} + \frac{C}{2\sigma^4_{\min}(A)} \int_0^T \zeta_t \eta_t \, dt \\
&\leq C' < +\infty, \numberthis \label{eq:weighted_time_average_of_squared_norm_2}
\end{align}
where $C'>0$ is another constant, noting that $|J(f)| = |\sum_x \rho_0(x) \bar{V}^f(x)| \leq 2/(1-\gamma)$ by \eqref{eq:trivial_bound}. It follows that the weighted time-average of the squared gradient converges to zero:
\begin{align*}
\frac{\int_0^T\zeta_t \|\nabla_P J(f_t)\|^2 \, dt}{\int_0^T \zeta_t \, dt} &\leq \frac{C'}{\log(1+T)} \overset{T\to+\infty}\to 0.
\end{align*}
Moreover,
\begin{equation}
C' \geq \int_0^T\zeta_t \|\nabla_P J(f_t)\|^2 \, dt \geq \inf_{t \in [0,T]} \|\nabla_P J(f_t) \|^2 \int_0^T \zeta_t \, dt =  \log(1+T) \, \inf_{t \in [0,T]} \|\nabla_P J(f_t) \|^2,
\end{equation}
so
\begin{equation}
\inf_{t \in [0,T]} \|\nabla_P J(f_t) \|^2 \leq \frac{C'}{\log(1+T)},
\end{equation}
and therefore $\displaystyle \liminf_{t\to \infty}\|\nabla_P J(f_t)\| = 0$.
\end{proof}

\begin{proof}[Proof of Theorem \ref{thm:global_actor_convergence}]
To show that $\displaystyle \lim_{t\to \infty}\|\nabla_P J(f_t)\| = 0$, assume the contrary; that is $\displaystyle \limsup_{t\to \infty}\|\nabla_P J(f_t)\| > 0$. Then we can find a constant $\epsilon_1>0$ and two increasing sequences $\{a_n\}_{n\ge 1}, \{b_n\}_{n\ge 1}$ such that 
\begin{align*}
a_1 <b_1 <a_2 <b_2 <a_3 <b_3 < \cdots,\\
\|\nabla_P J(f_{a_n})\| < \frac{\epsilon_1}{2},\quad \|\nabla_P J(f_{b_n})\| > \epsilon_1.
\end{align*}
Define the following cycle of stopping times:
\begin{align}
t_n &:= \sup\{s | s\in (a_n, b_n),\  \|\nabla_P J(f_s)\| < \frac{\epsilon_1}{2} \},\\
i(t_n) &:= \inf\{s | s\in (t_n, b_n),\  \|\nabla_\theta J(f_s)\| > \epsilon_1 \}.
\end{align}
Note that $ \|\nabla_P J(f_t)\| $ is continuous with respect to $t$, so
we have 
\begin{align*}
&a_n \le t_n < i(t_n) \le b_n \\
&\|\nabla_P J(f_{t_n})\| = \frac{\epsilon_1}{2}, \quad \|\nabla_P J(f_{i(t_n)})\| =\epsilon_1\\
&\frac{\epsilon_1}{2} \le \|\nabla_P J(f_s)\| \le \epsilon_1, \quad s\in(t_n, i(t_n)). \numberthis \label{eq:stopping_time_property_new}
\end{align*}
Notice that there is a constant $L > 0$ such that gradient is $L$-Lipschitz, see e.g., the proof of \cite[Lemma 7]{mei2020global}, we have for any $t_n$
\begin{align*}
\frac{\epsilon_1}{2} &=  \|\nabla_P J(f_{i(t_n)})\| - \|\nabla_P J(f_{t_n})\|\\
&\le \| \nabla_P J(f_{i(t_n)}) - \nabla_P J(f_{t_n}) \|\\
&\le L\| P_{i(t_n)} - P_{t_n} \|\\
&\le C \int_{t_n}^{i(t_n)} \zeta_s\| \nabla_P J(f_s)\| ds + C \int_{t_n}^{i(t_n)} \zeta_s \|\widehat{\nabla}_P J(f_s) - \nabla_P J(f_s)\|ds \\
&\le C \epsilon_1 \int_{t_n}^{i(t_n)} \zeta_s ds + C \int_{t_n}^{i(t_n)} \zeta_s \eta_s ds.
\end{align*}
Since $\int_0^\infty \zeta_s \eta_s \, ds < +\infty$, it follows that 
\begin{equation}
\label{eq:key2_new}
\frac{1}{2L} \le \liminf_{n\to \infty} \int_{t_n}^{i(t_n)} \zeta_s ds.
\end{equation}
Using \eqref{eq:stopping_time_property_new}, we see that 
\begin{equation*}
J(f_{\bar\theta_{i(t_n)}}) - J(f_{\bar\theta_{t_n}}) \ge C_1 \left(\frac{\epsilon_1}{2} \right)^2 \int_{t_n}^{i(t_n)} \zeta_s ds - C_2\int_{t_n}^{i(t_n)} \zeta_s\eta_s ds.
\end{equation*}
Due to the convergence of $J(f_{\theta_{t_n}})$ and the assumption of the learning rate, this implies that 
\begin{equation}
\lim_{n\to \infty} \int_{t_n}^{i(t_n)} \zeta_s ds = 0,
\end{equation}
which contradicts Equation \eqref{eq:key2_new}, and thus the convergence to the stationary point is proven.
\end{proof}

\section{Backgrounds on MDP}
\label{S:backgrounds_on_MDP}
In this appendix, we shall formally define when MDP is \textit{communicating} as given in \cite{Puterman2014markov, Kallenberg2002, Kallenberg2021}. We will then prove Propositions \ref{prop:equiv_communication_1} and \ref{prop:equiv_communication_2} on the equivalence between MDP being communicating and Assumption \ref{as:ergodic_assumption}, as well as the consequence on the uniform ergodicity of the induced and auxiliary chains. Finally, we will discuss the Lipschitzness of the stationary distributions with the policy.

\subsection{Communicating and Aperiodicity of MDP}
\label{SS:communicating_MDP}
We begin by noting that an MDP $\bm{\M}$ with policy $f$ defines the following ``marginal" Markov chain on the state space
\begin{equation}
(\bm{\M}, f)_{\avg} : \begin{cases} x_0 \sim \rho_0 \\ 
x_{k+1} \sim \bar{\mathbb{P}}_f(x_k, \cdot) := \sum_{a} p(\cdot \mid x_k, a) f(x_k, a)
\end{cases}.
\end{equation}

Note that this is a ``marginal" version of the induced Markov chain $(\bm{\M}, f)$. Recall that a policy $f$ is \textit{deterministic} if for any $x$, there exists a unique $a(x)$ such that
\begin{equation*}
f(x,a) = \begin{cases} 1 & a=a(x) \\ 0 & \text{otherwise} \end{cases}.
\end{equation*}
We could now define the notion of \textit{communication} for an MDP:

\begin{definition}
An MDP is \textit{communicating} if for any states $x, x' \in \bm{\X}$, there exists a deterministic policy $f := f_{x,x'}$ such that under the marginal Markov chain $(\bm{\M}, f_{x,x'})_{\avg}$, $x'$ is \textit{accessible} by $x$. In other words, there exists $n_0 \in \N$ such that
\begin{equation} \label{eq:communicating_MDP_C1}
0 < \bar{\mathbb{P}}_{f_{x,x'}}^{n_0}(x,x') = \sum_{a_0, x_1, a_1, ..., x_{n_0-1}, a_{n_0-1}} \prod_{\ell=0}^{n_0-1} f_{x,x'}(x_\ell, a_\ell) p(x_{\ell+1} \mid x_\ell, a_\ell),
\end{equation}
where in the product, $x_0$ is to be interpreted as $x$, and $x_{n_0}$ is to be interpreted as $x'$.
\end{definition}

For notational convenience, we shall identify $f_{x,x'}$ by the function $a_{x,x'} : z \in \bm{\X} \mapsto \mathbb{I}(f_{x,x'}(z,\cdot) = 1)$. With this, condition \eqref{eq:communicating_MDP_C1} could then be written as
\begin{equation} \label{eq:communicating_equation}
0 < \bar{\mathbb{P}}_{f_{x,x'}}^{n_0}(x,x') = \sum_{x_1, ..., x_{n_0-1}} \prod_{\ell=0}^{n_0-1}  p(x_{\ell+1} \mid x_\ell, a_{x,x'}(x_\ell)).
\end{equation}

We are now ready to prove Proposition \ref{prop:equiv_communication_1}. In fact, we shall prove the following stronger proposition:

\begin{proposition} \label{prop:C2}
The following are equivalent:
\begin{enumerate}
    \item $\bm{\M}$ is communicating,
    \item $(\bm{\M}, \mathsf{1})_{\avg}$ is irreducible,
    \item[2A.] $(\bm{\M}, \mathsf{1})$ is irreducible,
    \item $(\bm{\M}, f)_{\avg}$ is irreducible for any $f$ such that $f(x,a) \geq f_{\min} > 0$,
    \item[3A.] $(\bm{\M}, f)$ is irreducible for any $f$ such that $f(x,a) \geq f_{\min} > 0$,
    \item there exists a policy $f$ such that $(\bm{\M}, f)_{\avg}$ is irreducible.
\end{enumerate}

Note that Proposition \ref{prop:equiv_communication_1} corresponds to the equivalence $(1) \iff (2A)$.

\begin{proof}
We shall structure the proof as $(1) \implies (2) \implies (3) \implies (4) \implies (1)$, then noting $(2) \iff (2A)$ and $(2A) \iff (3A)$. Note that $(3) \implies (4)$ and $(3A) \implies (2A)$  are trivial. Throughout the proof, $x_0=x$ and $x_{n_0} = x'$. \\

$(1) \implies (2)$: pick $x,x' \in \bm{\X}$, then there is a policy $f_{x,x'}$ (identified by the function $a_{x,x'}(\cdot)$ as above) and $n_0 \in \N$ such that $\bar{\mathbb{P}}^{n_0}_{f_{x,x'}}(x,x') > 0$. Then
\begin{align*}
\bar{\mathbb{P}}^{n_0}_{\mathsf{1}}(x,x') 
&= \frac{1}{(\#\bm{\A})^{n_0}} \sum_{a_0, x_1, a_1, ..., x_{n_0-1}, a_{n_0-1}} \prod_{\ell=0}^{n_0-1} p(x_{\ell+1} \mid x_\ell, a_\ell) \\
&\geq \frac{1}{(\#\bm{\A})^{n_0}} \sum_{x_1, a_1, ..., x_{n_0-1}} \prod_{\ell=0}^{n_0-1} p(x_{\ell+1} \mid x_\ell, a_{x,x'}(x_\ell)) \\
&= \frac{1}{(\#\bm{\A})^{n_0}} \bar{\mathbb{P}}^{n_0}_{f_{x,x'}}(x,x') > 0, 
\end{align*}
so $x'$ is accessible by $x$ in $(\bm{\M}, \mathsf{1})_{\avg}$. \\

$(2) \implies (3)$: pick $x,x' \in \bm{\X}$, then there is an $n_0 \in \N$ such that $\bar{\mathbb{P}}^{n_0}_{\mathsf{1}}(x,x') > 0$. Then
\begin{align*}
\bar{\mathbb{P}}^{n_0}_{f}(x,x') 
&= \frac{1}{(\#\bm{\A})^{n_0}} \sum_{a_0, x_1, a_1, ..., x_{n_0-1}, a_{n_0-1}} \prod_{\ell=0}^{n_0-1} f(x_\ell, a_\ell) \, p(x_{\ell+1} \mid x_\ell, a_\ell) \\
&= \left(\frac{f_{\min}}{\#\bm{\A}} \right)^{n_0} \sum_{a_0, x_1, a_1, ..., x_{n_0-1}, a_{n_0-1}} \prod_{\ell=0}^{n_0-1} p(x_{\ell+1} \mid x_\ell, a_\ell) \\
&= f^{n_0}_{\min} \bar{\mathbb{P}}^{n_0}_{\mathsf{1}}(x,x') > 0, 
\end{align*}
so $x'$ is accessible by $x$ in $(\bm{\M}, f)_{\avg}$. \\

$(4) \implies (1)$: pick $x,x' \in \bm{\X}$, then there is an $n_0 \in \N$ such that $\bar{\mathbb{P}}^{n_0}_f(x,x') > 0$. We shall pick the smallest $n_0$, which exists by well-ordering principle and $n_0 \geq 1$. This implies we could find a sequence $(x_0, a_0, x_1, ..., x_{n_0-1}, a_{n_0-1}, x_{n_0})$ such that
\begin{itemize}
    \item $x_0 = x$ and $x_{n_0} = x'$,
    \item no duplication: if $i,j$ are chosen such that $0 \leq i,j < n_0$ and $i \neq j$ then $x_i \neq x_j$ 
    \item probability that this sequence happening is $>$ 0.
\end{itemize}
Since there are no duplications on $(x_\ell)$, for each $\ell$ we can construct $f_{x,x'}$ by assigning $a_{x,x'}(x_\ell) = a_\ell$. For each $\tilde{x} \neq x_\ell$ we could assign $a_{x,x'}(\tilde{x})$ to any element in $\bm{\A}$. The fact that the sequence happens with positive probability indicates
\begin{align*}
\frac{1}{(\#\bm{\A})^{n_0}} \prod_{\ell=0}^{n_0-1} f(x_\ell, a_\ell) \, p(x_{\ell+1} \mid x_\ell, a_\ell) > 0.
\end{align*}
None of the terms in the product must be zero, which means
\begin{align*}
\bar{\mathbb{P}}_{f_{x,x'}}(x,x') = \prod_{\ell=0}^{n_0-1} p(x_{\ell+1} \mid x_\ell, a_{x,x'}(x_\ell)) = \prod_{\ell=0}^{n_0-1} p(x_{\ell+1} \mid x_\ell, a_\ell) > 0.
\end{align*}

$(2A) \implies (2)$: Given states $x,x'$, it holds for any $a,a'$, there exists $n_0$ (that depends on $x,x',a,a'$) such that
\begin{equation*}
\mathbb{P}^{n_0}_{\mathsf{1}}((x,a), (x',a')) = \frac{1}{(\#\bm{\A})^{n_0}} \sum_{x_1,a_1,...,x_{n_0-1}, a_{n_0-1}} p(x_1\mid x, a) \prod_{\ell=1}^{n_0-1} p(x_{\ell+1} \mid x_\ell, a_\ell),
\end{equation*}
where we let $x_{n_0} = x'$. This implies 
\begin{align*}
\bar{\mathbb{P}}^{n_0}_{\mathsf{1}}(x,x') &= \frac{1}{(\#\bm{\A})^{n_0}} \sum_{a_0,x_1,a_1,...,x_{n_0-1},a_{n_0-1}} \prod_{\ell=0}^{n_0-1} p(x_{\ell+1} \mid x_\ell, a_\ell) \\
&\geq \frac{1}{(\#\bm{\A})^{n_0}} \sum_{x_1,a_1,...,x_{n_0-1},a_{n_0-1}} p(x_1\mid x,a) \prod_{\ell=1}^{n_0-1} p(x_{\ell+1} \mid x_\ell, a_\ell) \\
&= \mathbb{P}^{n_0}_{\mathsf{1}}((x,a), (x',a')) > 0,
\end{align*}
so $x'$ is accessible by $x$ in $(\bm{\M}, \mathsf{1})_{\avg}$. \\

$(2) \implies (2A)$: Consider the state-action pairs $(x,a)$ and $(x',a')$. Select $z$ such that $p(\tilde{x}\mid x,a) > 0$ (which exists as $\sum_{z'} p(z' \mid x,a) = 1$). Note that $x'$ is accessible by $z$ in the chain $(\bm{\M}, \mathsf{1})$, so there exists $n_0 \in \N$ such that $\bar{\mathbb{P}}^{n_0}_{\mathsf{1}}(z,x') > 0$. Then
\begin{align*}
\mathbb{P}^{n_0+1}_{\mathsf{1}}((x,a), (x',a')) &= \frac{1}{(\#\bm{\A})^{n_0+1}}\sum_{x_1, a_1, ..., x_{n_0}, a_{n_0}} p(x_1\mid x,a) \, p(x'\mid x_{n_0}, a_{n_0}) \prod_{\ell=1}^{n_0-1} p(x_{\ell+1} \mid x_\ell, a_\ell) \\
&\geq \frac{p(z\mid x,a)}{\#\bm{\A}} \frac{1}{(\#\bm{\A})^{n_0}} \sum_{a_1, x_2, a_2, ..., x_{n_0}, a_{n_0}} p(x' \mid x_{n_0}, a_{n_0}) \prod_{\ell=1}^{n_0-1} p(x_{\ell+1} \mid x_\ell, a_\ell) \\
&= \frac{p(z\mid x,a)}{\#\bm{\A}} \bar{\mathbb{P}}^{n_0}_{\mathsf{1}}(z,x') > 0.
\end{align*}
Therefore $(x',a')$ is accessible by $(x,a)$ in $(\bm{\M}, \mathsf{1})_{\avg}$. \\

$(2A) \implies (3A)$: Consider the state-action pairs $(x,a)$ and $(x',a')$, then there exists $n_0$ such that $\mathbb{P}^{n_0}_{\mathsf{1}}((x,a), (x',a')) > 0$. It follows that
\begin{align*}
\mathbb{P}^{n_0}_{f}((x,a), (x',a')) 
&= \frac{1}{(\#\bm{\A})^{n_0}} \sum_{x_1, a_1, ..., x_{n_0-1}, a_{n_0-1}} \prod_{\ell=0}^{n_0-1} f(x_{\ell+1}, a_{\ell+1}) \, p(x_{\ell+1} \mid x_\ell, a_\ell) \\
&\geq \left(\frac{f_{\min}}{\#\bm{\A}} \right)^{n_0} \sum_{x_1, a_1, ..., x_{n_0-1}, a_{n_0-1}} \prod_{\ell=0}^{n_0-1} p(x_{\ell+1} \mid x_\ell, a_\ell) \\
&= f_{\min}^{n_0} \, \mathbb{P}^{n_0}_{\mathsf{1}}((x,a),(x',a')) > 0, 
\end{align*}
where we abuse notation and denote $(x_0,a_0) = (x,a)$ and $(x_{n_0}, a_{n_0}) = (x',a')$. Therefore $(x', a')$ is accessible by $(x,a)$ in $(\bm{\M}, f)$.
\end{proof}
\end{proposition}

Next, we study the aperiodicity of an MDP.

\begin{lemma} \label{lem:C3}
$(\bm{\M}, \mathsf{1})$ is aperiodic $\iff$ $(\bm{\M}, f)$ is aperiodic for any fully supported $f$.
\end{lemma}

\begin{proof}
The $\Leftarrow$ is clear. For $\Rightarrow$, let us consider the following subsets of $\N$, defined based on the policy $f$:
\begin{equation*}
R^f_{(x,a)} := \{n \geq 1 \mid \mathbb{P}^n_f((x,a), (x,a)) > 0\}.
\end{equation*}
Note $R^{\mathsf{1}}_{(x,a)} \subseteq R^f_{(x,a)}$. This is shown when proving $(2A) \implies (3A)$ in the proof of Proposition \ref{prop:C2}. As for the other direction, say $n_0 \in R^f_{(x,a)}$, such that
\begin{align*}
\mathbb{P}^{n_0}_{f}((x,a), (x,a)) 
&= \frac{1}{(\#\bm{\A})^{n_0}} \sum_{x_1, a_1, ..., x_{n_0-1}, a_{n_0-1}} \prod_{\ell=0}^{n_0-1} f(x_{\ell+1}, a_{\ell+1}) \, p(x_{\ell+1} \mid x_\ell, a_\ell) > 0
\end{align*}
where we abuse notation and denote $(x_0,a_0) = (x_{n_0}, a_{n_0}) = (x,a)$. Then there is a sequence $(x^*_0=x, a^*_0=a,x^*_1,a^*_1,...,x^*_{n_0-1},a^*_{n_0-1},x^*_{n_0}=x,a^*_{n_0}=a)$ such that
\begin{align*}
\prod_{\ell=0}^{n_0-1} f(x^*_{\ell+1}, a^*_{\ell+1}) \, p(x^*_{\ell+1} \mid x^*_\ell, a^*_\ell) > 0.
\end{align*}
Since $f > 0$ (as it is fully supported), so we must have
\begin{align*}
\prod_{\ell=0}^{n_0-1} p(x^*_{\ell+1} \mid x^*_\ell, a^*_\ell) > 0.
\end{align*}
It follows that
\begin{align*}
\mathbb{P}^{n_0}_{\mathsf{1}}((x,a), (x',a')) &= \left(\frac{1}{\#\bm{\A}} \right)^{n_0} \sum_{x_1, a_1, ..., x_{n_0-1}, a_{n_0-1}} \prod_{\ell=0}^{n_0-1} p(x_{\ell+1} \mid x_\ell, a_\ell) \\
&\geq \left(\frac{1}{\#\bm{\A}} \right)^{n_0} \prod_{\ell=0}^{n_0-1} p(x^*_{\ell+1} \mid x^*_\ell, a^*_\ell) > 0.
\end{align*}
Therefore $n_0 \in R_{\mathsf{1}}$. It follows that $R^{\mathsf{1}}_{(x,a)} = R^f_{(x,a)}$, so $\mathsf{gcd}(R^{\mathsf{1}}_{(x,a)}) = \mathsf{gcd}(R^{f}_{(x,a)})$. This completes the proof.
\end{proof}

Using the same arguments as above, we can also show that if $(\bm{\M}, f)$ is irreducible and aperiodic for any fully supported policies $f$, then so is $(\bm{\M}, f)_{\aux}$. We shall omit the complete proof, but point out the main observation that for any $n$ and state-action pairs $(x,a)$, $(x',a')$, we have
\begin{equation*}
\Pi^{n}_f((x,a), (x',a')) \geq \gamma \mathbb{P}^n_f((x,a), (x',a')).
\end{equation*}
So if $(x',a')$ is accessible by $(x,a)$ in $(\bm{\M}, f)$ then so is the case in $(\bm{\M}, f)_{\aux}$. Moreover, if we define
\begin{equation*}
\tilde{R}^f_{(x,a)} := \{n \geq 1 \mid \Pi^n_f((x,a), (x,a)) > 0\},
\end{equation*}
then $R^f_{(x,a)} \subseteq \tilde{R}^f_{(x,a)}$. It follows that $1 \leq \mathsf{gcd}(\tilde{R}^f_{(x,a)}) \leq \mathsf{gcd}(R^f_{(x,a)})$. Therefore if $(\bm{\M}, f)$ is aperiodic then so is $(\bm{\M}, f)_{\aux}$. \\

So far, we have now shown that if $(\bm{\M}, \mathsf{1})$ is irreducible and aperiodic, then for all fully supported policies $f$, the chains $(\bm{\M}, f)$ and $(\bm{\M}, f)_{\aux}$ are both irreducible and aperiodic, and must admits stationary measures. Recall that we have denote them $\pi^f$ and $\sigma^f_{\rho_0}$ respectively. It remains for us to show that these measures are limiting. 

\begin{proposition}
\label{prop:geometric2}
If $(\bm{\M},\mathsf{1})$ is irreducible and aperiodic and that $f$ is a policy satisfying $f(x,a) \geq f_{\min}$ for any state-action pairs $\xi := (x,a)$, then
\begin{itemize}
\item the following lower (minorisation) bounds hold:
\begin{align}
\min_{x,a} \pi^f(x,a) &\geq C_{\mathsf{1}} (f_{\min} \#\bm{\A})^{n_0} \label{eq:lower_bound_MDP_general}, \\
\min_{x,a} \sigma^f_{\rho_0}(x,a) &\geq C_{\mathsf{1}} (\gamma f_{\min} \#\bm{\A})^{n_0}, \label{eq:lower_bound_general}
\end{align}
where the constants $n_0, C_{\mathsf{1}} > 0$ are defined in Assumption \ref{as:ergodic_assumption}.
\item the chains $(\bm{\M}, f)$ and $(\bm{\M}, f)_{\aux}$ are uniformly ergodic: for any $n$ and $\xi := (x,a)$:
\begin{align}
d_{\TV}(\mathbb{P}^{n}_f(\xi, \cdot), \pi^f) &\leq  (1 - MC_{\mathsf{1}} (f_{\min} \#\bm{\A})^{n_0}))^{\lfloor n/n_0 \rfloor}, \label{eq:geometric_MDP_general} \\
d_{\TV}(\Pi^{n}_f(\xi, \cdot), (1-\gamma) \sigma^f_{\rho_0}(\cdot)) &\leq  (1 - MC_{\mathsf{1}} (\gamma f_{\min} \#\bm{\A})^{n_0})^{\lfloor n/n_0 \rfloor}, \label{eq:geometric_aux_general}
\end{align}
where $M = \#\bm{\X} \times \#\bm{\A}$.
\end{itemize}
\end{proposition}

Note that the convergence rates in \eqref{eq:geometric_MDP_general} and \eqref{eq:geometric_aux_general} both depend on $f_{\min}$.

\begin{proof}
As $\pi^f$ is invariant under $\prob_f$, hence also under $\prob^{n_0}_f$. Therefore, for any $\xi = (x,a)$,
\begin{align*}
\pi^f(\xi) 
&= \sum_{\xi'} \pi^f(\xi') \prob^{n_0}_f(\xi', \xi) \\
&= \sum_{\xi_1, \xi_2, ..., \xi_{n_0-1}, \xi'} \pi^f(\xi') p(x_1|\xi') f(\xi_1) p(x_2|\xi_1) ... p(x|\xi_{n_0-1}) f(\xi) \\
&\geq f_{\min}^{n_0} \sum_{\xi_1, \xi_2, ..., \xi_{n_0-1}, \xi'} \pi^f(\xi') p(x_1|\xi') p(x_2|\xi_1) ... p(x|\xi_{n_0-1}) \\
&\geq C_{\mathsf{1}} (f_{\min} \#\bm{\A})^{n_0} \sum_{\xi'} \pi^{g^N_k}(\xi') = C_{\mathsf{1}} (f_{\min} \#\bm{\A})^{n_0}.
\end{align*}
Using the exact argument, one could also prove that for any $(x,a)$
\begin{equation*}
\min_{x,a} \sigma^f(x,a) \geq C_{\mathsf{1}} \left(\gamma f_{\min} \#\bm{\A} \right)^{n_0}.
\end{equation*}

For the second part of the theorem, let us first define the measure $c\nu_{\mathsf{1}}$ for all constants $c \in (0,1)$ that uniformly assigns $c\nu_{\mathsf{1}}(x,a) = cC_{\mathsf{1}}$. Then, viewing $\mathbb{P}^{n_0}_f$ as a measure, we have the following minorisation inequality
\begin{equation}
\mathbb{P}^{n_0}_f((x,a), \cdot) \geq (f_{\min} \#\bm{\A})^{n_0} \nu_{\mathsf{1}}(\cdot). \label{eq:minorisation_of_Pn0f}
\end{equation}
Therefore, by \cite[Theorem 16.2.4]{meyn2012markov} \eqref{eq:geometric_MDP_general} is true. With similar arguments, we could prove \eqref{eq:geometric_aux_general}.
\end{proof}

The above results imply the following:
\begin{proof}[Proof of Proposition \ref{prop:equiv_communication_2}]
This follows from Proposition \ref{prop:C2}, Lemma \ref{lem:C3} and Proposition \ref{prop:geometric2}.
\end{proof}

\begin{proof}[Proof of Lemma \ref{lem:geometric}]
Note that $\eta_t$ is decreasing, so for any $k \leq \lfloor NT \rfloor$ we have $\eta^N_k \geq \eta^N_{\lfloor NT \rfloor} \geq \eta_T$. Lemma \ref{lem:geometric} therefore follows from Proposition \ref{prop:geometric2} with $f_{\min} = \eta_T$.
\end{proof}

\subsection{Lipschitzness of the stationary measures}
\label{SS:lipschitzness_of_the_stationary_measures}
Consider a finite-state Markov chain with a transition kernel $\mathbf{P}(\xi,\xi')$ that describes the probability of transitioning from $\xi$ to $\xi'$. Following \cite{Hunter2006MarkovChainPerturbations, Ipsen2011ErgodicityCoefficient}, we could define the ergodicity coefficients for the Markov chain:

\begin{definition}[Ergodicity Coefficient, see Corollary 3.9 of \cite{Ipsen2011ErgodicityCoefficient}]
The ergodicity coefficient is defined as
\begin{equation}
\tau_1(\mathbf{P}) = 1-\min_{\xi, \xi'} \sum_{\xi''} \min(\mathbf{P}(\xi,\xi''), \mathbf{P}(\xi',\xi'')).
\end{equation}
\end{definition}

With this, \cite{Ipsen2011ErgodicityCoefficient} proves the following:

\begin{theorem}[Sensitivity Analysis] \label{thm:sensitivity_analysis}
Let $\mathbf{P}$, $\mathbf{P}'$ be two transition kernels that both induce irreducible Markov chain and admit (unique) invariant measures $\pi$ and $\pi'$ respectively. Assume further that $\tau_1(\mathbf{P}) < 1$, then for all $\xi$,
\begin{equation}
d_{\TV}(\pi, \pi') \leq \frac{1}{1-\tau_1(\mathbf{P})} \max_\xi d_{\TV}(\mathbf{P}(\xi,\cdot), \mathbf{P}'(\xi, \cdot)).
\end{equation}
\end{theorem}

Therefore, by bounding the ergodicity coefficients of $\mathbb{P}_f$ and $\Pi_f$, we could establish Proposition \ref{prop:Lipschitzness_of_stationary_measures}.

\begin{proof}[Proof of Proposition \ref{prop:Lipschitzness_of_stationary_measures}]
For the statement (1), recall that we have assumed that $f \geq f_{\min}$, so by \eqref{eq:minorisation_of_Pn0f}, for any $\xi,\xi'$ there exists $n_0$ and $C_{\mathsf{1}}$ such that
\begin{equation}
\mathbb{P}^{n_0}_f(\xi,\xi') \geq (f_{\min} \#\bm{\A})^{n_0} C_{\mathsf{1}} > 0.
\end{equation}
This prompts us to consider the Markov chain induced by the kernel $\mathbb{P}^{n_0}_f$, which also admits $\pi^f$ as the unique invariant measure. In particular, we have
\begin{equation*}
1 - \tau_1(\mathbb{P}^{n_0}_f) = \min_{\xi,\xi'} \sum_{\xi''} \min(\mathbb{P}^{n_0}_f(\xi,\xi''), \mathbb{P}^{n_0}_f(\xi',\xi'') \geq (f_{\min} \#\bm{\A})^{n_0} MC_{\mathsf{1}} > 0.
\end{equation*}
Therefore, by Theorem \ref{thm:sensitivity_analysis}, we have
\begin{equation*}
|\pi^f(\xi) - \pi^g(\xi)| \leq \max_{\xi'} \frac{d_{\TV}(\mathbb{P}^{n_0}_f(\xi',\cdot), \mathbb{P}^{n_0}_g(\xi', \cdot))}{MC_{\mathsf{1}} (f_{\min} \#\bm{\A})^{n_0}}.
\end{equation*}
We can obtain a sharper bound for the statement (2). This is to note that
\begin{align*}
1 - \tau_1(\Pi_{f,\rho_0}) 
&= \min_{\xi,\xi'} \sum_{\xi''} \min(\Pi_{f,\rho_0}(\xi,\xi''), \Pi_{f,\rho_0}(\xi',\xi'')) \\
&\geq \min_{\xi,\xi'} \sum_{\xi'' = (x'',a'')} (1-\gamma) f(\xi'') \rho_0(x'') \\
&= 1-\gamma > 0.
\end{align*}
Therefore, by Theorem \ref{thm:sensitivity_analysis}, we have
\begin{align*}
(1-\gamma)|\sigma^f_{\rho_0}(\xi) - \sigma^g_{\rho_0}(\xi)| &\leq \max_{\xi'} \frac{d_{\TV}(\Pi_f(\xi,\cdot), \Pi_g(\xi, \cdot))}{1-\gamma} \\
&= \frac{1}{1-\gamma} \max_{\xi'} \sum_{\xi''=(x'',a'')} |f(\xi'') - g(\xi'')| p(x''|\xi').
\end{align*}
\end{proof}

Let us demonstrate how the first statement could be used to establish the inequality \eqref{eq:difference_g_g_tilde}.

\begin{proof}[Proof of \eqref{eq:difference_g_g_tilde}]
Recall that 
\begin{align*}
g^N_s &= \frac{\eta^N_{\lfloor Ns \rfloor}}{\#\bm{\A}} + (1-\eta^N_{\lfloor Ns \rfloor}) f^N_s \\
\tilde{g}^N_s &= \frac{\eta_s}{\#\bm{\A}} + (1-\eta_s) f^N_s.
\end{align*}
Therefore, for any $\xi$, we have
\begin{align*}
g^N_s(\xi) - \tilde{g}^N_s(\xi) = (\eta^N_{\lfloor Ns \rfloor} - \eta_s) \left(\frac{1}{\#\bm{\A}} - f^N_s\right).
\end{align*}
Note that if $\eta_t = (1+(\log(1+t))^2)^{-1}$, then
\begin{align*}
\left|\frac{d\eta_t}{dt} \right| = \left|-\frac{2 \log (1+t)}{(1+t) (1+(\log(1+t))^2)^2}\right| \leq 2,
\end{align*}
and if $\eta_t = (1+t)^{-\varepsilon}$, then
\begin{align*}
\left|\frac{d\eta_t}{dt} \right| = \left|-\frac{\varepsilon}{(1+t)^{\varepsilon+1}}\right| \leq \varepsilon.
\end{align*}
In both cases, there exists constant $C > 0$ such that
\begin{align}
|\eta_s - \eta^N_{\floor{Ns}}|
&= |\eta_s - \eta_{\frac{\lfloor Ns \rfloor}N}| \nonumber \\
&\leq C \left|s-\frac{\floor{Ns}}{N}\right| \nonumber \\
&\leq C \left|\frac{\floor{Ns}}{N}+1 - \frac{\floor{Ns}}{N}\right| \leq \frac{C}{N}.
\end{align}
Let us once again abuse notation and set $\xi_0 = (x,a)$. Then, by using a telescoping sum argument in step $(*)$, we have
\begin{align*}
&\phantom{=} d_{\TV}(\mathbb{P}^{n_0}_f(\xi,\cdot), \mathbb{P}^{n_0}_g(\xi, \cdot)) \\
&\leq \sum_{\xi_{n_0}} \left|\sum_{\xi_1,...,\xi_{n_0-1}} \left(\prod_{\ell=0}^{n_0-1} g^N_s(\xi_{\ell+1}) p(x_{\ell+1} \mid \xi_\ell) - \prod_{\ell=0}^{n_0-1} \tilde{g}^N_s(\xi_{\ell+1}) p(x_{\ell+1} \mid \xi_\ell)\right)\right| \\
&\leq \sum_{\xi_1,...,\xi_{n_0}} \left| \prod_{\ell=0}^{n_0-1} g^N_s(\xi_{\ell+1}) p(x_{\ell+1} \mid \xi_\ell) - \prod_{\ell=0}^{n_0-1} \tilde{g}^N_s(\xi_{\ell+1}) p(x_{\ell+1} \mid \xi_\ell) \right| \\
&\overset{(*)}\leq \sum_{m=0}^{n_0-1} \sum_{\xi_1,...,\xi_{n_0}}  \left(\prod_{\ell \leq m-1} g^N_s(\xi_{\ell+1}) p(x_{\ell+1} \mid \xi_\ell)\right) |g^N_s(\xi_{\ell+1}) - \tilde{g}^N_s(\xi_{\ell+1})| p(\xi_{\ell+1}\mid \xi_\ell) \left(\prod_{m+1 \leq \ell \leq n_0-1} g^N_s(\xi_{\ell+1}) p(x_{\ell+1} \mid \xi_\ell)\right) \\
&\leq \frac{2}{N} \sum_{m=0}^{n_0-1} \sum_{\xi_1,...,\xi_{n_0}}  \left(\prod_{\ell \leq m-1} g^N_s(\xi_{\ell+1}) p(x_{\ell+1} \mid \xi_\ell)\right) \left(f^N_s(\xi_{\ell+1}) + \frac{1}{\#\bm{\A}} \right) p(\xi_{\ell+1}\mid \xi_\ell) \left(\prod_{m+1 \leq \ell \leq n_0-1} g^N_s(\xi_{\ell+1}) p(x_{\ell+1} \mid \xi_\ell)\right) \\
&\leq \frac{4n_0}{N}.
\end{align*}
Therefore, by applying statement (1) of Proposition \ref{prop:Lipschitzness_of_stationary_measures} with $f_{\min} = \eta_T/\#\bm{\A}$, we have
\begin{equation*}
|\pi^f(\xi) - \pi^g(\xi)| \leq \frac{4}{NMC_{\mathsf{1}} \eta_T^{n_0}} =: \frac{C_T}{N}.
\end{equation*}
\end{proof}

\section{Policy Gradient Theorem}
\label{S:policy_gradient}
The following appendix shall provide a short discussion regarding the backgrounds of the famous Policy Gradient Theorem, and how would the theorem be applied in the paper. We shall assume Assumptions \ref{as:MDP_basic} and \ref{as:ergodic_assumption} throughout.

\subsection{Backgrounds and Idea of Proof of Policy Gradient Theorem}
The Policy Gradient Theorem was first proved in \cite{policygradient1999}, which involved the unrolling of the Bellman equation. Here we shall provide justification of why the steps of unrolling is mathematically rigorous. \\

Assume a policy $f_\theta$ is parametrised by $\theta \in \R^{d_{\theta}}$. We could view the policy as a vector in $\R^M$, or a map 
\begin{equation*}
    \theta \in \R^{d_{\theta}} \mapsto f_\theta \in \R^M.
\end{equation*}
With this in mind, we could also see the state-value function as a map
\begin{equation}
    \theta \in \R^{d_\theta} \mapsto \bar{V}^{f_\theta} \in \R^{\#\bm{\X}},
\end{equation}
where $\bar{V}^{f_\theta}$ is a vector indexed by $x \in \bm{\X}$ and defined implicitly by the Bellman equation
\begin{equation}
    \bar{V}^{f_\theta}(x) = \sum_{a} f_\theta(x,a) \left[r(x,a) + \gamma \sum_{x'} \bar{V}^{f_\theta}(x') p(x'|x,a)\right]. \label{eq:Bellman_state_app}
\end{equation}
Similarly, the action-value function could be seen as a map 
\begin{equation}
\theta \in \R^{d_\theta} \mapsto V^{f_\theta} \in \R^M,
\end{equation}
as defined by 
\begin{equation}
V^{f_\theta}(x,a) = r(x,a) + \gamma \sum_{x'} \bar{V}^{f_\theta}(x') p(x'|x,a). \label{eq:Bellman_action_app}
\end{equation}
We shall justify that $\bar{V}^{f_\theta}$ is well-defined by showing that \eqref{eq:Bellman_state_app} always have a solution for any transition probabilities $p$. It is useful to rewrite \eqref{eq:Bellman_state_app} in the matrix form:
\begin{equation}
    \bar{V}^{f_\theta} = \bar{r}(\theta) + K(\theta) \bar{V}^{f_\theta}, \label{eq:Bellman_state_app_vec}
\end{equation}
where $\bar{r}(\theta) \in \R^{\#\bm{\X}}$ and $K(\theta) \in \R^{\#\bm{\X} \times \#\bm{\X}}$ are defined in terms of indices:
\begin{equation}
    [\bar{r}(\theta)]_x = \sum_a r(x,a) f_\theta(x,a), \quad [K(\theta)]_{xx'} = \sum_{a} f_\theta(x,a) p(x'|x,a).
\end{equation}
We note that the sum of column entries equals to 1:
\begin{align*}
    \sum_{x'} [K(\theta)]_{xx'} = \sum_{a,x'} f_\theta(x,a) p(x'|x,a) = \sum_a f_\theta(x,a) = 1,
\end{align*}
so, for any $\theta$, the matrix $K(\theta)$ has eigenvalues of norm $\leq 1$ by the Perron-Frobenius Theorem. As $\gamma < 1$, the matrix $I - \gamma K(\theta)$ now has non-zero eigenvalues, so must therefore be invertible. Thus \eqref{eq:Bellman_state_app} admits a unique solution for any $p$ and $\theta$.

We shall now derive the derivatives of the $\bar{V}^{f_\theta}$ with respect to $\theta$, with the assumption that $f_\theta$ has bounded derivatives. To do so, we note that \eqref{eq:Bellman_state_app_vec} could be further written in terms of the function $F:\R^{d_\theta} \times \R^{\#\bm{\X}} \to \R^{\#\bm{\X}}$:
\begin{equation*}
F(\theta, \bar{V}^{f_\theta}) = 0, \quad F(\theta, \bar{v}) = -\bar{r}(\theta) + (I - \gamma K(\theta))\bar{v}.
\end{equation*}
Thus we could invoke the Implicit Function Theorem to compute the derivatives of the $\bar{V}^{f_\theta}$. If we denote $D_\theta F$ as the Jacobian (as matrix in $\R^{\#\bm{\X} \times d_\theta}$) of $F$ that includes derivatives with respect to $\theta$, and similar $D_{\bar{v}} F$ as the Jacobian (as matrix in $\R^{\#\bm{\X} \times \#\bm{\X}}$) that includes derivatives with respect to $v$, then
\begin{equation*}
D_{\bar{v}} F = I - \gamma K(\theta),
\end{equation*}
so the total derivative $D_\theta \bar{V}^{f_\theta}$ (as matrix in $\R^{\#\bm{\X} \times d_\theta}$) is
\begin{equation*}
D_\theta \bar{V}^{f_\theta} = -(D_{\bar{v}} F)^{-1} D_\theta F = -(I - \gamma K(\theta))^{-1} D_\theta F.
\end{equation*}
To obtain the infinite series form, we note that the operator norm of $\gamma K(\theta)$ is strictly less than $1$, so it is true that
\begin{equation*}
(I-\gamma K(\theta))^{-1} = \sum_{k\geq 0} \gamma^k (K(\theta))^k,
\end{equation*}
with the right hand side converges in the operator norm. \\

It remains to compute $-D_\theta F$ explicitly. Recall that the rows of $D_\theta F$ are the transpose of the gradient of the $x$-entry of $F$, i.e. $(\nabla_\theta [F]_x)^\top$. We shall, therefore, have
\begin{align*}
-[D_\theta F]_{x,:} &= -(\nabla_\theta [F]_x)^\top = \sum_a r(x,a) (\nabla_\theta f_\theta(x,a))^\top + \gamma \sum_{a,x'} p(x'|x,a) \bar{V}^{f_\theta}(x') (\nabla_\theta f_\theta(x,a))^\top \\
&= \sum_a \left[r(x,a) + \gamma \sum_{x'} p(x'|x,a) \bar{V}^{f_\theta}(x') \right] (\nabla_\theta f_\theta(x,a))^\top \\
&\overset{\eqref{eq:Bellman_action_app}}= \sum_{a} V^{f_\theta}(x,a) (\nabla_\theta f_\theta(x,a))^\top  
\end{align*}
The $x$-row of the entries $(K(\theta))^k D_\theta F$ could therefore be computed inductively -- in fact,
\begin{align*}
-[(K(\theta))^k D_\theta F]_{x,:} &= -\sum_{x^{(1)}} [K(\theta)]_{x,x^{(1)}} [(K(\theta))^{k-1} D_\theta F]_{x^{(1)},:} \\
&= -\sum_{a,x^{(1)}} f_\theta(x,a) p(x^{(1)}|x,a) [K(\theta)^{k-1} D_\theta F]_{x^{(1)},:} \\
&= ... \text{(unrolling)} \\
&= -\sum_{a,x^{(1)},a^{(1)},...,x^{(k)}} f_\theta(x,a) p(x^{(1)}|x,a) ... f_\theta(x^{(k-1)}, a^{(k-1)}) p(x^{(k)}|x^{(k-1)}, a^{(k-1)}) ... [D_\theta F]_{x^{(k)},:} \\
&= \sum_{x^{(k)}} \prob(x_k=x^{(k)} \mid x_0 = x) \, [-D_\theta F]_{x^{(k)},:}.
\end{align*}
Therefore,
\begin{align*}
[D_\theta \bar{V}^{f_\theta}]_x 
&= \sum_{k\geq 0} \sum_{x'} \gamma^k \prob(x_k=x' \mid x_0 = x) [-D_\theta F]_{x',:} \\
&= \sum_{x'} \left(\sum_{k\geq 0} \gamma^k \prob(x_k=x' \mid x_0 = x)\right) [-D_\theta F]_{x',:} \\
&= \sum_{x'} \nu^{f_\theta}_{x}(x') [-D_\theta F]_{x',:} \\
&= \sum_{x',a} \nu^{f_\theta}_{x}(x') V^{f_\theta}(x',a) (\nabla_\theta f_\theta(x',a))^\top.
\end{align*}
This, therefore, establishes the Policy Gradient Theorem for the state-value function:
\begin{equation}
\nabla_\theta \bar{V}^{f_\theta}(x) = \sum_{x',a} \nu^{f_\theta}_{x}(x') V^{f_\theta}(x',a) \nabla_\theta f_\theta(x',a). \label{eq:Policy_gradient_state_app}
\end{equation}
We can also derive a Policy Gradient theorem for the action-value function by differentiating \eqref{eq:Bellman_action_app}:
\begin{align*}
\nabla_\theta V^{f_\theta}(x,a) &= \gamma \sum_{x'} p(x'|x,a) \nabla_\theta \bar{V}^{f_\theta}(x') \\
&= \gamma \sum_{x',a',x''} p(x'|x,a) \nu^{f_\theta}_{x'}(x'') V^{f_\theta}(x'',a') \nabla_\theta f_\theta (x'',a'). \numberthis \label{eq:Policy_gradient_action_app}
\end{align*}
Finally, we recall that the target function used in the Actor-Critic algorithm is the weighted average $J(\theta) = \sum_{x} \rho_0(x) \bar{V}^{f_\theta}(x)$, which could be differentiated to obtain
\begin{align*}
\nabla_\theta J(\theta) = \sum_{x, a, x'} \nu^{f_\theta}_{x}(x') \nabla_\theta f_\theta(x',a) V^{f_\theta}(x',a) \rho_0(x) = \sum_{a, x'} \nu^{f_\theta}_{\rho_0}(x') \nabla_\theta f_\theta(x',a) V^{f_\theta}(x',a) .
\end{align*}
If we assume that $f_\theta(x,a) > 0$ for all $(x,a) \in \bm{\X} \times \bm{\A}$, then 
\begin{align*}
\nabla_\theta J(\theta) &= \sum_{x,a,x'} \bracket{\sum_{k\geq 0} \gamma^k f_\theta(x,a) \p(x_k = x' \,|\, x) \rho_0(x)} \frac{1}{f_\theta(x',a)} \nabla_\theta f_\theta(x',a) V^{f_\theta}(x',a) \\
&= \sum_{a,x'} \sigma^{f_\theta}_{\rho_0}(x',a) V^{f_\theta}(x',a) \nabla_\theta (\ln f_\theta(x',a)).
\end{align*}

\subsection{Applications of the Policy Gradient Theorem}
The simplest application of \eqref{eq:Policy_gradient_state_app} and \eqref{eq:Policy_gradient_action_app} would be to study how the value functions evolve according to the limiting ODE \eqref{limit odes}, which could be done by directly parametrising the policy with time $t$, i.e., replace $\theta$ with $t$.

We could also look at the special case when we have a tabulated policy, i.e. $\theta \in \R^M$ and $\theta \mapsto f_\theta$ is the identity map itself. Denoting $\theta_{x,a}$ as the $(x,a)$-entry of $\theta$, \eqref{eq:Policy_gradient_state_app} and \eqref{eq:Policy_gradient_action_app} now reads the following:
\begin{align}
\frac{\partial \bar{V}^{f_\theta}(x_0)}{\partial \theta_{x,a}} &= \nu^{f_\theta}_{x_0}(x) V^{f_\theta}(x_0,a) \\
\frac{\partial V^{f_\theta}(x_0,a_0)}{\partial \theta_{x,a}} &= \gamma \sum_{x'} p(x'|x_0,a_0) \nu^{f_\theta}_{x'}(x) V^{f_\theta}(x',a). \label{eq:Policy_gradient_tabular_action_app}
\end{align}
From \eqref{eq:Policy_gradient_tabular_action_app}, we could see that the derivatives of $V^{f_\theta}$ is always bounded, as
\begin{equation*}
\left|\frac{\partial V^{f_\theta}(x_0, a_0)}{\partial \theta_{x,a}}\right| \leq \gamma \sum_{x'} p(x'|x_0,a_0) \nu^{f_\theta}_{x'}(x) |V^{f_\theta}(x',a)| \leq \frac{\gamma}{(1-\gamma)^2} \sum_{x'} p(x'|x_0,a_0) = \frac{\gamma}{(1-\gamma)^2}. \label{eq:boundedness_of_tabular_value_function}
\end{equation*}
Therefore $\theta \mapsto V^{f_\theta}$ is Lipschitz in $\theta$ in the tabular case.

\bibliographystyle{plain}
\bibliography{cite}
\end{document}